\documentclass{article}

\usepackage{microtype}
\usepackage{graphicx}
\usepackage{booktabs} 

\usepackage{hyperref}

\usepackage[accepted]{icml2020}

\icmltitlerunning{Deep PQR: Solving Inverse Reinforcement Learning using Anchor Actions}

\usepackage{dsfont}
\usepackage{bm}

\newcommand{\curly}[1]{\left\{#1\right\}}
\newcommand{\norm}[1]{\lVert #1 \rVert}

\usepackage{amsthm}

\newcommand{\bx}{\mathbf{x}}
\newcommand{\be}{\mathbf{e}}

\newcommand{\bs}{\mathbf{s}}
\newcommand{\by}{\mathbf{y}}
\newcommand{\ba}{\mathbf{a}}

\newcommand{\bA}{\mathbf{A}}

\newcommand{\bH}{\mathbf{H}}

\newcommand{\bS}{\mathbf{S}}

\newcommand{\bepsilon}{\mathbf{\epsilon}}

\newcommand{\bomega}{\bm{\omega}}
\newcommand\given[1][]{\:#1\vert\:}

\newcommand{\EE}{{\mathds{E}}}

\usepackage{mathtools}
\DeclarePairedDelimiter\abs{\lvert}{\rvert}%
\makeatletter
\let\oldabs\abs
\def\abs{\@ifstar{\oldabs}{\oldabs*}}

\usepackage{makecell}
\usepackage{comment}
\usepackage{natbib}
\usepackage{graphicx}

\usepackage{multirow}
\usepackage{float}

\usepackage{pbox}

\usepackage{graphicx}

\usepackage{capt-of}
\usepackage{caption}
\usepackage{subcaption}
\usepackage{microtype}

\usepackage{algorithm}
\usepackage{algorithmicx}
\usepackage{algpseudocode}
\algnewcommand\algorithmicinput{\textbf{Input:}}
\algnewcommand\algorithmicoutput{\textbf{Output:}}
\algnewcommand\algorithmicReturn{\textbf{Return:}}

\algnewcommand\Input{\item[\algorithmicinput]}%
\algnewcommand\Output{\item[\algorithmicoutput]}%

\algnewcommand{\Initialize}[1]{%
	\State \textbf{Initialize:}
	\Statex \hspace*{\algorithmicindent}\parbox[t]{.8\linewidth}{\raggedright #1}
}

\usepackage{comment}

\newcommand{\eq}[2]{\begin{equation} \label{eq:#1} #2 \end{equation}}
\newcommand{\eqs}[1]{\begin{equation*} #1 \end{equation*}}
\newcommand{\ali}[2]{\begin{align} \label{eq:#1} \begin{split}#2\end{split}   \end{align}}
\newcommand{\alis}[1]{\begin{align*}\begin{split} #1 \end{split}\end{align*}  }

\newtheorem{lemma}{Lemma}
\newtheorem{theorem}{Theorem}
\newtheorem{remark}{Remark}
\newtheorem{definition}{Definition}

\newtheorem{assumption}{Assumption}

\usepackage[inline]{enumitem}
\newlist{inlineenum}{enumerate*}{1}
\setlist*[inlineenum,1]{%
      label=(\roman*),%
}

\begin{document}
\twocolumn[
\icmltitle{Deep PQR: Solving Inverse Reinforcement Learning using Anchor Actions} 



\icmlsetsymbol{equal}{*}

\begin{icmlauthorlist}
\icmlauthor{Sinong Geng}{csPrinceton}
\icmlauthor{Houssam Nassif}{amazon} 
\icmlauthor{Carlos A. Manzanares}{amazon}
\icmlauthor{A. Max Reppen}{orfePrinceton}
\icmlauthor{Ronnie Sircar}{orfePrinceton}

\end{icmlauthorlist}
 
\icmlaffiliation{csPrinceton}{Department of Computer Science, Princeton University, Princeton, New Jersey, USA}
\icmlaffiliation{amazon}{Amazon, Seattle, Washington, USA}
\icmlaffiliation{orfePrinceton}{Operations Research and Financial Engineering, Princeton University, Princeton, New Jersey, USA}

\icmlcorrespondingauthor{Sinong Geng}{sgeng@princeton.edu}

\icmlkeywords{Machine Learning, ICML}

\vskip 0.3in
]



\printAffiliationsAndNotice{}  

\begin{abstract}
We propose a reward function estimation framework for inverse reinforcement learning with deep energy-based policies. 
We name our method PQR, as it sequentially estimates the Policy, the $Q$-function, and the Reward function by deep learning. 
PQR does not assume that the reward solely depends on the state, instead it allows for a dependency on the choice of action. Moreover, PQR allows for stochastic state transitions.
To accomplish this, we assume the existence of one anchor action whose reward is known, typically the action of doing nothing, yielding no reward.
We present both estimators and algorithms for the PQR method.
When the environment transition is known, we prove that the PQR reward estimator uniquely recovers the true reward.
With unknown transitions, we bound the estimation error of PQR.
Finally, the performance of PQR is demonstrated by synthetic and real-world datasets. \end{abstract}

\section{Introduction}
\label{sec:introduction}


Inverse reinforcement learning (IRL)~\citep{russell1998learning, ng2000algorithms} aims to estimate the reward function of an agent given its decision-making history (often referred to as \emph{\textbf{demonstrations}}). IRL assumes that the policy of the agent is optimal. 
The reward function characterizes the preferences of agents used for decision-making, and is thus crucial in many scenarios. In reinforcement learning, the reward function is important to replicate decision-making behavior in novel environments~\citep{fu2017learning}; in financial markets, the reward function depicts the risk tolerance of agents, suggesting how much compensation they require in exchange for asset volatility~\citep{merton1973intertemporal,merton1975optimum}; in industrial organization, reward functions are synonymous with firm profit functions~\citep{abbring2010identification,aguirregabiria2013recent}.

 However, different reward functions may lead to the same behaviour of agents, making it impossible to identify the true underlying reward function~\citep{ng2000algorithms}. This \emph{\textbf{identification issue}} has been a bottleneck for policy optimization in a transferred environment \citep{fu2017learning}, and for recovering utility and profit functions in economics~\citep{abbring2010identification,bajari2010identification, manzanares2015improving}.
To partly address this issue, Maximal Entropy Inverse Reinforcement Learning (MaxEnt-IRL)~\citep{ziebart2008maximum, wulfmeier2015maximum, finn2016guided,fu2017learning} introduces stochastic policies to the problem formulation, and estimates the reward function using maximal entropy. However, the reward function is still not fully identifiable~\citep{fu2017learning}. Based on MaxEnt-IRL, to deal with this issue, Adversarial Inverse Reinforcement Learning (AIRL) assumes that the reward is solely a function of state variables, and is independent of actions. However, even with state-only reward functions, AIRL and related methods~\citep{kostrikov2018discriminator} are guaranteed to work only for deterministic environment transitions with the so-called decomposability condition. In some scenarios (especially in many real-world settings in economics), these requirements are hard to satisfy. 

This work proposes a Policy $Q$-function Reward (PQR) approach, combined with an anchor-action assumption, to identify and flexibly estimate reward functions in infinite-horizon Markov decision processes with stochastic transitions. We assume that agents follow energy-based policies derived from optimally solving an entropy-augmented reinforcement learning objective~\citep{kappen2005path,todorov2007linearly,haarnoja2017reinforcement}. 
We estimate the policy by taking advantage of recent developments in Maximal Entropy Adversarial Imitation Learning methods (MaxEnt-AIL), such as Generative Adversarial Imitation Learning (GAIL)~\citep{ho2016generative} and the aforementioned AIRL~\citep{fu2017learning}. These methods either do not attempt to identify the reward function, or require strong assumptions to do so. However, they provide efficient and effective methods to replicate the policy from agent demonstrations. We aim to recover the reward function by leveraging the accurately estimated policies. 

Our estimators for the $Q$-function and reward function depend on the estimated policy. We handle the \emph{\textbf{identification issue}} by assuming the existence of an \emph{\textbf{anchor action}} with known rewards. One natural example is an action that provides no rewards across all possible realizations of states.
Such a zero-reward action may correspond to doing nothing, and exists in many applications~\citep{bajari2010identification}. Using this assumption, under fairly general conditions, we show that, if
\begin{inlineenum} 
    \item the environment transition is given and
    \item policy functions are identifiable from observed demonstrations, the reward estimator uniquely recovers the true underlying reward function. 
\end{inlineenum} When the environment transition is unknown, the calculation of the proposed estimator is equivalent to first solving a Markov Decision Process (MDP) with only one available action, and then solving a supervised learning problem. Leveraging this characteristic, we provide effective learning algorithms. We also prove that the error of PQR is bounded and show how it decreases as the sample size increases. 

Many imitation learning approaches~\citep{reddy2019sqil,ross2011reduction,mahler2017learning,bansal2018chauffeurnet,gupta2019relay, de2019causal} do not attempt to recover reward functions, ineased focusing on identifying policy functions. 
Another related approach involves agent alignment~\citep{leike2018scalable,reddy2019learning}, which focuses on
\begin{inlineenum}
    \item learning reward functions from demonstrations, and
    \item scaling these reward functions to even more complex domains. 
\end{inlineenum}
We focus on the first part and provide a more accurate reward function estimator. 
In summary, our main contributions are two fold:
\begin{itemize}
    \item We propose a reward estimator that is guaranteed to uniquely recover the true reward function under the anchor-action assumption when the environment transition is given.
    \item Using the same anchor-action assumption, we propose a learning algorithm to calculate the estimator when the environment transition is unknown, and provide its convergence analysis. 
\end{itemize}

The utility of PQR is demonstrated on synthetic data and an illustrative airline market entry analysis.


\section{Deep Energy-Based Modeling}
\label{sec:PQR}
In this section, we characterize our setting for how agents make decisions. Similar to MaxEnt-IRL, we assume that agents maximize the entropy-augmented long-term expected reward~\citep{kappen2005path,todorov2007linearly,haarnoja2017reinforcement}. 

We describe an MDP by the tuple $M = (\mathcal{S}, \mathcal{A}, \textnormal{P}, \gamma, r)$, where $\mathcal{S}$ is the state space, $\mathcal{A}$ is the action space, $\textnormal{P}$ is the transition distribution, $\gamma\in (0,1)$ is the discount factor, and $r$ is the reward function. Let $\bS_t$ take values in $\mathcal{S}$ and $\bA_t$ in $\mathcal{A}$ to denote state and action variables, respectively. Given a starting state $\bS_0 = \bs$, the value function is defined as
\ali{energy-control}{
V(\bs) &:= \max_\pi \sum_{t=0}^\infty \gamma^t \, \mathds{E}[ 
 r(\bS_t, \bA_t)
 \\ &\quad +\alpha \mathcal{H}(\pi( \bS_t, \cdot))  \given \bS_0 = \bs
].
}
\eqs{\mathcal{H}(\pi (\bs, \cdot)) := - \int_{\mathcal{A}} \log(\pi(\bs,\ba)) \pi(\bs, \ba) \, d\ba } is the information entropy. 
$\alpha \mathcal{H}(\pi(\cdot \given \bS_t))$ with $\alpha >0$ acts as an encouragement for the randomness of agent behavior. $\pi(\bs,\ba)$ denotes the stochastic policy of agents, representing the conditional probability $\textnormal{P}(\bA_t = \ba_t \given \bS_t = \bs)$. Thus, the expectation in \eqref{eq:energy-control} is over both the transition $\mathcal{P}$ and the stochastic policy $\pi$. 
We assume that the agent takes a stochastic energy-based policy:
\eqs{\pi(\bs, \ba) = \frac{ \exp(-\mathcal{E}(\bs, \ba))}{\int_{\ba ' \in \mathcal{A}} \exp(-\mathcal{E}(\bs, \ba') d\ba')  },} where $\mathcal{E}$ is an energy function represented by deep neural networks. Such an energy distribution is widely used in machine learning, see Boltzmann machines~\citep{hinton2012practical}, Boltzmann bandits~\citep{biswas2019seeker}, and Markov random fields~\citep{geng2018temporal}.
The likelihood of the decision-making via such stochastic energy-based policies is reported in Lemma~\ref{lem:PQR}. 


\begin{lemma}[Summary of \citet{haarnoja2017reinforcement}]
\label{lem:PQR}
By assuming that agents solve \eqref{eq:energy-control} optimally with energy-based policies, the likelihood of the observed demonstrations $\mathds{X} = \curly{\bs_0, \ba_0, \bs_1, \ba_1, \cdots, \bs_T, \ba_T}$ can be derived as
\eqs{L(\mathds{X};r) = \prod_{t=0}^{T} \pi^*(\bs_t, \ba_t)
\prod_{t=0}^{T-1}\textnormal{P}(\bs_{t+1} \given \bs_t, \ba_t),}
where the optimal policy function taken by agents follows
\eq{policy-s}{
    \pi^*(\bs , \ba) = \frac{ \exp {\left(
\frac{1}{\alpha}
Q(\bs, \ba)\right)} }{ \int_{\ba' \in \mathcal{A}}\exp {\left(
\frac{1}{\alpha}
Q(\bs, \ba')\right)} d\ba'},
}
with
\alis{
Q&(\bs, \ba):= r(\bs, \ba)
\\&+\max_\pi \mathds{E}{\left\{
\sum_{t=1}^\infty \gamma^t\left[
r(\bS_{t}, \bA_t) + \alpha \mathcal{H}(\pi(\bS_{t}, \cdot))
\right] \given \bs, \ba
\right\}}. 
} 

\end{lemma}

Lemma~\ref{lem:PQR} derives the likelihood for our inverse reinforcement learning problem. The behavior of agents is governed by $Q$-function: $(\bs, \ba)$ pairs with higher $Q$ values are more likely to appear in the dataset. In other words, given $\bs$, an agent is more likely (but not guaranteed) to select better actions. This behavior is more realistic in practice than assuming agents always deterministically make the best decision. 
The two hyper parameters $\gamma$ and $\alpha$ correspond to the patience and the rationality of agents, respectively. The larger the $\gamma$ is, the more agents value future returns. As $\alpha$ increases, the difference between $Q$ values of different actions diminishes, and the influence of randomness in agent decision-making increases. At the extreme, this is equivalent to choosing actions randomly, which could be interpreted as agents not knowing or trusting existing information. When $\alpha$ approaches $0^+$, \eqref{eq:energy-control} is equivalent to a classic control or RL problem with a deterministic policy. 

\section{Related Models}
\label{sec:related-models}
In this section, we review related settings for reward function estimation: Inverse Reinforcement Learning (IRL) in machine learning and Dynamic Discrete Choice (DDC) in economics. We show that our setting is more general and facilitates more efficient and effective estimation methods.

\subsection{Inverse Reinforcement Learning}
The most relevant model for our problem is MaxEnt-IRL~\citep{ziebart2008maximum}, where the likelihood of the observation set $\mathds{X} = \curly{\bs_0, \ba_0, \bs_1, \ba_1, \cdots, \bs_T, \ba_T}$ is 
\eq{maxent}{
L^{MaxEnt}(\mathds{X};r) \propto  \prod_{t=1}^{T}  \exp {\left(
r(\bs_t, \ba_t)\right)}.
}
Comparing \eqref{eq:maxent} with the likelihood in Lemma~\ref{lem:PQR}, MaxEnt-IRL tends to treat the $Q$-function in Lemma~\ref{lem:PQR} as the reward function and thus simplifies the problem. 
 As a result, the reward function of MaxEnt-IRL incorporates some environment information, an issue known as environment entanglement~\citep{fu2017learning}. Instead, we properly represent the relationship between $r$ and $Q$, and aim to estimate the reward function. When $\gamma = 0$, $r(\bs,\ba) = Q(\bs,\ba)$, our setting degenerates to the MaxEnt-IRL setting. When $\gamma \neq 0$, MaxEnt-IRL methods can still be applied to our setting, albeit with the simplification of using our $Q$-function as their reward function.

Classic IRL models~\citep{ng2000algorithms} with deterministic policies are also a special case. As $\alpha$ approaches $0$, \eqref{eq:policy-s} gets close to hard maximization over $\ba_t \in \mathcal{A}$. This recovers classic IRL with deterministic policies, where agents take the best decision deterministically every time. This behavior is not realistic in practice: in almost all real-world settings that model human behavior, the reward function includes information available to the agent but not available to the statistician. Classic IRL also suffers identification issues.

\subsection{Dynamic Discrete Choice}
DDCs~\citep{rust1987optimal} are dynamic structural models with discrete and finite action spaces, popular in economics. A very popular DDC setting assumes that there exists two sets of state variables: $\bS_t$ and $\bepsilon _t$, where $\bepsilon _t$ are random shocks following a Type-I extreme value distribution~\citep{ermon2015learning}. While $\bS_t$ is observable and recorded in the dataset, $\bepsilon_t$ is not observable in the dataset, and is only known by the agent when making decisions at time $t$. With some further assumptions listed in Section~\ref{sec:sup-ddc} of the Supplements, the likelihood of DDC becomes
\alis{
L^{DDC}(\mathds{X};r) \propto  \prod_{t=1}^{T}  \exp {\left[
Q^{DDC}(\bs_t,\ba_t)\right]},
} 
where $Q^{DDC}(\bs,\ba)$ is a counterpart to our $Q$-function for a discrete finite action space with $\alpha=1$. Details about $Q^{DDC}(\bs,\ba)$ is deferred to the Section~\ref{sec:sup-ddc}. The estimation method in Section~\ref{sec:estimation} can also be applied to DDCs to efficiently recover reward functions.

\section{Policy-$Q$-Reward Estimation}
\label{sec:estimation}
We now provide our Policy-$Q$-Reward estimation, which proceeds by estimating agent policies, $Q$-functions and rewards in that order. We only propose the estimators and algorithms in this section, and theoretically justify our estimation strategy in Section~\ref{sec:error}. 


\subsection{Anchor-Action Assumption}
\begin{assumption}
\label{asm:identification}
There exists a known anchor action $\ba^{A} \in \mathcal{A}$ and a function $g: \mathcal{S}\longmapsto \mathds{R}$, such that $r(\bs, \ba^{A}) = g(\bs)$. 
\end{assumption}
To correctly identify the reward function, we require Assumption~\ref{asm:identification}, which stipulates that there exists an anchor action $\ba^{A}$ whose reward function value is known a priori.
A special case is $g(\bs) = 0$, indicating that there exists an anchor action providing no rewards. This class of assumption is widely used in economics and has been proved crucial for even much simpler static models~\citep{bajari2010identification}. Assumption~\ref{asm:identification} is more realistic in many scenarios~\citep{hotz1993conditional,bajari2010identification} than those in \citet{fu2017learning} which require a state-only reward function and a deterministic environment transition. In settings where firms are involved, any non-action (like not selling a good, not entering a market) typically leads to zero rewards and provides a natural anchor~\citep{sawant2018causal}.

\subsection{Policy estimation}
\label{sec:airl}
Let $\hat{\pi}$ be an estimator to the true policy function $\pi^*(\ba, \bs)$. Our PQR procedure works with any stochastic policy estimator for energy distributions~\citep{geng2017efficient,kuang2017screening,fu2017learning,geng2018stochastic} . We use the Adversarial Inverse Reinforcement Learning (AIRL) procedure of \citet{fu2017learning} as an example. Recall from Section~\ref{sec:PQR} that $\pi^*(\bs, \ba)=\textnormal{P} (\bA_t = \ba \given \bS_t = \bs)$. AIRL estimates the conditional distribution of $\bA_t$ given $\bS_t$ with a GAN, where the discriminator is
$D(s,a) := \frac{\exp\curly{f(\bs,\ba)}}{\exp\curly{f(\bs,\ba)} + P(\ba\given \bs)}$,
and $P(\ba \given \bs)$ is updated by solving \eqref{eq:energy-control} substituting the $Q$-function with $\log(1-D(\bs,\ba)) - \log D(\bs,\ba)$. At optimality, this procedure leads to $f(\bs, \ba) = \log(\pi^*(\bs, \ba))$. While, the aforementioned procedure does not estimate the $Q$-function or the reward function, \citet{fu2017learning} provides a disentangling procedure to estimate the reward function. We refer to AIRL with the disentangling procedure as D-AIRL. D-AIRL requires a state-only reward function and deterministic transitions. The proposed PQR does not use the disentangling procedure, and we compare PQR with D-AIRL in Section~\ref{sec:exp}.

\subsection{Reward Estimation}
\label{sec:reward}


For ease of explanation, we tackle the reward estimator before the $Q$-function estimator. 
\begin{definition}
\label{def:re}
Given Q-function estimator $\hat{Q}(\bs,\ba)$ and policy estimator $\hat{\pi}(\bs, \ba)$, the reward estimator is defined as:
\begin{align}
\label{eq:main}
\begin{split}
\hat{r}(\bs, \ba) &:= \hat{Q}(\bs, \ba) 
 \\&- \gamma \hat{\mathds{E}}_{\bs'}{\left[-\alpha \log(\hat{\pi}(\bs', \ba^{A})) + \hat{Q}(\bs', \ba^{A})\given \bs, \ba
\right]},
\end{split}
\tag{\textsc{R-Estimator}} 
\end{align}
where $\ba^{A}$ is the anchor action. $\hat{\mathds{E}}_{\bs'}$ denotes the estimated expectation over $\bs'$, the one-step look-ahead state variable. When the environment transition is known, we obtain the exact expectation $\EE_{\bs'}$. 
\end{definition}

Note that \ref{eq:main} shares a very similar form as the Bellman equation for $Q$:
\eqs{Q(\bs,\ba) = r(\bs, \ba) + \gamma \EE{\left[ V(\bs') \given \bs, \ba \right]},} 
except the value function is represented by $\hat{Q}(\bs,\ba^{A})$ and $\hat{\pi}(\bs, \ba^{A})$.  
This representation avoids directly calculating the value function, which is very challenging and unavailable by many methods including AIRL.
Also, the representation requires a one-step-ahead expectation instead of the entire trajectory, making the calculation or estimation for the expectation much easier.

The intuition behind this representation is to replace the value function in the Bellman equation with a specification that makes agent reward functions explicitly a function of agent policies.
Representing reward functions in this way can be traced back to \citet{hotz1993conditional}, and has been used by various subsequent economics publications~\citep{hotz1994simulation,train2009discrete,arcidiacono2011conditional}.
We extend this rational to our more general PQR setting, which eventually facilitates more efficient estimation methods.
Given $\hat{\pi}(\bs,\ba)$ and $\hat{Q}(\bs,\ba)$, Algorithm~\ref{alg:reward} provides an implementation of \ref{eq:main} to estimate the reward, and uses
a deep neural network to estimate the expectation. The calculation of \ref{eq:main} reduces to a supervised learning problem.

\begin{algorithm}[t]
		\caption{Reward Estimation (RE)}
		\label{alg:reward}
		\begin{algorithmic}[1]
			\Input Dataset: $\mathds{X} = \curly{\bs_0, \ba_0, \cdots, \bs_T, \ba_T}$.
			\Input $\hat{Q}(\bs,\ba)$ and $\hat{\pi}(\bs,\ba)$.
			\Output $\hat{r}(\bs, \ba)$.
			\Initialize{ Initialize a deep function $h:\mathcal{S}\times\mathcal{A} \longmapsto \mathds{R}$.}
			\For {$t\in [T]$}
			\State {\small$y_t ~\leftarrow -\alpha \log(\hat{\pi}(\bs_{t+1}, \ba^{A})) + \hat{Q}(\bs_{t+1},\ba^{A})$}
			\EndFor
			\State Train $h$ using $\curly{y_t}_{t=0}^{T-1}$ with $\curly{\bs_t}_{t=0}^{T-1}$ and $\curly{\ba_t}_{t=0}^{T-1}$.
			\State \Return $\hat{r}(\bs, \ba) = \hat{Q}(\bs,\ba) - \gamma h(\bs,\ba)$.
		\end{algorithmic}
	\end{algorithm}

\subsection{$Q$-function Estimation}
\label{sec:identification} 
$Q$-function estimation is susceptible to the identification issue. 
Let $Q(\bs, \ba)$ be the ground-truth $Q$-function. For any function $\phi: \mathcal{S}\longmapsto \mathds{R}$, 
\eqs{Q'(\bs, \ba) := Q(\bs, \ba)+ \phi(\bs)} leads to the same behaviour of agents as $Q(\bs,\ba)$, making it impossible to recover $Q(\bs,\ba)$ without further assumptions. 
The $Q$-function is \textbf{\emph{shaped}} by $\phi(\bs)$, as per \citet{ng1999policy} terminology.
Even worse, if we wrongly take $Q'(\bs, \ba)$ as the true $Q(\bs, \ba)$, Algorithm~\ref{alg:reward} can lead to a highly shaped reward estimation:
\begin{lemma}
\label{lem:confound}
Let $Q$ estimator (denoted by $Q'$) be shaped by a function $\phi(\bs)$. \ref{eq:main} leads to a wrongly estimated reward function, even when the expectation and policy function are calculated uniquely and exactly:
\eqs{
        r'(\bs, \ba) = r(\bs, \ba) + \Phi(\bs,\ba),
}
with $\Phi(\bs,\ba) :=  \phi(\bs) - \gamma \hat{\EE}[\phi(\bs)\given \bs, \ba]$.
\end{lemma}
Lemma~\ref{lem:confound} is proved in  Section~\ref{sec:confound-proof} of the Supplements. According to Lemma~\ref{lem:confound}, the wrongly estimated $Q(\bs,\ba)$ induces a reward function estimation shaped in terms of \textbf{\emph{both}} the action variable and the state variable. This additive shaping is a bottleneck for policy optimization in a transferred environment \citep{fu2017learning}. It is also a long-studied problem in economics, where identification of primitives of utility functions is a first step in modeling customer choices~\citep{bajari2010identification, abbring2010identification}. An entire field called partial identification~\citep{tamer2010partial} also focuses on this issue; namely, the conditions under which utility and reward functions (or parameters of these functions) are uniquely identified, or whether they can instead be characterized by identified sets. To deal with this identification issue, we use a two-step strategy. We first identify the $Q$-function for $\ba^{A}$, and then use it to recover $Q(\bs,\ba)$ for other $\ba \in \mathcal{A}$.

\begin{algorithm}[t]
\caption{FQI-I}
		\label{alg:fqii}
		\begin{algorithmic}[1]
			\Input Dataset: $\mathds{X}$.
			\Input $\gamma$, $\alpha$, $N$, and $\hat{\pi}(\bs, \ba)$.
			\Output $\hat{Q}(\bs, \ba)$. 
			\Initialize{ Initialize a deep function $h:\mathcal{S} \longmapsto
 \mathds{R}$.}
			\State {$y_t = 0$ for $t \in \curly{t\given \ba_t = \ba^{A}}$}
			\For{$k \in [N]$}
			\For{$t \in \curly{t\given \ba_t = \ba^{A}}$}
			\State {\small $y_t  \leftarrow g(\bs_t)- \gamma \alpha \log(\hat{\pi}(\bs_{t+1}, \ba^{A}))+\gamma h(\bs_{t+1}) $} 
			\EndFor
			\State Update $h$ using $\curly{y_t}_{ \curly{t\given \ba_t = \ba^{A}}}$ and $\curly{s_t}_{ \curly{t\given \ba_t = \ba^{A}}}$.
			\EndFor
			\State $\hat{Q}(\bs, \ba) \leftarrow \alpha \log(\hat{\pi}(\bs, \ba)) - \alpha\log(\hat{\pi}(\bs, \ba^{A})) + h(\bs)$
			\State \Return $\hat{Q}(\bs, \ba)$
		\end{algorithmic}
	\end{algorithm}

\subsubsection{Identifying $Q(\bs,\ba^A)$}
Let $Q^A(\bs) := Q(\bs, \ba^A)$. We define our $Q^A(\bs)$ estimator as a fixed point to an operator: 
\begin{definition}
Let $f(s) := Q(\bs, \ba^{A})$ and $C(\mathcal{S})$ the set of continuous bounded functions $f: \mathcal{S} \longmapsto \mathds{R}$ . Under Assumption~\ref{asm:identification}, $f(\bs)$ is the unique solution to
\eq{contraction}{
\hat{Q}^A(\bs) = \hat{\mathcal{T}}\hat{Q}^A(\bs),
\tag{\textsc{Q$^A$-Estimator}} 
}
where $\hat{\mathcal{T}}$ is an operator on the set of continuous bounded functions $f: \mathcal{S} \longmapsto \mathds{R}$. $\hat{\mathcal{T}}$ is defined as
\eqs{    \hat{\mathcal{T}}f(\bs) := g(\bs)+\gamma \hat{\EE}_{\bs'}{\left[ -\alpha \log(\hat{\pi}(\bs', \ba^{A})) + f(\bs') \given \bs, \ba^{A} \right]}.
} 
$\hat{\mathds{E}}_{\bs'}$ denotes the estimated expectation over one-step transition of state variables.
\end{definition}

%
The definition of $\hat{\mathcal{T}}$ follows from inserting $\ba^A$ into \ref{eq:main}.
When $\hat{\EE}_\bs'$ in $\hat{\mathcal{T}}$ is exactly calculated, $\hat{\mathcal{T}}$ becomes a contraction, which guarantees the uniqueness of \ref{eq:contraction} (Lemma~\ref{lem:fix} in the Supplements).
We derive \ref{eq:contraction} by repeatedly applying the operator $\hat{\mathcal{T}}$. The $\hat{\mathcal{T}}$ operator is similar to the $Q$-function Bellman operator when solving an MDP with only one available action $\ba^{A}$. We will show that identifying $\hat{Q}^A(\bs)$ reduces to solving a simple one-action MDP.

\subsubsection{Identifying $Q(\bs,\ba)$}
We now derive the estimator to $Q(\bs,\ba)$.
\begin{definition}
Given anchor estimator $\hat{Q}^A(\bs)$ and policy estimator $\hat{\pi}(\bs,\ba)$, we define the $Q$-function estimator as
\eq{q}{
    \hat{Q}(\bs,\ba) := \alpha \log(\hat{\pi}(\bs,\ba)) - \alpha \log(\hat{\pi}(\bs,\ba^{A})) + \hat{Q}^A(\bs).
    \tag{\textsc{Q-Estimator}} 
}
\end{definition}

\ref{eq:q} uses the fact that $\phi(s)$ depends only on the state variable. Accordingly, the difference $Q(\bs,\ba^{A}) - Q(\bs,\ba)$ is equal to the shaped $Q$-function difference $Q'(\bs,\ba^{A}) - Q'(\bs,\ba)$. Thus, knowing $Q(\bs,\ba^{A})$ is sufficient to correctly identify $Q(\bs,\ba)$ for any $(\bs,\ba) \in \mathcal{S} \times \mathcal{A}$. When the environment transition is unknown, we estimate $Q^A(\bs)$ using fitted-$Q$-iteration~\citep{riedmiller2005neural}. The algorithm uses a deep function for $\hat{Q}^A(\bs)$, and implements the operator $\hat{\mathcal{T}}$ by repeatedly leveraging the observed demonstrations. With \ref{eq:contraction}, we can 
easily derive \ref{eq:q}. The 2-step procedure for the $Q$-function is summarized as the Fitted-$Q$-Iteration Identification (FQI-I) method in Algorithm~\ref{alg:fqii}.
Again, FQI-I reduces to solving a simple one-action MDP. Note that when $\hat{\pi}(\bs,\ba)$ or $\hat{\pi}(\bs,\ba^A)$ are too small, the $\log$ in \ref{eq:q} will explode. In practice, one may resort to clipping by capping the policy probability~\citep{ionides2008truncated}.


\subsection{Full Algorithm}
\label{sec:alg}
We now tie the PQR components together and provide our overall framework in Algorithm~\ref{alg:main}. Our method has three steps: 
\begin{inlineenum}
    \item estimate the policy function by AIRL or other policy estimation methods;
    \item estimate $Q(\bs, \ba)$ using the FQI-I method in Algorithm~\ref{alg:fqii};
    \item estimate the reward function using Algorithm~\ref{alg:reward}.
\end{inlineenum}
Note that the deep neural network method used in the proposed algorithms can be replaced by many machine learning methods~\citep{nassif2012RDP, kuusisto2014support,athey2019estimating}.

\begin{algorithm}[t]
		\caption{PQR Algorithm}
		\label{alg:main}
		\begin{algorithmic}[1]
			\Input $\mathds{X} = \curly{\bs_0, \ba_0,  \cdots, \bs_T, \ba_T}$.
			\Input $\gamma$, $\alpha$ and $N$.
			\Output $\hat{r}(\bs, \ba)$.
			\State $\hat{\pi}(\bs, \ba) \leftarrow $\Call{AIRL}{$\mathds{X}$}
			\State $\hat{Q}(\bs, \ba) \leftarrow $\Call{FQI-I}{$\mathds{X}$, $N$, $\hat{\pi}(\bs, \ba)$}
			\State $\hat{r}(\bs, \ba) \leftarrow $\Call{RE}{$\mathds{X}$, $\hat{\pi}(\bs, \ba)$, $\hat{Q}(\bs, \ba)$}
			\State \Return $\hat{r}(\bs, \ba)$.
		\end{algorithmic}
	\end{algorithm} 

 \begin{figure*}[t]
\begin{center}
	\centering
	\begin{subfigure}{0.23\textwidth}
		\centering
		\includegraphics[scale=0.18]{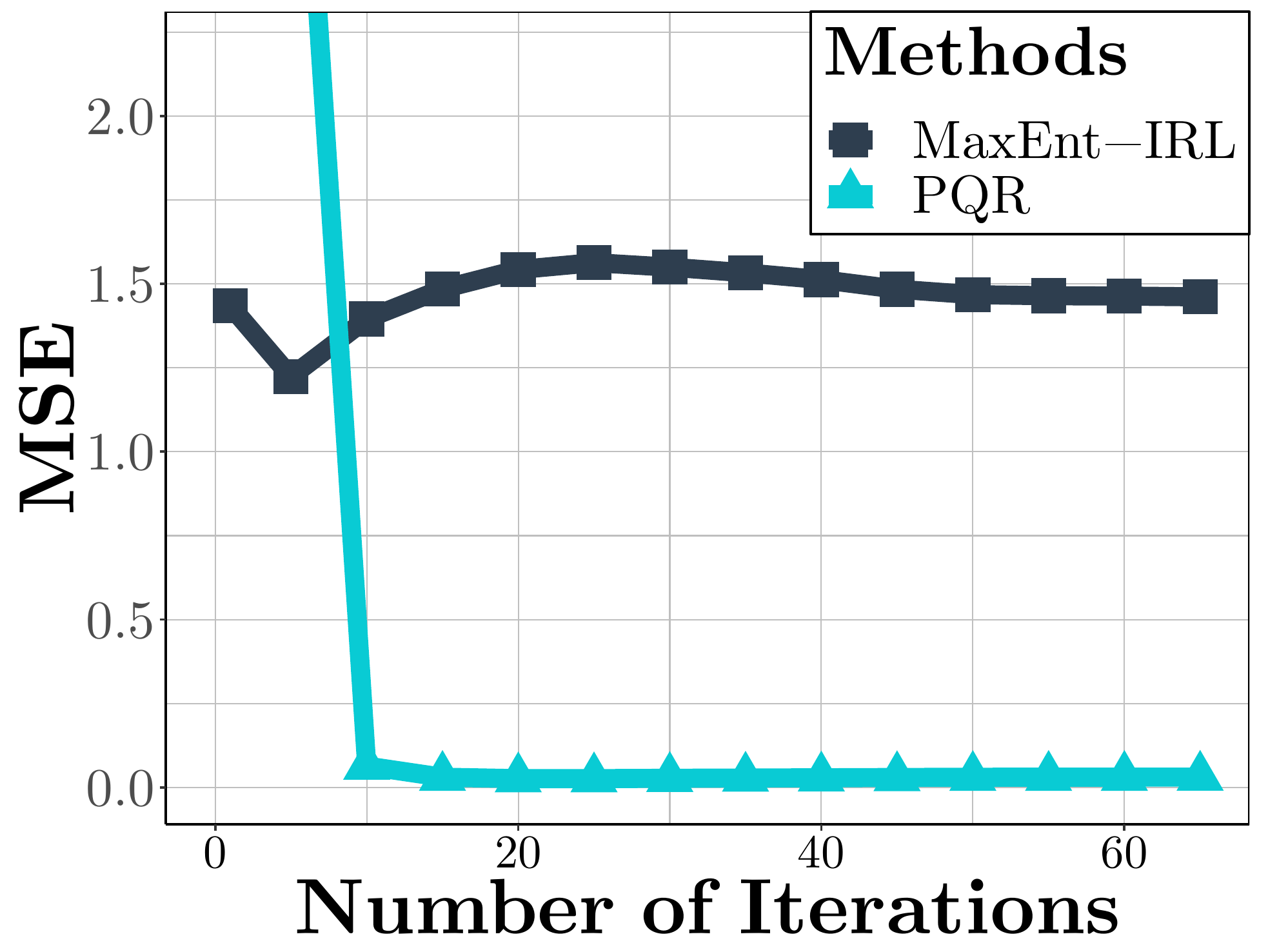}
		\caption{$p = 5$}
	\end{subfigure}
	\centering
	\begin{subfigure}{0.23\textwidth}
		\centering
		\includegraphics[scale=0.18]{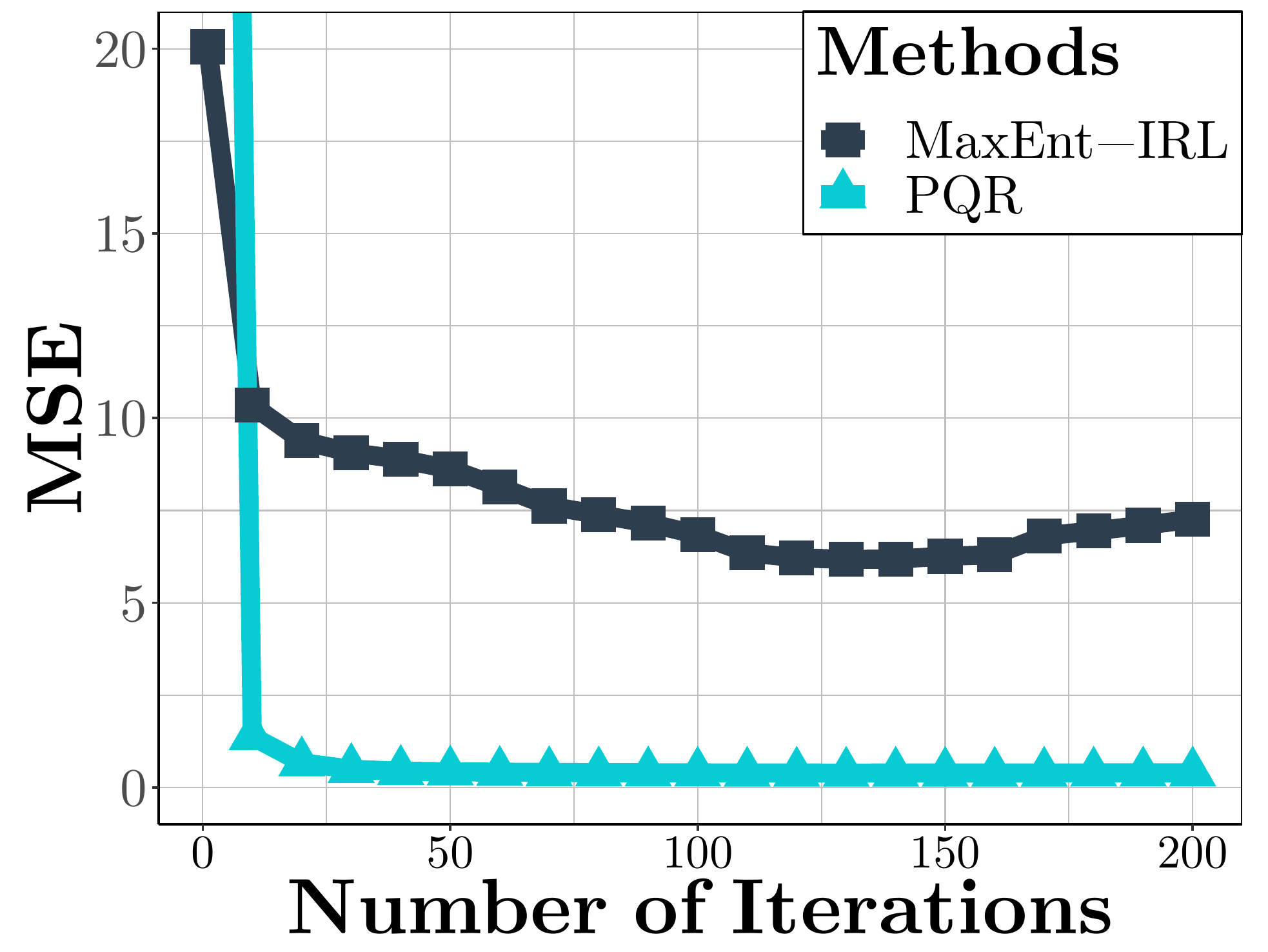}
		\caption{$p = 10$}
	\end{subfigure}
	\centering
	\begin{subfigure}{0.23\textwidth}
		\centering
		\includegraphics[scale=0.18]{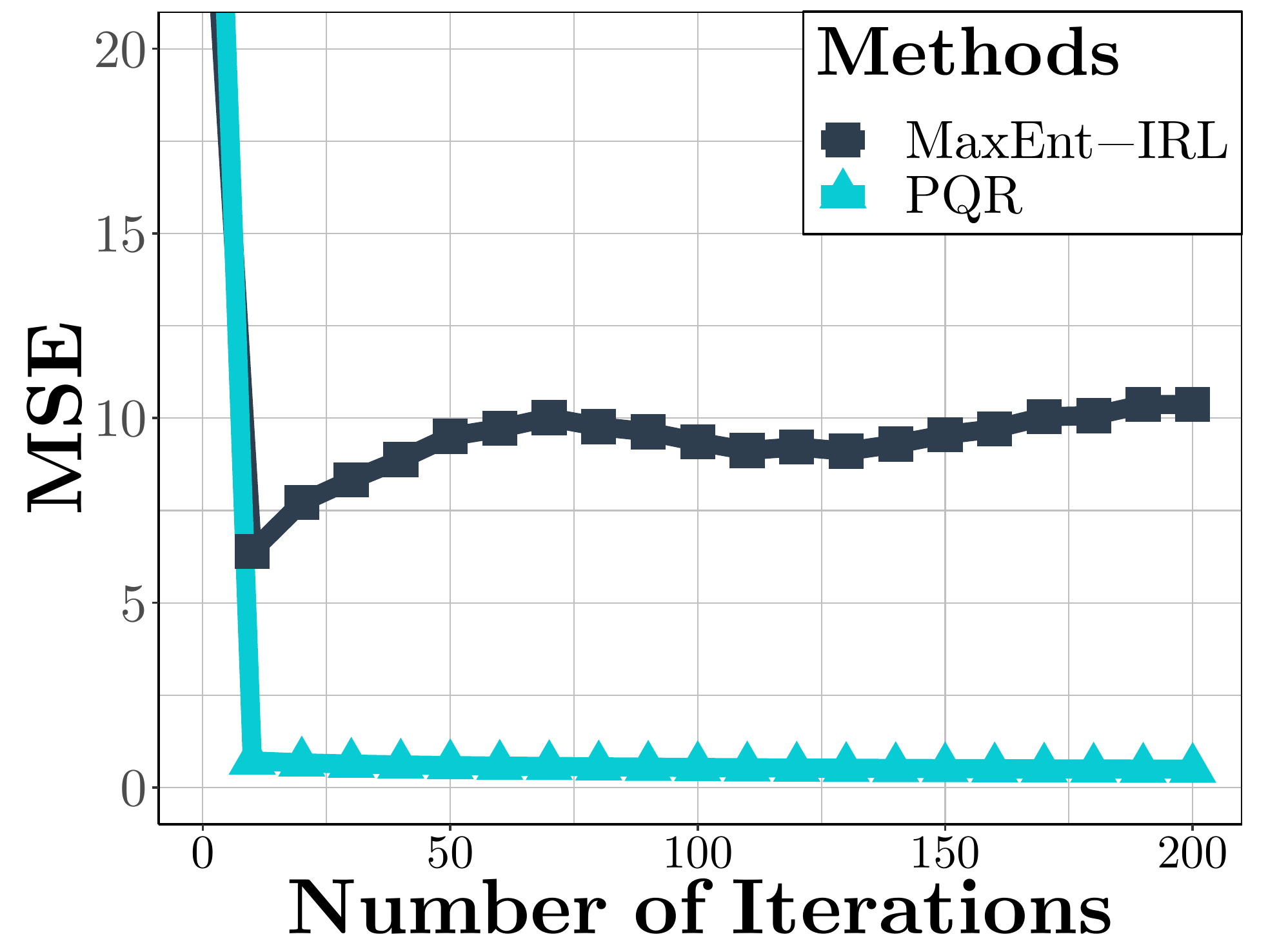} 
		\caption{$p = 20$}
	\end{subfigure} 
	\centering
	\begin{subfigure}{0.23\textwidth}
		\centering
		\includegraphics[scale=0.18]{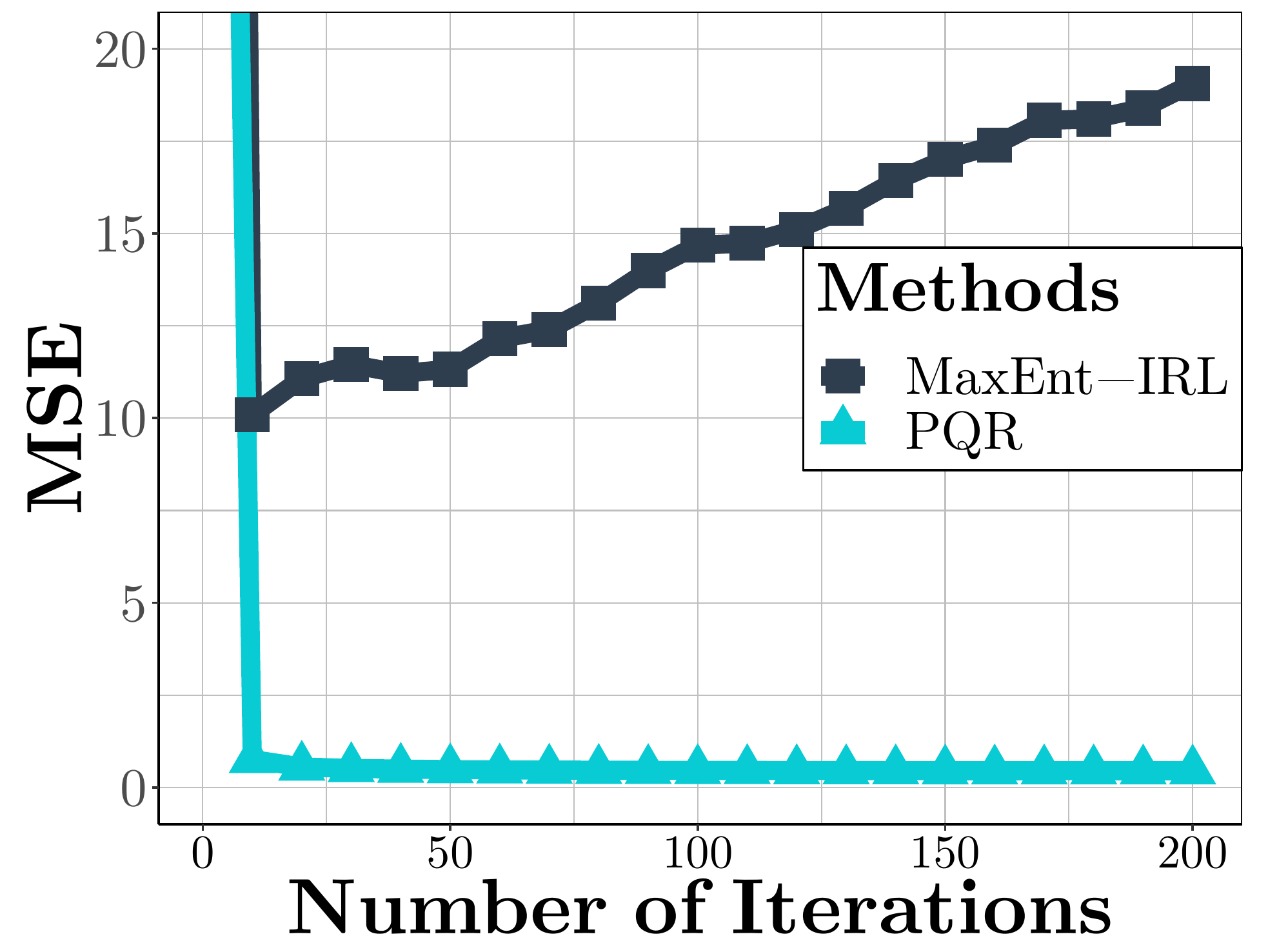}
		\caption{$p = 40$}
	\end{subfigure} 
	\centering
	\caption{MSE for $Q$-function recovery with different state variable dimensions $p$}
	\label{fig:q} 
\end{center}
\end{figure*}

\section{Theoretical Analysis}
\label{sec:error}
We theoretically study the estimation error of PQR. We first demonstrate the intuition behind PQR by focusing on a simplified setting with \begin{inlineenum}
    \item known transitions, and
    \item an accurate estimation to the policy function.
\end{inlineenum}
Then, we study the more general setting and provide a convergence analysis. All the proofs and required extra technical assumptions are deferred to Section~\ref{sec:theoretical-results} of the Supplements. 


\subsection{Estimators Error with Known Transitions}
\label{sec:asym}
In this section we assume that the environment transition is given and that 
that the policy estimator $\hat{\pi}(\bs,\ba)$ is accurate.
We show that the \ref{eq:main} accurately recovers the true reward function without identification issues. By using this idealized setting, we focus on the errors of the proposed estimators.

\begin{theorem}
\label{thm:asym}
Let $Q(\bs,\ba)$ and $r(\bs,\ba)$ denote the true $Q$ and reward function. Let $\hat{r}(\bs,\ba)$, $\hat{Q}^{A}(\bs)$ and $\hat{Q}(\bs,\ba)$ be the \ref{eq:main}, \ref{eq:contraction}, and \ref{eq:q}, whose  expectations are exactly calculated. Assume that the policy estimator $\hat{\pi}(\bs,\ba)$ is accurate. 
Then, under the formulation in Section~\ref{sec:PQR} with Assumption~\ref{thm:asym},
\begin{itemize}
    \item $\hat{Q}^{A}(\bs) = Q(\bs,\ba^{A})$;
    \item $\hat{Q}(\bs,\ba) = Q(\bs,\ba)$;
    \item $\hat{r}(\bs,\ba) = r(\bs,\ba)$.
\end{itemize}\end{theorem}

Once the transition information is known, \ref{eq:main} is able to uniquely recover the true reward function. In other words, the error of estimators comes from the estimated expectation over the transitions. This justifies the intuition behind PQR. 

\subsection{PQR Error with Unknown Transitions}
\label{sec:nonasym}
Next, we study the errors induced by PQR when the environment transition is unknown. 
 As we suggested in Section~\ref{sec:estimation}, Algorithm~\ref{alg:reward} and Algorithm~\ref{alg:fqii} requires deep supervised learning and deep-FQI. While extensively used in practice, these methods are not yet endowed with theoretical guarantees: theoretical analysis is very challenging due to the nature of deep learning. Existing theoretical work inevitably calls for a series of extra assumptions. We follow the assumptions in \citet{munos2008finite,du2018gradient,arora2019fine,yang2019theoretical}. For ease of presentation, we provide the minimal assumptions required to clarify the results. Other assumptions are deferred to Section~\ref{sec:proof-nonasym} of the Supplements. By Assumption~\ref{asm:policy} and Assumption~\ref{asm:deep}, we assume that the error of the log policy estimation from existing methods like AIRL is bounded, and constrain the training data.   
 \begin{assumption}
 \label{asm:policy}
 Let $\hat{\pi}(\bs, \ba)$ be the estimated policy function in Algorithm~\ref{alg:main}. The estimation error can be bounded by
$\max_{(\bs,\ba) \in \mathcal{S} \times \mathcal{A}}\abs{\log(\hat{\pi}) - \log(\pi)} \leq \epsilon_{\pi}.$
 \end{assumption}
\begin{assumption}
\label{asm:deep}
Let there be $n$ IID training data extracted from $\mathds{X} = \curly{\bs_0, \ba_0, \bs_1, \ba_1, \cdots, \bs_T, \ba_T}$ for the neural networks in Algorithm~\ref{alg:reward} and Algorithm~\ref{alg:fqii}. Specifically, we use $\curly{(\bx_i, y_i)}_{i=1}^n$ with $\bx_i \in\mathcal{S}\times \mathcal{A}$ to denote the training data for Algorithm~\ref{alg:reward}, and $\curly{(\bx_i, y^k_i)}_{i=1}^n$ for the $k^{\text{th}}$ iteration of Algorithm~\ref{alg:fqii}.  
\end{assumption}

\begin{figure*}[t]
\begin{center}
	\centering
	\begin{subfigure}{0.23\textwidth}
		\centering
		\includegraphics[scale=0.18]{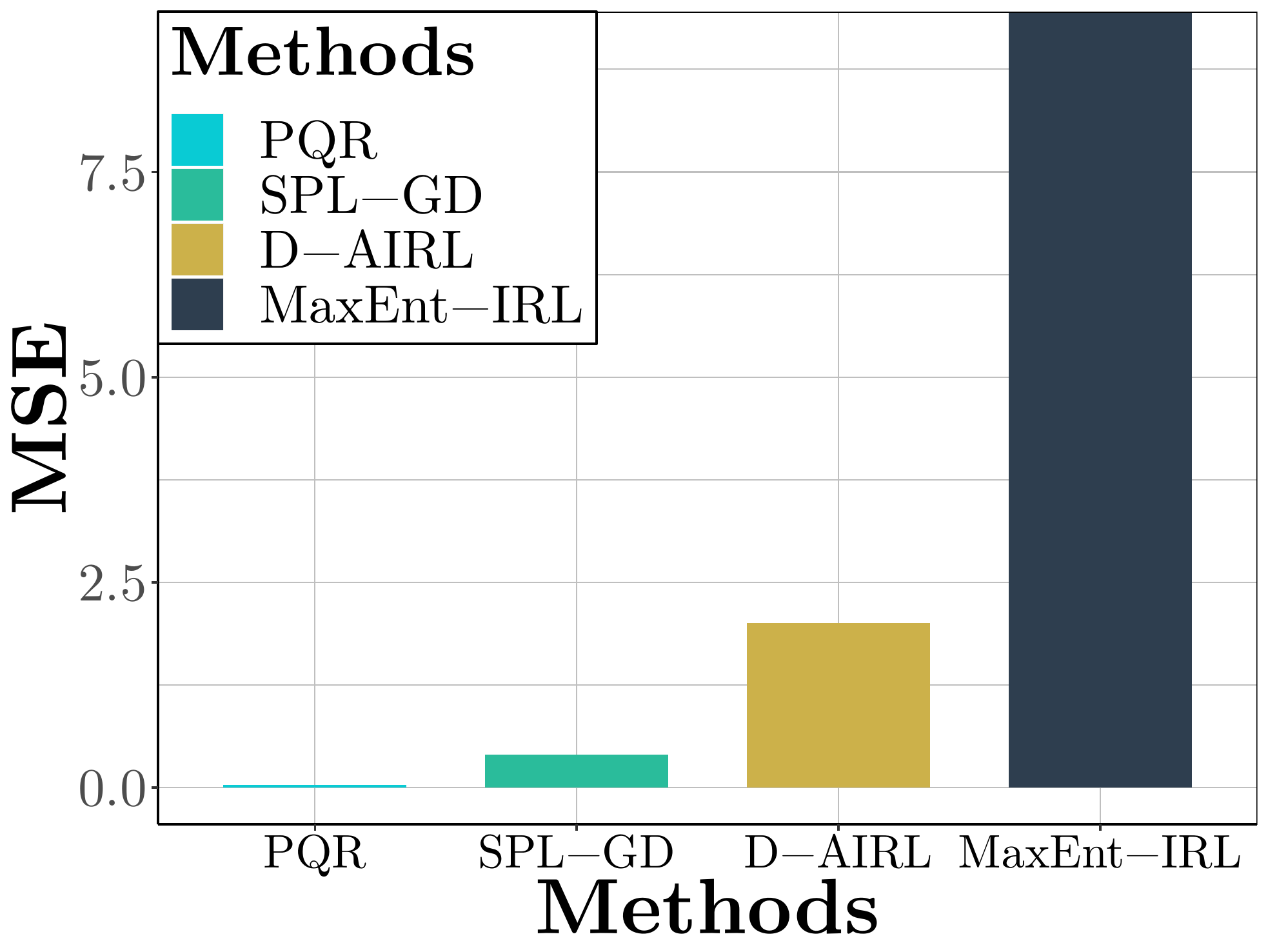}
		\caption{$p = 5$}
	\end{subfigure}
	\centering
	\begin{subfigure}{0.23\textwidth} 
		\centering
		\includegraphics[scale=0.18]{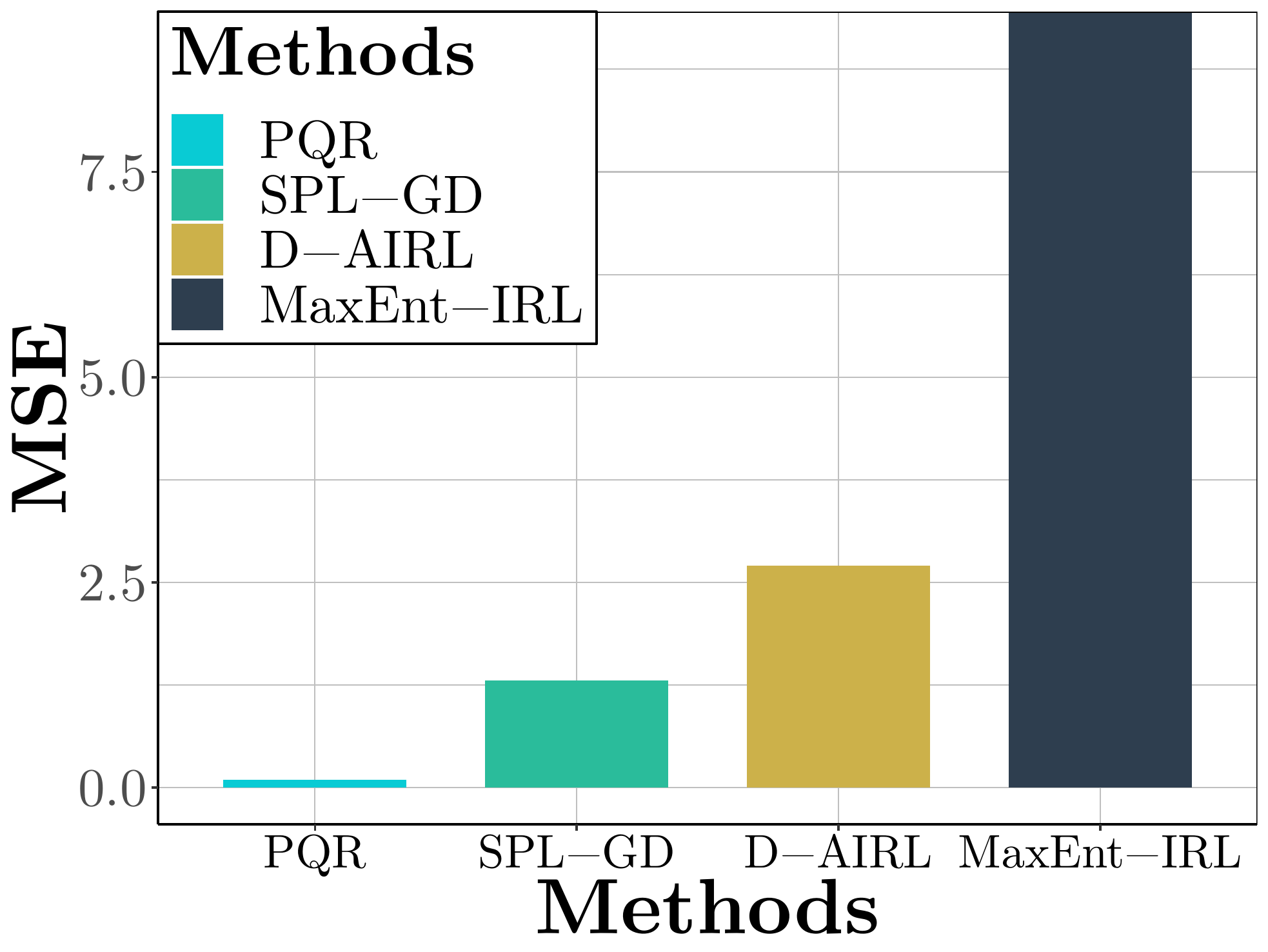}
		\caption{$p = 10$}
	\end{subfigure}
	\centering
	\begin{subfigure}{0.23\textwidth}
		\centering
		\includegraphics[scale=0.18]{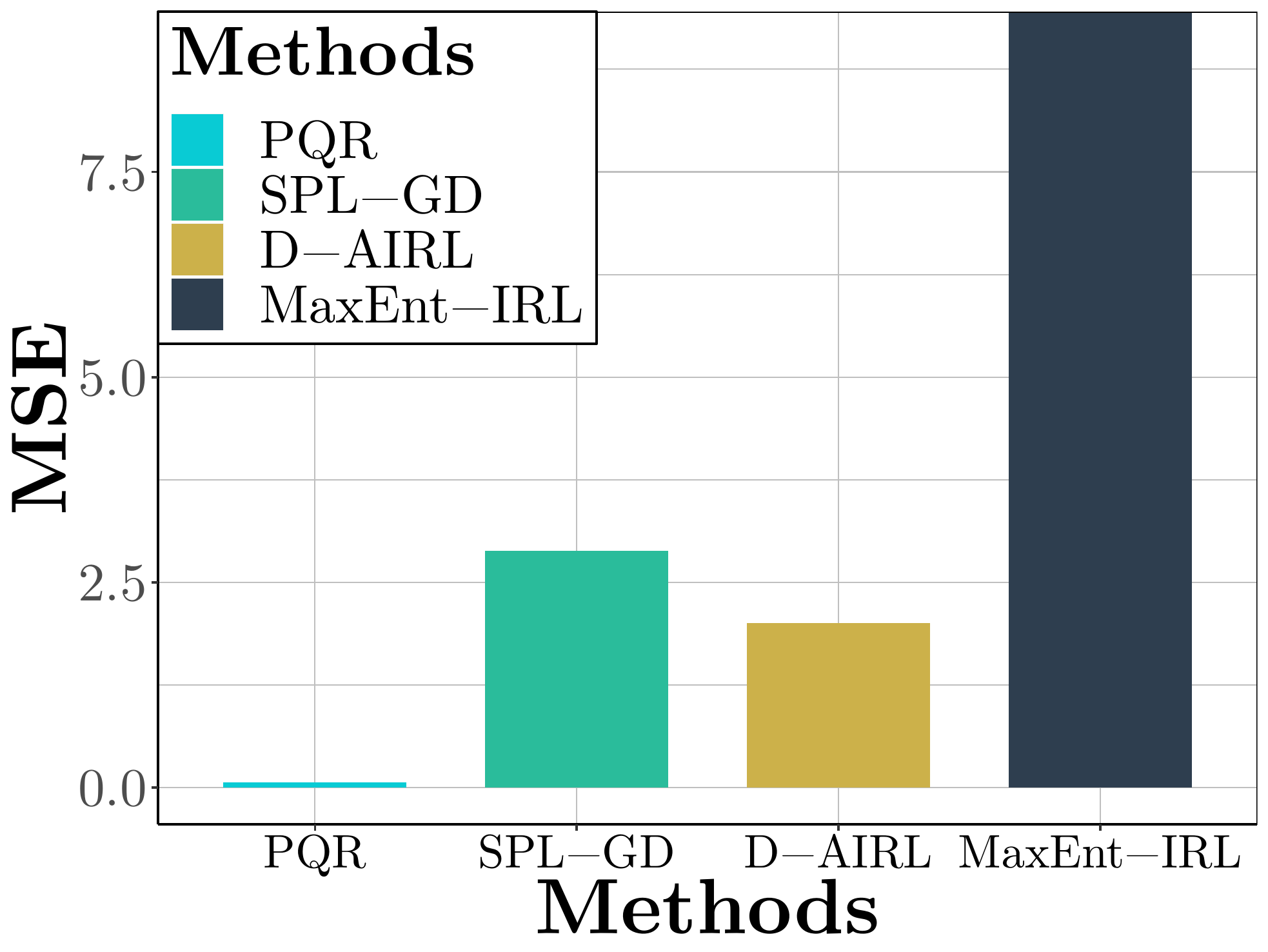}
		\caption{$p = 20$}
	\end{subfigure}
	\centering
	\begin{subfigure}{0.23\textwidth}
		\centering
		\includegraphics[scale=0.18]{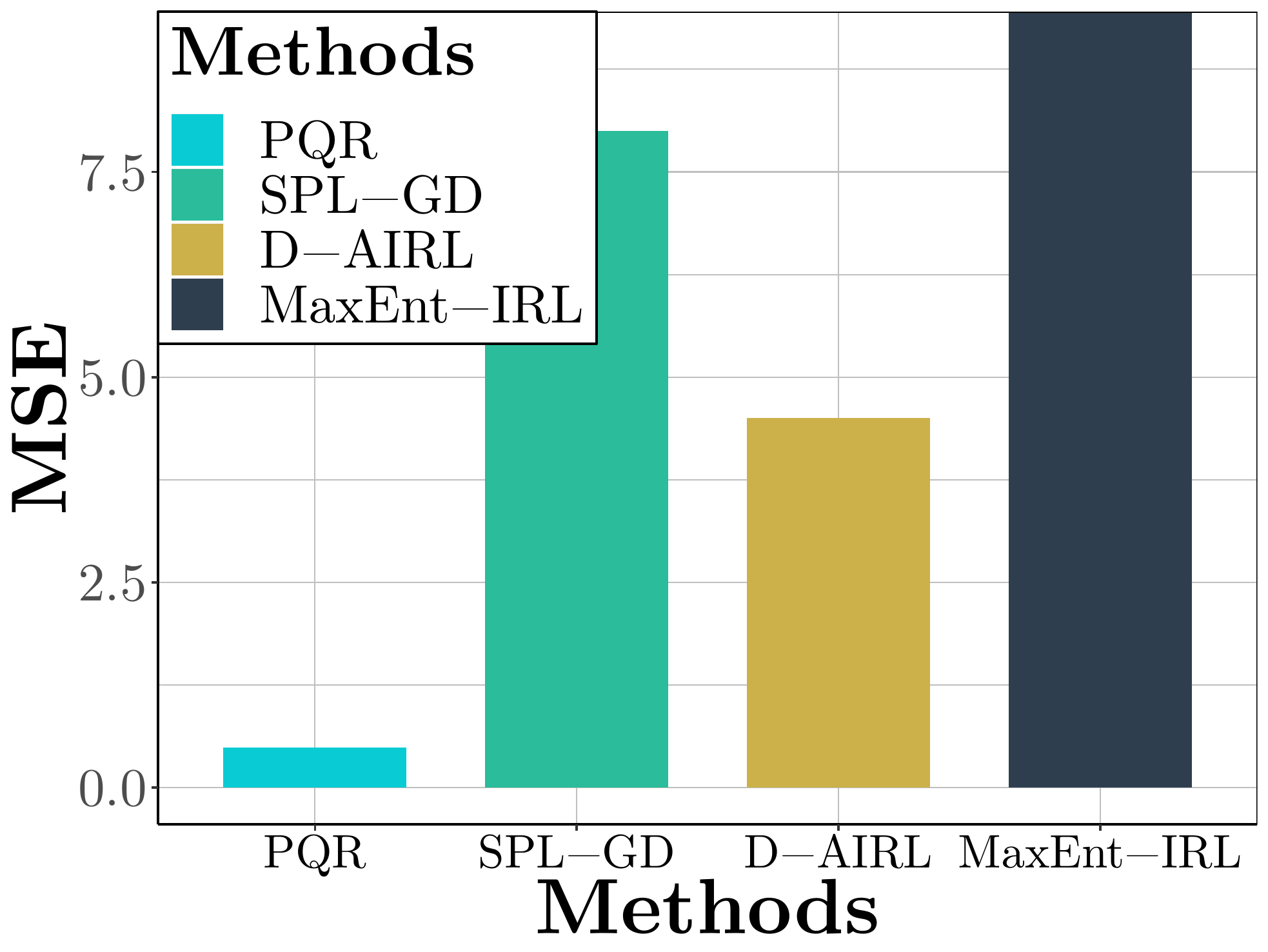}
		\caption{$p = 40$} 
	\end{subfigure}
	\centering
	\caption{MSE (truncated at $9$) for reward recovery with different state variable dimensions $p$}
	\label{fig:r} 
\end{center}
\end{figure*}

We now provide the convergence result for PQR.  
\begin{theorem}
\label{thm:nonasym}
Let $Q(\bs,\ba)$ be the true $Q$-function, $r(\bs,\ba)$ the true reward function, $\hat{Q}(\bs,\ba)$ the result of Algorithm~\ref{alg:fqii}, and $\hat{r}(\bs,\ba)$ the result of Algorithm~\ref{alg:main}.  Let Assumptions~\ref{asm:identification} to~\ref{asm:deep} and the assumptions listed in Section~\ref{sec:extra-assm} of the Supplements be satisfied. For any $P \in (0,1)$, there exists a constant $C>0$, such that the error of both $\hat{Q}(\bs,\ba)$ and $\hat{r}(\bs,\ba)$ can be bounded with probability at least $1-P$ by:
{\ali{q-error}{ 
\EE_{(\bs,\ba)\sim \mathcal{D_{\bs,\ba}}}& \left[\abs{\hat{Q}(\bs,\ba) - {Q}(\bs,\ba)}\right] 
 \\ &\leq  \frac{C}{1-\gamma}\max_{k \in[N]} \sqrt{\frac{2\by^T_k(\bH^\infty)^{-1} \by_k}{n}} 
\\&\quad +\gamma^N \frac{2C}{1-\gamma}+ O\left(
\sqrt{\frac{\log(\frac{n(N+2)C}{ P})}{n}} 
\right) 
\\&\quad +\frac{\gamma \alpha \epsilon_\pi}{1-\gamma},
}
\ali{reward-error}{
\EE_{(\bs,\ba)\sim \mathcal{D_{\bs,\ba}}} &[\abs{\hat{r}(\bs,\ba) - r(\bs,\ba)}] 
\\&\leq \frac{(1+\gamma)C}{1-\gamma}\max_{k \in[N]} \sqrt{\frac{2\by^T_k(\bH^\infty)^{-1} \by_k}{n}}
\\& \quad + \sqrt{\frac{2\by^T(\bH^\infty)^{-1} \by}{n}} 
+\gamma^N C(1+\gamma)+
\frac{2\alpha\epsilon_{\pi}}{1-\gamma}
\\&\quad +O\left(\sqrt{\frac{\log(\frac{n(N+2)C}{ P})}{n}} 
\right),
}}
where $\by = \curly{y_1, \cdots, y_n}^\top$, $\by^k = \curly{y_1^k, \cdots, y_n^k}^\top$ and the components of $\bH^\infty$ are defined as 
\eqs{    H^\infty_{ij} = \frac{\bx_i^T\bx_j (\pi-\arccos(\bx_i^\top \bx_j))}{2\pi}. 
}
\end{theorem}

Following the rationale in \citep{arora2019fine}, $\sqrt{\frac{2\by^T(\bH^\infty)^{-1} \by}{n}}$ and $\max_{k \in[N]} \sqrt{\frac{2\by^T_k(\bH^\infty)^{-1} \by_k}{n}}$ quantify the generalization errors of the neural networks used in Algorithms~\ref{alg:reward} and~\ref{alg:fqii}, respectively. For the $Q$-function estimation, the error of the neural network in Algorithm~\ref{alg:fqii} is amplified by $\frac{C}{1-\gamma}$ due to the accumulation in each iteration, according to \eqref{eq:q-error}. However, this coefficient is smaller than that of a general FQI method, which is $\frac{2\gamma C}{(1-\gamma)^2}$ (Theorem 4.5 of \citet{yang2019theoretical}). This is consistent with the fact that Algorithm~\ref{alg:fqii} solves a much simplified MDP with only one available action. For the reward estimation, the policy estimation error $\epsilon_{\pi}$ is amplified by the coefficient $\frac{2\alpha}{1-\gamma}$, according to \eqref{eq:reward-error}. $\epsilon_{\pi}$ does not significantly accumulate or explode to harm the proposed reward estimation.

\section{Experimental Results}
\label{sec:exp}

We now demonstrate the utility of PQR.
Throughout this section, we assume that $g(\bs) = 0$ in Assumption~\ref{asm:identification}, since this is the most common case in practice. 

\paragraph{Competing Methods}
For IRL, we consider the classic MaxEnt-IRL proposed in \citet{ziebart2008maximum}, which estimates the $Q$-function and treats it as the reward function. The estimated $Q$-function faces the identification issue described in Section~\ref{sec:identification}. 
Further, we include the D-AIRL method~\citep{fu2017learning}, which attempts to distinguish the reward function from the $Q$-function by a disentangling procedure. D-AIRL assumes that the reward does not depend on actions, and that the transition is deterministic. As for DDCs, we consider the Simultaneous Planning and Learning - Gradient Descent (SPL-GD) method~\citep{ermon2015learning}. SPL-GD directly formulates the reward function as a linear function. 

\subsection{Synthetic Experiments}
\label{sec:synthetic}
For simulation, we build an MDP environment with $p$-dimensional state variables taking values in $\mathds{R}^p$, one-dimensional action variables, and a nonlinear reward function. We take $\delta = 0.9$ and $\alpha = 1$ and sovle the MDP with a deep energy-based policy by the soft $Q$-learning method in \citet{haarnoja2017reinforcement}. 
By conducting the learned policy for $50000$ steps, we obtain the demonstration dataset, on which we compare PQR, MaxEnt-IRL, AIRL, and SPL-GD. We assume that $\gamma$ and $\alpha$ are known as required by existing methods. The detailed data-generation procedure is provided in Section~\ref{sec:syn-details} of the Supplements. We let $\ba^A = 0$ which means $r(\bs,0) = 0$.  

\paragraph{Q Recovery}
We first consider $Q$-function estimation. 
We compare our method with MaxEnt-IRL which also generates an estimated $Q$-function. 
The results are summarized in Figure~\ref{fig:q}. 
With sufficient training iterations, Algorithm~\ref{alg:fqii} provides a much more accurate $Q$-function estimation. This result is consistent with the analysis in Section~\ref{sec:estimation} and also \citet{fu2017learning}: without a proper identification method, $Q$-function estimation may not recover the true $Q$-function. On the other hand, when not well trained, the proposed procedure provides poor estimates. This is not surprising since FQI-I outputs the sum of three deep neural networks (line 9 of Algorithm~\ref{alg:fqii}) and thus has a bigger error at an early stage of training.

 \begin{figure*}[t]
\begin{center}
\begin{minipage}[t]{0.48\textwidth} 
	\centering
			\includegraphics[scale=0.16]{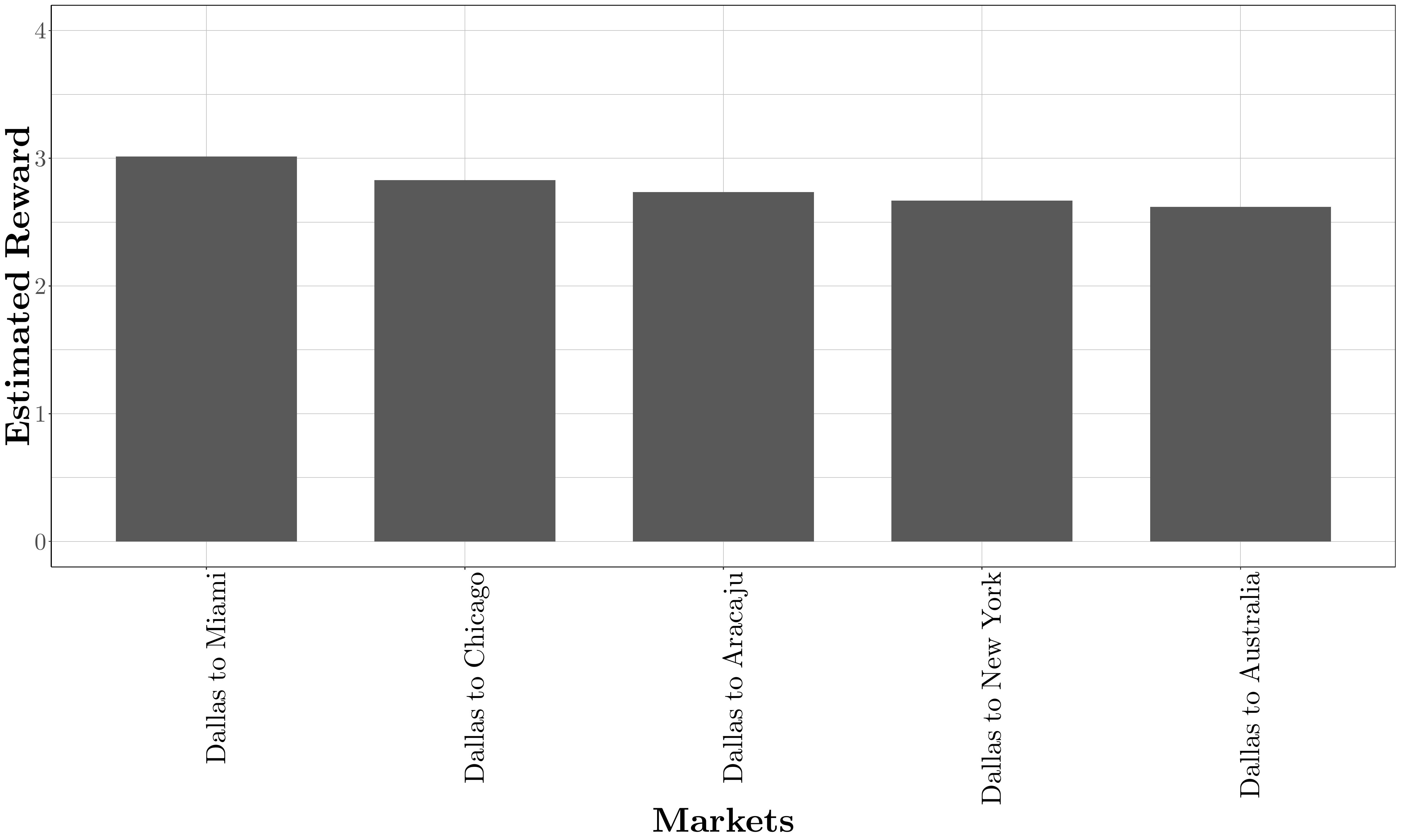}
		\centering 
	\centering
	\caption{Estimated Reward for American Airlines} 
	\label{fig:airline-main}
\end{minipage}
\begin{minipage}[t]{0.48\textwidth}
	\centering
    \includegraphics[scale=0.16]{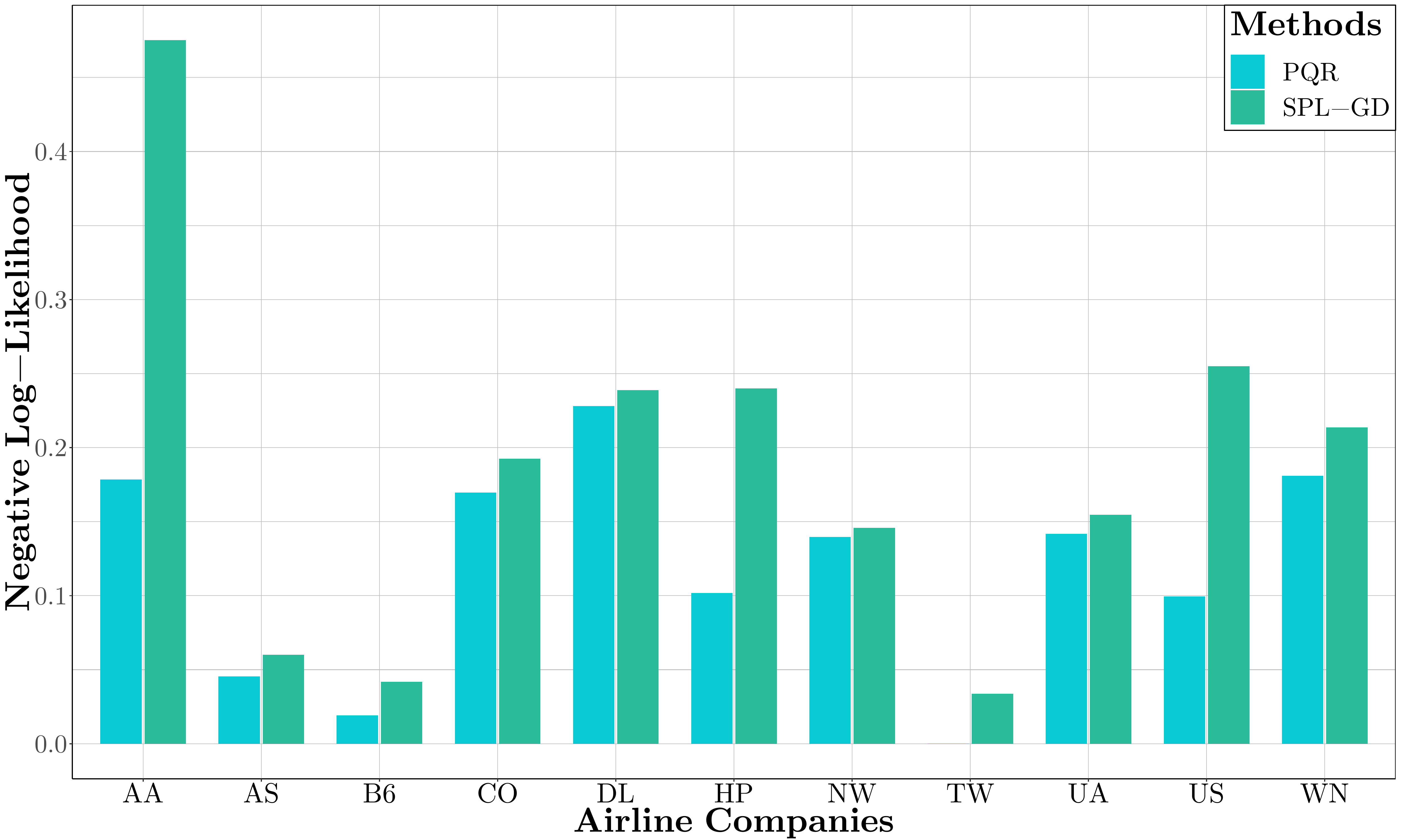} 
	\centering
	\caption{Negative log-likelihood for the airline market entry analysis.}
    \label{fig:airline-likelihood}
\end{minipage} 
\end{center}
\end{figure*}

\paragraph{Reward Recovery} Since only the proposed method uses the information that $r(\bs,0) = 0$, for a fair comparison, we ground all the other reward estimators by the estimated $r(\bs,0)$. In other words, by this normalization procedure, all the methods provide reward estimates with zero reward at the anchor action, and the estimated reward function should have the same scale as the true reward function. Figure~\ref{fig:r} reports the results of reward function estimation. 
PQR performs the best and provides the most accurate estimates. Reward recovery may not be a fair task for MaxEnt-IRL, since it estimates the $Q$ function instead of the reward function. It also performs the worst of the methods considered. D-AIRL improves upon MaxEnt-IRL by attempting to disentangle the reward from the environment effects. However, since the assumptions required (deterministic transitions and state-only reward functions) are too strong, the estimates are still inaccurate. SPL-GD assumes a linear reward function, and performs worse than PQR in our nonlinear-reward setting. As the dimension of the problem increases, the error increases. Our results show that PQR provides accurate Q-function and reward function estimates. 

Note that the performance of the proposed method may not be guaranteed if the anchor-action assumption is not satisfied (such as the setting of D-AIR with state-only reward functions). More detailed analysis with experiments is provided in Section~\ref{sec:robust-analysis} of the Supplements.
We conduct sensitivity (to $\alpha$ and $\delta$), in Section~\ref{sec:sensitivity-analysis}. An $\alpha$ selection procedure is provided in Section~\ref{sec:alpha} of the Supplements. 
 PQR may benefit tasks like imitation learning and agent alignment by identifying the reward function accurately. We leave such extensions to future work.

\subsection{Airline Market Entry Analysis}
\label{sec:airline}

After demonstrating the performance of PQR on a synthetic, machine-learning setting, we use the method in a heavily-studied setting in economics. Specifically, we consider the dynamic market entry decisions of airline carriers, where markets are defined as unidirectional city pairs~\citep{berry2010tracing,manzanares2016essays}. For example, Denver to Chicago is one market. Airline entry competition has been studied extensively in economics and other literature, and it is well-known that these problems are computationally hard to solve. In fact, among the top $60$ composite statistical areas (CSA's) alone, there are $60\times 59/2 = 1770$ different markets. 
We combine public data from the Official Activation Guide\footnote{\url{https://www.oag.com/}} for scheduled flight information, with public data from the Bureau of Transportation Statistics\footnote{\url{https://www.transtats.bts.gov}} for airline company information. 

First, we follow~\citet{benkard2004dynamic} to aggregate airports by Composite Statistical Areas (CSAs) where possible, and Metropolitan Statistical Areas (MSAs) otherwise. We focus on the top $60$ CSAs with the most itineraries in 2002, and thus consider $1770$ markets. The action variable $\bA_t$ is to select markets from the $1770$ options: select entry or not for each market. We study 11 airline companies, see Section~\ref{sec:real-world-sup} for a list of considered CSAs and airline companies. The considered state variables include origin/destination city characteristics, company characteristics, competitor information for each market and so on. We assume airline companies make decisions given the observed state information, in order to maximize their reward function.  
\begin{figure}[t]
    \centering
    \includegraphics[scale = 0.35]{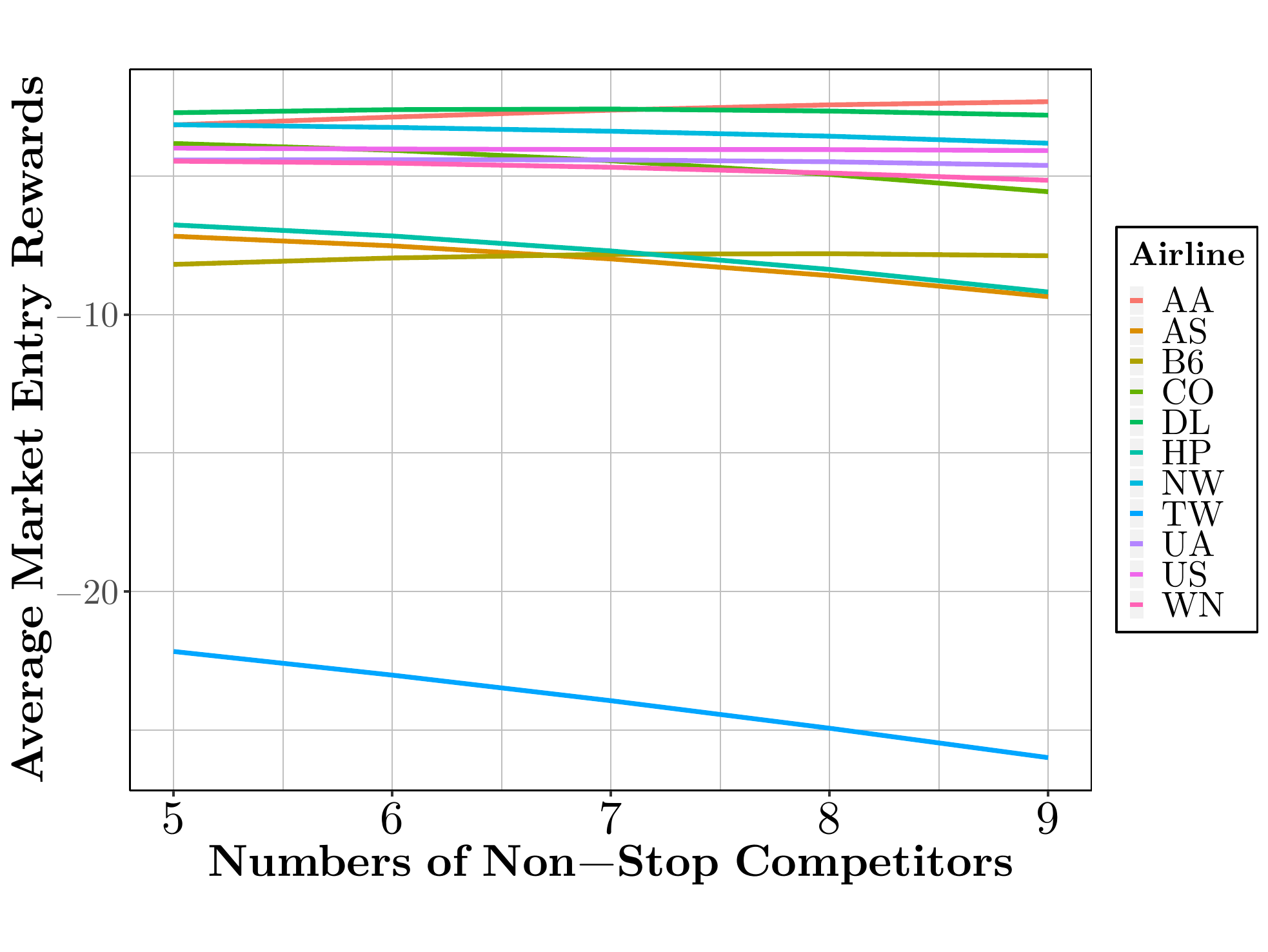}
    \caption{Average market entry rewards v.s. numbers of non-stop competitors.}
    \label{fig:reward}
\end{figure}

 \paragraph{Reward Estimation}
We implement the proposed PQR method on the airline dataset (including the recorded state variables and airline market entry history) to estimate the reward function. Intuitively, if an airline company decides to not enter a specific market, there will be no reward from this decision, which acts as the anchor action. 
We then apply the estimated reward function to each company for each potential market, to estimate the profit an airline company could make by entering a specific market.  
Figure~\ref{fig:airline-main} plots the top five PQR-estimated reward markets of American Airlines. The top five most profitable markets are related to Dallas, where American Airlines has its biggest hub. The hub can reduce the cost of entering connecting markets, and thus generate higher rewards~\citep{berry1996airline}, which is consistent with our empirical results.

\paragraph{Behaviour Prediction} We predict the decision-making of airline companies using the estimated reward function, and compare our method to SPL-GD in Figure~\ref{fig:airline-likelihood}. Our proposed method achieves a smaller negative log-likelihood than SPL-GD, reflecting a better prediction of the decision-making of airline companies.

\begin{figure}[t]
    \centering
    \includegraphics[scale = 0.35]{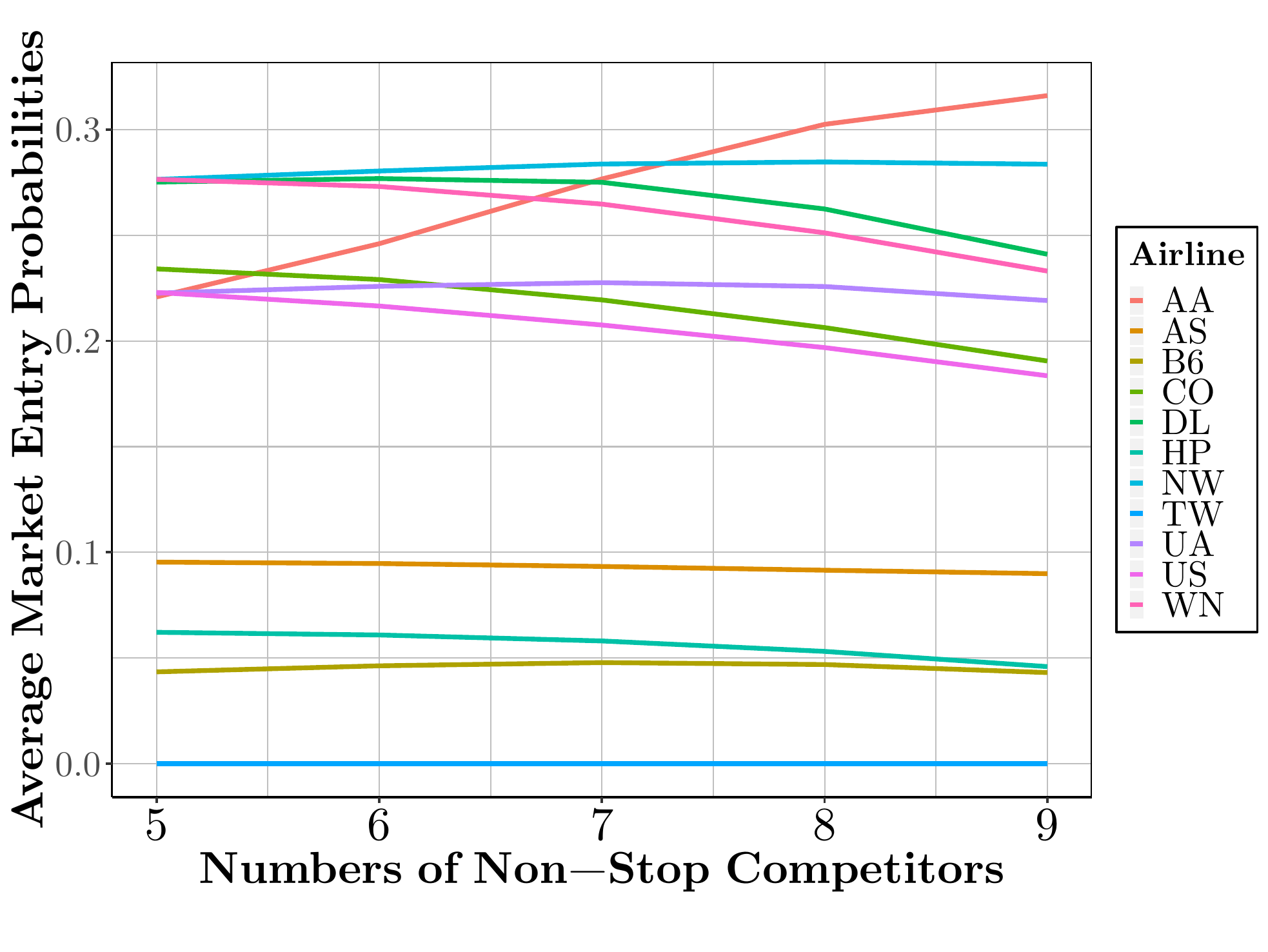}
    \caption{Average market entry probabilities v.s. numbers of non-stop competitors.}
    \label{fig:counterfactual}
\end{figure}
\paragraph{Counterfactual Analysis} We also conduct counterfactual analysis on the number of non-stop competitors using the estimated reward function. Specifically, we change the number of non-stop competitors of each market in year 2015 from $5$ to $9$. Then, we calculate the average entry reward and predict the average market entry likelihood over the $1770$ markets for each airline company. The results are summarized in Figure~\ref{fig:reward} and Figure~\ref{fig:counterfactual}.
Note that, as the number of competitors increases, most airline companies tend to choose not to enter the market. This is consistent with the analysis in \citet{berry2010tracing}. This trend is more significant for smaller airline companies, which experience harsher negative consequences from competition, given traditionally small profit margins. 

American Airlines is the only exception, as its chances of entering markets appear to increase with additional competitors. This phenomenon suggests that there exists unobserved confounders distorting the effect of the number of non-stop competitors. This motivates combining PQR with an instrumental variable (IV) identification strategy for modeling non-stop competition, which in principle would eliminate the confounder impact. See \citet{kuang2019mendelian, kuang2020ivy} for an example of an IV identification strategy. We leave the combination of IV strategies with PQR for future work.  

\section{Conclusions}
We propose a novel Policy-$Q$-Reward method to uniquely identify and estimate reward functions in IRL. 
We provide effective estimators that use the anchor-action assumption to recover the true underlying reward function, under both known and unknown transitions. 
We demonstrate the utility of PQR using both synthetic and real-world data. While we focus on reward function estimation, PQR can also be applied to imitation learning in the machine learning community and counterfactual analysis in the economics community.  

\section*{Acknowledgements}
The authors would like to thank Amazon Web Services for providing computational resources for the experiments in this paper. We acknowledge the Amazon Prime Economics team as well as Dr. Zongxi Li for very helpful discussions. Dr. Max Reppen is partly supported by the Swiss National Science Foundation grant SNF 181815.

\bibliography{example_paper}
\bibliographystyle{plainnat}

\newpage\clearpage\onecolumn
	\appendix
	\addcontentsline{toc}{section}{Appendices}

\section*{Supplements}
First we provide more detailed information about a closely-related economic model, dynamic discrete choice models, is deferred to Section~\ref{sec:sup-ddc}. Then, theoretical details including the proof and extra assumptions for our main results are provided in Section~\ref{sec:theoretical-results}. We provide extended experiment details in Section~\ref{sec:experiments}.

\section{Dynamic Discrete Choice Models}
\label{sec:sup-ddc}
In this Section, we detail the formulation of DDCs~\citep{rust1987optimal} to support our argument in Section~\ref{sec:related-models}. The results here are mainly a review of existing work. 
DDCs assume that agents make decisions following a Markov decision process described by the tuple $\curly{ \curly{\mathcal{S}, \mathcal{E}}, \mathcal{A},\textnormal{P},\gamma, r}$, where
\begin{itemize}
\item $\curly{\mathcal{S}, \mathcal{E}}$ denotes the space of state variables. 
\item $\mathcal{A}$ is a set of $J$ actions as the action space.
\item $r$ is the reward function.
\item $\gamma \in [0,1)$ is the discount factor. 
\item $\textnormal{P}$ denotes the transition distribution.
\end{itemize} 

At time point $t$, agents observe state variable $\bS_t \in \mathcal{S}$, and $\be_t = \curly{e_{t1}, e_{t2},\cdots, e_{tJ}}\in\mathcal{E}$. While $\bS_t$ is observable and recorded in the dataset, $\be_t$ is not observable in the dataset, and is only known by the agent when making decisions at time point $t$. 

The decision variable is defined as a $J \times 1$ vector, $\bA_t = \curly{A_{t1}, A_{t2}, \cdots, A_{tJ}}^\top \in \mathcal{A}$, satisfying
\begin{itemize}
    \item $\sum_{j=1}^JA_{tj} = 1 $,
    \item $A_{tj} \in \curly{0,1}$.
\end{itemize}
Thus, the decision is indeed a selection over $J$ choices.

The control problem agents are solving is formulated by the following value function: 
\eq{value}{
V^{DDC}(\bs,\bepsilon) = \max_{\curly{\ba_t}_{t=0}^{\infty}} \EE \left[\sum_{t=0}^\infty \gamma^{t} r(\bS_t, \be_t, \ba_t) \given \bs , \be \right].  
}

The reward function of DDCs is defined by decomposing over different actions. 
    For $\bs\in \mathcal{S}$, $\be \in \mathcal{E}$, and $\ba\in \mathcal{A}$, the reward function  $r(\bs, \bepsilon, \ba)$ is defined by
\eqs{
    r(\bs, \be, \ba) = r^*(\bs, \ba) + \ba^\top \be,  
}
where $r^*$ is the estimation target.

\subsection{Assumptions}
\label{sec:assumptions}
We first discuss the assumptions for DDCs:

\begin{assumption}
\label{asm:trainsition}
 The transition distribution from $\bS_t$ to $\bS_{t+1}$ is independent from $\be_t$ 
 \eqs{
 \textnormal{P}(\bS_{t+1} = \bs'\given \bS_t = \bs, \be_t = \be, \bA_t = \ba) = \textnormal{P}(\bS_{t+1} = \bs'\given \bS_t=\bs, \bA_t=\ba).
 }
\end{assumption}
\begin{assumption}
\label{asm:epsilon}
The $\be_t$ are independent and identically distributed (IID) according to a Type-I extreme value distribution (aka the Gumbel distribution). 
\end{assumption}

Note that $\bepsilon_t$ could also follow other parametric distributions. As suggested by \cite{arcidiacono2011practical}, Assumption~\ref{asm:epsilon} is virtually standard for all dynamic discrete choice models.  We use Type-I extreme value distribution as an example, other distributions follow a similar analysis.

\subsection{Likelihood for DDCs}

Next, we derive the likelihood of DDCs.
\begin{definition}[Conditional Value Function]
	\label{def:conditional-truncated-value-function}
	\eq{conditional-value-function}{
		Q^{DDC}(\bs, \ba) = r^*(\bA_t, \bS_t) + \gamma \int_{\mathcal{S}}  \bar{V}(\bs') f(\bS_{t+1} = \bs'\given \bS_{t}=\bs, \bA_t=\ba) d\bs', 
	}
	where 
	\eq{integrated-value-function}{
		\bar{V}(\bs) = \EE_{\epsilon}[V^{DDC}(\bs,\bepsilon)].
	}
\end{definition} 

\begin{lemma}
\label{lem:likelihood}
Let the assumptions in Section~\ref{sec:assumptions} be satisfied. The likelihood of observation $\mathds{X} = \curly{\bs_1, \ba_1, \bs_2, \ba_2, \cdots, \bs_T, \ba_T}$ becomes: 

\ali{likelihood}{
\text{L}^{DDC}(\mathds{X}) = \prod_{t=1}^{T-1} \textnormal{P}(\bS_{t+1} = \bs_{t+1} \given \bS_t = \bs_t, \bA_t = \ba_t) \prod_{t=1}^{T}\left[\exp\left\{ Q^{DDC}(\bs_t, \ba_t) \right\}/Z^{DDC}(\bs_t)\right],
}

where  
\alis{
	Z^{DDC}(\bs) = \sum_{\ba\in \mathcal{A}} \exp\left\{ v(\bs, \ba) \right\}.
}
\end{lemma}  

Comparing \eqref{eq:likelihood} with the likelihood in Lemma~\ref{lem:PQR}, it is clear that DDC is a special case of the considered setting when the action space is discrete and $\alpha=1$.

\section{Extended Theoretical Results} 
In this section, we provide more information about our theoretical results including Lemma~\ref{lem:confound}, Theorem~\ref{thm:asym}, and Theorem~\ref{thm:nonasym}.
 
\label{sec:theoretical-results}
\subsection{Proof of  Lemma~\ref{lem:confound}}
\label{sec:confound-proof}
In this section, we prove Lemma~\ref{lem:confound}. By Theorem~\ref{thm:asym}, we know
\eq{temp-confound}{
r(\bs,\ba)=   Q(\bs,\ba) - \gamma  \mathds{E}\left[-\alpha \log(\pi^*(\bs', \ba^{A})) + Q(\bs', \ba^{A})\given \bs, \ba
\right]. 
}
In other words, we can recover the true reward function if the $Q$ input to \ref{eq:main} is not shaped. Next, we take an input shaped by $\phi(\bs)$:  
\eqs{
Q'(\bs,\ba) := Q(\bs,\ba) + \phi(\bs).
} 
The according reward estimator $r'(\bs, \ba)$ can be derived as 
\alis{
r'(\bs, \ba) &= Q'(\bs,\ba) - \gamma  \mathds{E}\left[-\alpha \log(\pi^*(\bs', \ba^{A})) + Q'(\bs', \ba^{A})\given \bs, \ba
\right] 
\\ &= Q(\bs,\ba) - \gamma  \mathds{E}\left[-\alpha \log(\pi^*(\bs', \ba^{A})) + Q(\bs', \ba^{A})\given \bs, \ba
\right] +\phi(\bs) - \gamma \EE \left[ \phi(\bs')\given \bs, \ba
\right].
}
Finally, by \eqref{eq:temp-confound} and the definition of $\Phi(\bs,\ba)$, we have
\eqs{
r'(\bs,\ba) = r(\bs,\ba) + \Phi(\bs,\ba).
}

\subsection{Theoretical Results for Theorem~\ref{thm:asym}}
\label{sec:thm1}
The details for the derivation of Theorem~\ref{thm:asym} is provided in this Section.  

\subsubsection{Auxiliary Lemmas for Theorem~\ref{thm:asym}}
To start with, we provide required lemmas.  
\begin{lemma}
\label{lem:bellman}
Following the definitions and Lemma~\ref{lem:PQR}, we can derive:

\eqs{
V(\bs) = \EE(Q(\bs,\bA)) + \alpha \mathcal{H}(\pi^*(\bs, \cdot )), 
}
where the expectation is over the action variable $\bA$ following the optimal policy  \eqref{eq:policy-s}.
\end{lemma}
\begin{proof}
To start with, we insert the optimal policy to the definition of $V$: 
\alis{
V(\bs) &=\max_\pi \sum_{t=0}^\infty \gamma^t \, \mathds{E}[
 r(\bS_t, \bA_t)+ \alpha \mathcal{H}(\pi( \bS_t, \cdot))  \given \bS_0 = \bs
]
\\ &=  \sum_{t=0}^\infty \gamma^t \, \mathds{E}[
 r(\bS_t, \bA_t)+ \alpha \mathcal{H}(\pi^*( \bS_t, \cdot))  \given \bS_0 = \bs].
}
Next, we take the case $t=0$ out of the summation, and derive
\alis{
V(\bs) &=  \mathds{E}[
 r(\bs, \bA_0)+ \alpha \mathcal{H}(\pi^*( \bs, \cdot))\given \bS_0 = \bs] 
 + \sum_{t=1}^\infty \gamma^t \, \mathds{E}[
 r(\bS_t, \bA_t)+ \alpha \mathcal{H}(\pi^*( \bS_t, \cdot))  \given \bS_0 = \bs]
 \\ &= \alpha \mathcal{H}(\pi^*( \bs, \cdot)) + \EE \left[
 r(\bs, \bA) \given \bS_0 = \bs
 \right]+ \mathds{E}\left\{\sum_{t=1}^\infty \left[ \gamma^t
 r(\bS_t, \bA_t)+ \gamma^t\alpha \mathcal{H}(\pi^*( \bS_t, \cdot)) \right] \given \bS_0 = \bs \right\}.
}

Finally, by the definition of $Q$ in Lemma~\ref{lem:PQR}, we can finish the proof:
\eqs{
V(\bs) = \EE[Q(\bs,\bA)] + \alpha \mathcal{H}(\pi^*(\bs, \cdot )).
}
\end{proof}
 
\begin{lemma}
\label{lem:v}
The value function \eqref{eq:energy-control} satisfies the following equation:
\eq{new-bell}{
V(\bs) =  \alpha \log \int_{\ba\in\mathcal{A}}\exp\left(\frac{Q(\bs,\ba)}{\alpha}\right) d\ba. 
}
\end{lemma}
\begin{proof}
By Lemma~\ref{lem:bellman}, we can derive the relationship between $Q$ and $V$ as 
\eqs{
V(\bs) = \EE[Q(\bs,\bA)] + \alpha \mathcal{H}(\pi^*(\bs, \cdot )),
}
where the expectation is over the action variable $\bA$ following the optimal policy  \eqref{eq:policy-s}. 
Next, by the definitions of expectation and information entropy, we can derive
\alis{
V(\bs) &= \int_{\ba\in \mathcal{A}} Q(\bs,\ba)\pi^*(\bs,\ba) d\ba - \alpha \int_{\ba\in \mathcal{A}} \log(\pi^*(\bs,\ba)) \pi^*(\bs,\ba) d\ba 
\\ &= \int_{\ba\in \mathcal{A}} Q(\bs,\ba)\pi^*(\bs,\ba) d\ba
 - \alpha \int_{\ba\in \mathcal{A}} \frac{Q(\bs,\ba)}{\alpha}  \pi^*(\bs,\ba) d\ba 
 \\&\quad + \alpha\int_{\ba\in \mathcal{A}} \log \left[
 \int_{\ba' \in \mathcal{A}}
 \exp\left(
 \frac{Q(\bs,\ba')}{\alpha}
 \right)d \ba'
 \right]\pi^*(\bs,\ba) d\ba
  \\&= \alpha \log \left[
 \int_{\ba' \in \mathcal{A}}
 \exp\left(
 \frac{Q(\bs,\ba')}{\alpha}
 \right)d \ba'
 \right]
 \\&= \alpha \log \left[
 \int_{\ba \in \mathcal{A}}
 \exp\left(
 \frac{Q(\bs,\ba)}{\alpha}
 \right)d \ba
 \right].
}

\end{proof}
As $\alpha$ approaches $0$, Lemma~\ref{lem:v} is consistent with the case with deterministic policies corresponding to $\alpha = 0$. We summarize the analysis in Remark~\ref{rem:1}

\begin{remark}
\label{rem:1}
As $\alpha$ approaches $0^+$, \eqref{eq:new-bell} is consistent with the Bellman equation with deterministic polices:
\eqs{
\lim_{\alpha \to 0^+} V(\bs) = \max_a Q'(\bs,\ba),
}
where 
\eqs{
Q'(\bs,\ba) = \lim_{\alpha \to 0^+} Q(\bs,\ba).
}

\end{remark}
\begin{proof}
To start with, we review the definition of $L^p$-norm of functions:
\begin{definition}
Given a measurable space $(\mathcal{A}, \mu)$ and a real number $p \in [1, \infty)$, the $L^p$-norm of a function  $f:\mathcal{A} \mapsto \mathds{R}$ is defined as
\eqs{
\norm{f}_{p} = \left(
\int_{ \mathcal{A}}\abs{f}^p d\mu
\right)^{1/p}
}

\end{definition}

Therefore, \eqref{eq:new-bell} can be represented by the $L^p$-norm of the function of $\ba$ $\exp(Q(\bs,\cdot))$ with $p = 1/\alpha$:
\alis{
V(\bs)
 &= \alpha \log \int_{\ba\in\mathcal{A}}\exp\left(\frac{Q(\bs,\ba)}{\alpha}\right) d\ba
 \\
  &=  \log\abs{ \int_{\ba\in\mathcal{A}}\exp\left(\frac{Q(\bs,\ba)}{\alpha}\right) d\ba}^{\alpha}
\\ &= 
\log\abs{ \int_{\ba\in\mathcal{A}}  \left[\exp\left(Q(\bs,\ba)
\right)\right]^{1/\alpha} d\ba}^{\alpha} 
\\ &= 
\log \norm{\exp(Q(\bs,\cdot))}_{1/\alpha}.
}
Then, we take limit to the both sides of $V(\bs)$ and derive
\eqs{
\lim_{\alpha \to 0^+} V(\bs)
 = \lim_{\alpha \to 0^+}\log \norm{\exp(Q(\bs,\cdot))}_{1/\alpha}=
\log \norm{\exp(Q'(\bs,\cdot))}_{\infty} = \norm{Q'(\bs,\cdot)}_{\infty},
}
with $Q'(\bs,\ba) = \lim_{\alpha \to 0^+} Q(\bs,\ba)$. Note that the second equation is true since both $\log$ and $\exp$ are monotonic functions.  

\end{proof}

\begin{lemma}
\label{lem:alt}
Let $\pi^*(\bs , \ba)$ be the optimal policy of agents, $Q(\bs,\ba)$ the ground-truth $Q$-function, and $r(\bs,\ba)$ the true reward. Under the formulation in Section~\ref{sec:PQR}, we have
\eqs{
Q(\bs,\ba) = r(\bs,\ba) + \gamma  \mathds{E}\left[-\alpha \log(\pi^*(\bs', \ba^{A})) + Q(\bs', \ba^{A})\given \bs, \ba
\right].
}
\end{lemma}

\begin{proof}
According to Lemma~\ref{lem:PQR}, we have
\eqs{
    \pi^*(\bs , \ba) = \frac{ \exp {\left(
\frac{1}{\alpha}
Q(\bs, \ba)\right)} }{ \int_{\ba' \in \mathcal{A}}\exp {\left(
\frac{1}{\alpha}
Q(\bs, \ba')\right)} d\ba'}   .
}
By Lemma~\ref{lem:v}, we have
\eq{v}{
V(\bs) = \alpha \log \int_{\ba\in\mathcal{A}}\exp\left(\frac{Q(\bs,\ba)}{\alpha}\right) d\ba. 
}
For the next step, we consider a specific action $\ba^A$, and extract $ \alpha \log \left [\exp \left(\frac{Q(\bs,\ba^A)}{\alpha}\right)\right]$ from  \eqref{eq:v}:
\ali{v-al}{
V(\bs) &= \alpha \log \left[ \frac{\int_{\ba\in\mathcal{A}}\exp\left(\frac{Q(\bs,\ba)}{\alpha}\right) d \ba }{ \exp \left(\frac{Q(\bs,\ba^A)}{\alpha}\right)}
\right] + \alpha \log \left [\exp \left(\frac{Q(\bs,\ba^A)}{\alpha}\right)\right]
\\ & =  \alpha \log\left(
\frac{1}{\pi^*(\bs,\ba^A)}
\right) + Q(\bs, \ba^A)
\\ &=  -\alpha \log\left(\pi^*(\bs,\ba^A)\right) + Q(\bs,\ba^A).
}

According to Theorem~1 in \citet{haarnoja2017reinforcement}, we have
\eq{bell-q}{
    Q(\bs,\ba) = r(\bs,\ba) + \gamma \EE\left[
    V(\bs') \given \bs, \ba
    \right]
}

Finally, by taking \eqref{eq:v-al} into \eqref{eq:bell-q}, we prove the result. 

\end{proof}

\begin{lemma}
\label{lem:fix}
Let $\pi^*(\bs , \ba)$ be the optimal policy of agents, $Q(\bs,\ba)$ the ground-truth $Q$-function, and $r(\bs,\ba)$ the true reward. Define $\mathcal{T}$ as an operator on the set of continuous bounded functions $f: \mathcal{S} \longmapsto \mathds{R}$:
\eq{t-star}{
    \mathcal{T}f(\bs) := g(\bs)+\gamma \EE{\left[ -\alpha \log(\pi^*(\bs', \ba^{A})) + f(\bs') \given \bs, \ba^{A} \right]}.
}
Under the formulation in Section~\ref{sec:PQR} and Assumption~\ref{asm:identification}, $Q^A(\bs):=Q(\bs,\ba^A)$ is the unique solution to 
\eq{fixed}{
f(\bs) = \mathcal{T}f(\bs).
}
\end{lemma}

\begin{proof}
It is obvious that $\mathcal{T}$ satisfies the monotonicity and discounting condition. Therefore, we can conclude that $\mathcal{T}$ is a contraction, and \eqref{eq:fixed} has the unique solution. 

On the other hand, by taking $\ba = \ba^A$ to Lemma~\ref{lem:alt}, we have
\eq{qa}{
Q^A(\bs) = \mathcal{T}Q^A(\bs).
}
Thus, $Q^A(\bs)$ is the unique solution.
\end{proof}

\subsubsection{Proof of Theorem~\ref{thm:asym}}
\begin{proof}
Since the transition probabilities are given, and $\hat{\pi}(\bs,\ba)$ is accurate, we have $\hat{\EE} = \EE $, $\hat{\pi} = \pi^*$, and thus $\hat{\mathcal{T}} = \mathcal{T}$. According to Lemma~\ref{lem:fix}, \ref{eq:contraction} (denoted as $\hat{Q}^A(\bs)$) uniquely recovers the true $Q^A(\bs)$:
\eqs{
\hat{Q}^A(\bs) = Q^A(\bs).
}
By Lemma~\ref{lem:PQR}, we have
\eqs{
\pi^*(\bs,\ba) = \frac{\exp\left( \frac{Q(\bs,\ba)}{\alpha}\right)}{
\int_{\ba' \in \mathcal{A}} \exp\left(
\frac{Q(\bs,\ba')}{\alpha}
\right) d\ba'
}.
}

Then, we have
\alis{
    &\log(\hat{\pi}(\bs,\ba)) - \log(\hat{\pi}(\bs,\ba^{A}))
    \\ &= \log\left[\exp\left(\frac{Q(\bs,\ba)}{\alpha}\right)\right] -\log\left[\exp\left(\frac{Q^A(\bs)}{\alpha}\right)\right]
    \\ &\quad + \log\left(
\int_{\mathcal{A}} \exp(\frac{Q(\bs, \ba)}{\alpha}) d \ba
\right) - \log\left(
\int_{\mathcal{A}} \exp(\frac{Q(\bs, \ba)}{\alpha}) d \ba
\right) 
    \\&= \frac{1}{\alpha}(Q(\bs,\ba) - Q^A(\bs)),
}
suggesting 
\eqs{\alpha \log(\hat{\pi}(\bs,\ba)) - \alpha \log(\hat{\pi}(\bs,\ba^{A})) + \hat{Q}^A(\bs) = Q(\bs,\ba).}
Therefore, \ref{eq:q} recovers the true $Q$-function:
\eq{qt}{
\hat{Q}(\bs,\ba) = Q(\bs,\ba).
}
Finally, by taking \eqref{eq:qt} into \ref{eq:main}, we derive
\alis{
\hat{r}(\bs, \ba) &= \hat{Q}(\bs, \ba)- \gamma \hat{\mathds{E}}\left[-\alpha \log(\hat{\pi}(\bs', \ba^{A})) + \hat{Q}(\bs', \ba^{A})\given \bs, \ba
\right] 
\\ &= {Q}(\bs, \ba)- \gamma {\mathds{E}}\left[-\alpha \log({\pi}^*(\bs', \ba^{A})) + {Q}(\bs', \ba^{A})\given \bs, \ba
\right],
}
where $\hat{\pi} = \pi^*$ according to the setting of Theorem~\ref{thm:asym}. Finally, by Lemma~\ref{lem:alt}, we have
\eqs{
\hat{r}(\bs, \ba) = {Q}(\bs, \ba)- \gamma {\mathds{E}}\left[-\alpha \log({\pi}^*(\bs', \ba^{A})) + {Q}(\bs', \ba^{A})\given \bs, \ba
\right] = r(\bs,\ba).
}

\end{proof}

\subsection{Theoretical Results for Theorem~\ref{thm:nonasym}}
\label{sec:proof-nonasym}

In this Section, we provide all the assumptions required by Theorem~\ref{thm:nonasym}, and prove Theorem~\ref{thm:nonasym}.  
\subsubsection{Extra Definitions and Assumptions for Theorem~\ref{thm:nonasym}}
\label{sec:extra-assm}
We first list the notions required for the proof of Theorem~\ref{thm:nonasym}. To start with we define $\bar{r}$ and $\bar{Q}$ with $\hat{\pi}$ and the exact expectation $\EE$. 
\begin{definition}
\label{def:q-bar}
We use $Q_k$ to denote the result of Algorithm~\ref{alg:fqii} after the $k^{\text{th}}$ iteration. Accordingly, $Q_k^{A}(\bs) = Q_k(\bs,\ba^A)$.
Define $\bar{Q}^A(\bs)$ as the solution to a fixed point to the operator $\bar{\mathcal{T}}$:
\eqs{
\bar{\mathcal{T}} \bar{Q}^{A}(\bs) = \bar{Q}^{A}(\bs),
}with
\eqs{
    \bar{\mathcal{T}} \bar{Q}^A(\bs) =g(\bs)+\gamma \EE{\left[ -\alpha \log(\hat{\pi}(\bs', \ba^{A})) + \bar{Q}^A(\bs') \given \bs, \ba^{A} \right]},
} 
with an exact expectation and the estimated policy. 
\end{definition}

\begin{definition}
\label{def:reward}
Denote the estimated reward function by Algorithm~\ref{alg:reward} as $\hat{r}(\bs, \ba)$. Define 
\eqs{
\bar{r}(\bs,\ba) :=\hat{Q}(\bs, \ba)- \gamma \mathds{E}{\left[-\alpha \log(\hat{\pi}(\bs', \ba^{A})) + \hat{Q}(\bs', \ba^{A})\given \bs, \ba
\right]},
}
with an exact expectation and the estimated policy. 
\end{definition}

\begin{definition}[Definition 5.1 in \citet{arora2019fine}]
A distribution is $(\lambda, P, n)$-non-degenerate if for $n$ IID samples $\curly{(\bx_i, y_i)}_{i=1}^n$ with $\bx_i \in\mathcal{S}\times \mathcal{A}$, $\lambda_{min}(\bH^{\infty})\geq \lambda \geq 0$ with probability at least $1-P$. $\lambda_{min}$ denotes the smallest eigen value. $\bH^{\infty}$ is defined in Theorem~\ref{thm:nonasym}. \end{definition}

\begin{definition}
\label{def:sample}
$\curly{(\bx_i, y_i)}_{i=1}^n$ follows a $(\lambda, \frac{P}{3}, n)$-non-degenerate distribution $\mathcal{D}$, and $\curly{(\bx_i, y^k_i)}_{i=1}^n$ follows a $(\lambda, \frac{P}{3}, n)$-non-degenerate distribution $\mathcal{D}^k$. We use subscripts like $D_{\bs}$ to denote the marginal distribution of the state variables according to $\mathcal{D}$.
\end{definition}
It can be noticed that \ref{def:sample} is trying to quantify the quality of sample generation process. In practice, it is very hard to exactly derive $\lambda$.  

\begin{definition}
\label{def:rho}
Define 
\eqs{
\rho_{k}(\bs) :=  g(\bs)+\gamma \EE\left[ -\alpha \log(\hat{\pi}(\bs', \ba^{A})) + {Q}^A_{k-1}(\bs') \given \bs, \ba^{A} \right] -{Q}^A_k(\bs)
}
and 
\eqs{\epsilon_{Max} := \max_{k \in [N]} \left[ \int_{\mathcal{S}} \abs{\rho_{k}(\bs)}^2 d\mathcal{D}_{\bs}(\bs) \right]^{1/2}.} 
\end{definition}

\begin{definition}
\label{def:m}
We use $m$ to denote the number of neurons in the hidden layer of neural networks in Algorithm~\ref{alg:reward} and Algorithm~\ref{alg:fqii}. 
\end{definition}

To quantify the generalization errors of neural networks, we require the following assumptions. 
\begin{assumption}
\label{asm:iterations}
The numbers of training iterations of neural networks in Algorithm~\ref{alg:reward} and Algorithm~\ref{alg:fqii} equal to $\Omega\left(\frac{1}{\eta \lambda} \log \frac{n}{P}\right)$, where $\eta>0$ is the learning rate. $m$ in Definition~\ref{def:m} satisfies $m \geq c^{-2} \textnormal{poly}(n, \lambda^{-1}, \frac{3}{P})$, where $c = O\left(
\frac{\lambda P}{n}
\right)$.
\end{assumption}

\begin{assumption}
\label{asm:deep-train}
The neural networks in Algorithm~\ref{alg:reward} and Algorithm~\ref{alg:fqii} are two-layer ReLU networks trained by the randomly initialized gradient descent.

\end{assumption}

Further, for the MDP in Algorithm~\ref{alg:fqii}, we pose the following assumptions commonly used for the FQI method~\citep{munos2008finite, yang2019theoretical}. 

\begin{assumption}
\label{asm:R}
The true reward function can be bounded by $\max_{\bs, \ba}\abs{r(\bs,\ba)}\leq R$.
\end{assumption}
 \begin{assumption}
 \label{asm:fqi}
Define the operator $\bar{\mathcal{T}}$  on the set of continuous bounded functions $f: \mathcal{S} \longmapsto \mathds{R}$ 
\eqs{
    \bar{\mathcal{T}}f(\bs) := g(\bs)+\gamma {\EE}{\left[ -\alpha \log(\hat{\pi}(\bs', \ba^{A})) + f(\bs') \given \bs, \ba^{A} \right]}.
}
Define the concentration coefficient as 
\eqs{
\kappa(m) = \left[ \EE_{\mathcal{D}_{\bs}} \abs{\frac{d\bar{\mathcal{T}}^m \mathcal{D}_{\bs}}{d\mathcal{D}_{\bs} }}^2  \right]^{1/2}.
}
We assume that $(1-\gamma) \sum_{m\geq 1} \gamma^{m-1} \kappa(m) \leq \psi$. 
\end{assumption}

\subsubsection{Auxiliary Lemmas for Theorem~\ref{thm:nonasym}}

\begin{lemma}[Hoeffding's Inequality]
\label{lem:hoeffding}
Let $X_1,X_2,\cdots,X_n$ be $n$ IID~random variables drawn from distribution $\mathcal{D}$, with  $0 \leq X_i \leq a$, $\forall i \in \curly{1,2,\cdots,n}$. Let $\bar{X} := \frac{1}{n}\sum_{i=1}^n X_i$. Then, for any $t >0$, 
\begin{equation*}
    \text{P}(\abs{\bar{X} - \EE[\bar{X}] }\geq t ) \leq 2\exp{\left(-\frac{2nt^2}{a^2} \right)}.
\end{equation*}
\end{lemma}

To start with, we consider the errors induced by Algorithm~\ref{alg:reward} by studying the difference between $\hat{r}(\bs,\ba)$ and $\bar{r}(\bs,\ba)$. 
\begin{lemma}
\label{lem:deep}

With Definition~\ref{def:reward} and \ref{def:sample}, under Assumption~\ref{asm:deep}, \ref{asm:deep-train}, and \ref{asm:iterations}, we have
\alis{
\EE_{(\bs,\ba)\sim \mathcal{D_{\bs,\ba}}} [\abs{\hat{r}(\bs,\ba) - \bar{r}(\bs,\ba)}] 
\leq \sqrt{\frac{2\by^T(\bH^\infty)^{-1} \by}{n}} + O{\left(
\sqrt{\frac{\log(\frac{n}{\lambda P})}{n}} 
\right)},
}
with probability at least $1-P$. 
\begin{proof}
Note that the difference between $\hat{r}$ and $\bar{r}$ is no more than the estimation error of the expectation in $\hat{r}$, which is conducted by neural networks. Therefore, we apply Theorem 5.1 in \citet{arora2019fine}. 
\end{proof}
\end{lemma}

Next, we consider the error of Algorithm~\ref{alg:fqii}. First, we bound the difference between $\hat{Q}^A(\bs)$ and $\bar{Q}^A(\bs)$.
\begin{lemma}
\label{lem:fqi1}
With Definitions~\ref{def:q-bar} and \ref{def:rho}, under Assumptions~\ref{asm:R} and \ref{asm:fqi}, we have
\alis{
\EE_{(\bs,\ba)\sim \mathcal{D_{\bs,\ba}}} \left[\abs{\hat{Q}^A(\bs) - \bar{Q}^A(\bs)}\right] 
&\leq \sum_{k=0}^{N-1}\gamma^{N-k-1}\kappa(N-k-1) \epsilon_{Max}+\gamma^N \frac{2R}{1-\gamma},
}
with probability at least $1-P$. 

\end{lemma}
\begin{proof}
By Lemma C.2 in \citet{yang2019theoretical}, we can derive
\eqs{
    {\hat{Q}}^A(\bs) - \bar{Q}^A(\bs) = \sum_{k=0}^{N-1}\left( \gamma^{N-k-1} \bar{\mathcal{T}}^{N-k-1} \rho_{k+1}(\bs)\right) + \gamma^N \bar{\mathcal{T}}^N (\bar{Q}^A(\bs) - {Q}^A_0 (\bs)),
}

where ${Q}^A_0$ denotes the initialized estimator for Algorithm~\ref{alg:fqii}. Then, we apply the expectation to the absolute value of both sides of the equation above: 
\alis{
\EE_{\bs\sim \mathcal{D_{\bs}}} \left[\abs{\hat{Q}^A(\bs) - \bar{Q}^A(\bs)}\right] 
&\leq \sum_{k=0}^{N-1} \left(\gamma^{N-k-1}\EE_{\bs\sim \mathcal{D_{\bs}}} \left[\abs{\bar{\mathcal{T}}^{N-k-1} \rho_{k+1}(\bs)}\right]\right) 
\\&\quad + \gamma^N \EE_{\bs\sim \mathcal{D_{\bs}}} \left[\abs{ \bar{\mathcal{T}}^N (\bar{Q}^A(\bs) - {Q}^A_0(\bs)) }\right]
\\ &\leq \sum_{k=0}^{N-1} \gamma^{N-k-1}\EE_{\bs\sim \mathcal{D_{\bs}}} \left[\abs{\bar{\mathcal{T}}^{N-i-1} \rho_{k+1}(\bs)}\right] +\gamma^N \frac{2R}{1-\gamma}.
}
By Cauchy-Schwarz inequality, we have
\alis{
\EE_{(\bs,\ba)\sim \mathcal{D_{\bs,\ba}}} \left[\abs{\bar{\mathcal{T}}^{N-k-1} \rho_{k+1}}\right] &\leq \left[ \int_{\mathcal{S}} \abs{\frac{d(\bar{\mathcal{T}}^{N-k-1}\mathcal{D}_{\bs})}{d \mathcal{D}_{\bs}}(\bs)}^2 d\mathcal{D}_{\bs}(\bs) \right]^{1/2} \left[ \int_{\mathcal{S}} \abs{\rho_{k+1}(\ba)}^2 d\mathcal{D}_{\bs}(\bs) \right]^{1/2}
\\ &\leq \kappa(N-k-1) \left[ \int_{\mathcal{S}} \abs{\rho_{k+1}(\ba)}^2 d\mathcal{D}_{\bs}(\bs) \right]^{1/2},
}
where the last inequality is according to Assumption~\ref{asm:fqi}. Therefore, 
\alis{
\EE_{\bs\sim \mathcal{D}_{\bs}} \left[\abs{\hat{Q}^A(\bs) - \bar{Q}^A(\bs)}\right] 
&\leq \sum_{i=0}^{N-1} \gamma^{N-k-1}\EE_{\bs\sim \mathcal{D}_\bs} \left[\abs{P^{N-k-1} \rho_{k+1}(\bs) }\right] +\gamma^N \frac{2R}{1-\gamma}
\\ &\leq \sum_{i=0}^{N-1}\gamma^{N-k-1}\kappa(N-k-1) \left[ \int_{\mathcal{S}} \abs{\rho_{k+1}(\bs)}^2 d \mathcal{D}_{\bs}(\bs) \right]^{1/2}+\gamma^N \frac{2R}{1-\gamma}
\\ &\leq \sum_{k=0}^{N-1}\gamma^{N-k-1}\kappa(N-k-1) \epsilon_{Max}+\gamma^N \frac{2R}{1-\gamma},
} 
where the last inequality is according to Definition~\ref{def:rho}.
\end{proof}

We study the $\rho_k(\bs)$ in $\epsilon_{Max}$:
\alis{
\rho_{k}(\bs) &=  g(\bs)+\gamma \EE\left[ -\alpha \log(\hat{\pi}(\bs', \ba^{A})) + {Q}^A_{k-1}(\bs') \given \bs, \ba^{A} \right] -{Q}^A_k.
}
According to the the fifth line of \ref{alg:fqii}, $Q_k^A$ is a deep function trained using the samples $g(\bs)+\gamma  -\alpha \log(\hat{\pi}(\bs', \ba^{A})) + {Q}^A_{k-1}(\bs')$, where $\bs'$ is the next-step state variable. In other words, the $Q_k^A$ is trained to estimate $g(\bs)+\gamma \EE\left[ -\alpha \log(\hat{\pi}(\bs', \ba^{A})) + {Q}^A_{k-1}(\bs') \given \bs, \ba^{A} \right]$, the first part of $\rho_{k}(\bs)$. Therefore, it can be seen that $\epsilon_{Max}$
is really caused by the neural networks used in Algorithm~\ref{alg:fqii}.  We use the results in \citet{arora2019fine} again to bound $\epsilon_{Max}$.

\begin{lemma}
\label{lem:fqi2}
With Definition~\ref{def:sample}, under Assumptions~\ref{asm:deep}, \ref{asm:deep-train}, and \ref{asm:iterations}, we have
\eqs{
\epsilon_{Max} \leq\max_{k \in[N]} \sqrt{\frac{2\by^T_k(\bH^\infty)^{-1} \by_k}{n}} + O{\left(
\sqrt{\frac{\log(\frac{n}{\lambda P})}{n}} 
\right)}.
}
with probability at least $1-(N+1)P$. 
\end{lemma}

\begin{proof}
We consider 
\ali{temp1}{
\quad&\left[ \int_{\mathcal{S}} \abs{\rho_{k}(\bs)}^2 dD_{\bs}(\bs) \right]^{1/2} 
\\&= \left[ \int_{\mathcal{S}} \abs{Q^A_k(\bs) - g(\bs)-\gamma \EE\left[ -\alpha \log(\hat{\pi}(\bs', \ba^{A})) + {Q}^A_{k-1}(\bs') \given \bs, \ba^{A} \right]}^2 dD_{\bs}(\bs) \right]^{1/2}
\\ &\leq\left[ \int_{\mathcal{S}\times\mathcal{Y}} \abs{y - g(\bs)-\gamma \EE\left[ -\alpha \log(\hat{\pi}(\bs', \ba^{A})) + {Q}^A_{k-1}(\bs') \given \bs, \ba^{A} \right]}^2 dD_{\bs,y}(\bs,y) \right]^{1/2}
\\&\quad + \left[ \int_{\mathcal{S}\times\mathcal{Y}} \abs{{Q}^A_k(\bs) - y}^2 dD_{\bs,y}(\bs,y) \right]^{1/2}
}
It should be noticed that the second part on the right hand side is the generation error of the deep estimation.  According to Theorem 5.1 in \citet{arora2019fine}, with probability at least $1-P$, we have
\eqs{
\left[ \int_{\mathcal{S}\times\mathcal{Y}} \abs{{Q}^A_k(\bs) - y}^2 dD_{\bs,y}(\bs,y) \right]^{1/2} \leq \sqrt{\frac{2\by^T_k(\bH^\infty)^{-1} \by_k}{n}} + O{\left(
\sqrt{\frac{\log(\frac{n}{\lambda P})}{n}} 
\right)}.
}
We bound the first part of the RHS of \eqref{eq:temp1} by Hoeffding's Inequality (Lemma~\ref{lem:hoeffding}). This provides an error with smaller order, and thus can be ignored. 
Therefore, by the union bound, we can conclude that, with probability at least $1-(N+1)P$, 
\eqs{
\epsilon_{Max} \leq \max_{k \in[N]} \sqrt{\frac{2\by^T_k(\bH^\infty)^{-1} \by_k}{n}} + O{\left(
\sqrt{\frac{\log(\frac{n}{\lambda P})}{n}} 
\right)}.
}
\end{proof}

\begin{lemma}
\label{lem:a-q}
Under Assumption~\ref{asm:policy}, the estimation error of $\hat{Q}$ can be bounded by:
\eqs{\EE_{(\bs,\ba)\sim \mathcal{D}_{\bs,\ba}} \left[\abs{\hat{Q}(\bs,\ba) - {Q}(\bs,\ba)}\right] \leq \EE_{\bs\sim \mathcal{D}_{\bs}} {\left[\abs{\hat{Q}^A(\bs) - Q^A(\bs)}\right]} + 2 \alpha \epsilon_{\pi} .} 
\end{lemma}
\begin{proof}
By the definition of $\pi^*$ in \eqref{eq:policy-s}, we can derive 
\alis{
\log(\pi^*(\bs,\ba)) - \log(\pi^*(\bs,\ba^A)) = \frac{Q(\bs,\ba) - Q(\bs, \ba^A)}{\alpha}.}
Therefore, 
\eqs{
    Q(\bs,\ba) = \alpha\log(\pi^*(\bs,\ba)) -\alpha \log(\pi^*(\bs,\ba^A)) +  Q^A(\bs). 
}

Similarly, according to Algorithm~\ref{alg:fqii}, 
\eqs{
    \hat{Q}(\bs, \ba) = \alpha\log(\hat{\pi}(\bs,\ba)) -\alpha \log(\hat{\pi}(\bs,\ba^A)) + \hat{Q}^A(\bs). 
}
Therefore, by Assumption~\ref{asm:policy}, 
\eqs{\EE_{(\bs,\ba)\sim \mathcal{D}_{\bs,\ba}} \left[\abs{\hat{Q}(\bs,\ba) - {Q}(\bs,\ba)}\right] \leq \EE_{\bs\sim \mathcal{D}_{\bs}} {\left[\abs{\hat{Q}^A(\bs) - Q^A(\bs)}\right]} + 2 \alpha \epsilon_{\pi} .} 
\end{proof}

\begin{lemma}
\label{lem:qbar-q}
With Definitions~\ref{def:q-bar}, under Assumption~\ref{asm:policy}, we have
\eqs{
\EE_{\bs\sim \mathcal{D}_{\bs}} \left[\abs{  \bar{Q}^A(\bs) - {Q}^A(\bs)}\right] \leq \frac{\gamma \alpha \epsilon_\pi}{1-\gamma}.
}
\end{lemma}
\begin{proof}
According to Lemma~\ref{lem:fix}, we have
\eqs{
    Q^A(\bs) := g(\bs)+\gamma \EE{\left[ -\alpha \log(\pi^*(\bs', \ba^{A})) + Q^A(\bs') \given \bs, \ba^{A} \right]}.
}
On the other hand, by Definition~\ref{def:q-bar}
\eqs{
    \bar{Q}^A(\bs) =g(\bs)+\gamma \EE{\left[ -\alpha \log(\hat{\pi}(\bs', \ba^{A})) + \bar{Q}^A(\bs') \given \bs, \ba^{A} \right]}.
}
Therefore, we have
\alis{
&\quad \EE_{\bs\sim \mathcal{D}_{\bs}} \left[\abs{\bar{Q}^A(\bs)-{Q}^A(\bs) }\right]  \\&=\gamma\EE_{\bs\sim \mathcal{D}_{\bs}} \left\{\abs{ \EE\left[ \alpha \log(\pi^*(\bs', \ba^{A}))-\alpha \log(\hat{\pi}(\bs', \ba^{A})) + \bar{Q}^A(\bs')  - {Q}^A(\bs') \given \bs, \ba^{A} \right] } \right\}
\\ &\leq  \gamma\EE_{\bs\sim \mathcal{D}_{\bs}} \left[\abs{{Q}^A(\bs) - \bar{Q}^A(\bs)}\right] + \gamma \alpha \epsilon_{\pi}, 
}which means
\eqs{
\EE_{\bs\sim \mathcal{D}_{\bs}} \left[\abs{{Q}^A(\bs) - \bar{Q}^A(\bs)}\right] \leq \frac{\gamma \alpha \epsilon_\pi}{1-\gamma}. 
}
\end{proof}

\begin{lemma}
\label{lem:rebar-re}
With Definition~\ref{def:q-bar} and \ref{def:reward}, under Assumption~\ref{asm:policy}, we have 

\alis{
&\EE_{(\bs,\ba)\sim \mathcal{D}_{\bs, \ba}} [\abs{{r}(\bs,\ba) - \hat{r}(\bs,\ba)}] 
\\ &\leq (1+\gamma)\EE_{\bs\sim \mathcal{D}_{\bs}} \left[\abs{{Q}^A(\bs) - \hat{Q}^A(\bs)} \right] + 
\EE_{(\bs,\ba)\sim \mathcal{D}_{\bs, \ba}} [\abs{\bar{r}(\bs,\ba) - \hat{r}(\bs,\ba)}]
\\&+ (\gamma + 2)\alpha \epsilon_{\pi}.
}

\end{lemma}
\begin{proof}
To start with, by the tower principle, the triangle inequality, Lemma~\ref{lem:alt} and Definition~\ref{def:q-bar}, we can derive
\alis{
\EE_{(\bs,\ba)\sim \mathcal{D}_{\bs, \ba}} \left[\abs{{r}(\bs,\ba) - \hat{r}(\bs,\ba)} \right] 
\leq& \EE_{(\bs,\ba)\sim \mathcal{D}_{\bs,\ba}} [\abs{{Q}(\bs,\ba) - \hat{Q}(\bs,\ba)}] 
\\ &+ \gamma \alpha \EE_{\bs\sim \mathcal{D}_{\bs}} \left[\abs{\log(\pi^*(\bs,\ba^A)) - \log(\hat{\pi}(\bs,\ba^A)) }\right] 
\\& +\gamma\EE_{(\bs,\ba)\sim \mathcal{D}_{\bs, \ba}} \left[\abs{{Q}^A(\bs) - \hat{Q}^A(\bs)}\right]
\\ \leq& 
\EE_{(\bs,\ba)\sim \mathcal{D}_{\bs, \ba}} \left[\abs{{Q}(\bs,\ba) - \hat{Q}(\bs,\ba)} \right] 
\\&+\EE_{(\bs,\ba)\sim \mathcal{D}_{\bs, \ba}} [\abs{\bar{r}(\bs,\ba) - \hat{r}(\bs,\ba)}]+ \gamma \alpha \epsilon_{\pi} 
\\&+ \gamma\EE_{\bs\sim \mathcal{D}_{\bs}} \left[\abs{{Q}^A(\bs) - \hat{Q}^A(\bs)} \right].
}
Next, we use the results of Lemma~\ref{lem:a-q}, and derive
\begin{align*}
\EE_{(\bs,\ba)\sim \mathcal{D_{\bs,\ba}}} [\abs{{r}(\bs,\ba) - \hat{r}(\bs,\ba)}] 
&\leq (1+\gamma)\EE_{\bs\sim \mathcal{D}_{\bs}} \left[\abs{{Q}^A(\bs) - \hat{Q}^A(\bs)} \right]  
\\&\quad+ 
\EE_{(\bs, \ba)\sim \mathcal{D}_{\bs, \ba}} [\abs{\bar{r}(\bs,\ba) - \hat{r}(\bs,\ba)}]+ 
(\gamma+ 2)\alpha \epsilon_{\pi}, 
\end{align*}
which proves the lemma. 
\end{proof}

\subsubsection{Proof of Theorem~\ref{thm:nonasym}}

Now, we prove Theorem~\ref{thm:nonasym}. To start with, we study the error of Algorithm~\ref{alg:fqii}. Combining Lemma~\ref{lem:fqi1} and Lemma~\ref{lem:fqi2}, we can derive that with the probability of at least $1-(N+1)P$:
\ali{qtilde-qbar}{
& \EE_{(\bs,\ba)\sim \mathcal{D_{\bs,\ba}}} \left[\abs{\hat{Q}^A(\bs) - \bar{Q}^A(\bs)}\right] 
\\ &\leq
\sum_{k=0}^{N-1}\kappa(N-k-1)\gamma^{N-k-1} \left[ \max_{k \in[N]} \sqrt{\frac{2\by^T_k(\bH^\infty)^{-1} \by_k}{n}} + O\left(
\sqrt{\frac{\log(\frac{n}{\lambda P})}{n}} 
\right)\right]+\gamma^N \frac{2R}{1-\gamma}
\\ &\leq  \frac{\phi}{1-\gamma}\max_{k \in[N]} \sqrt{\frac{2\by^T_k(\bH^\infty)^{-1} \by_k}{n}} +\gamma^N \frac{2R}{1-\gamma}+ O{\left(
\sqrt{\frac{\log(\frac{n}{\lambda P})}{n}} 
\right)},
}
where the last inequality is by Assumption~\ref{asm:fqi}.
Further, by using Lemma~\ref{lem:a-q} and Lemma~\ref{lem:qbar-q}, we can derive
\alis{
&\EE_{(\bs,\ba)\sim \mathcal{D_{\bs,\ba}}} \left[\abs{\hat{Q}(\bs,\ba) - {Q}(\bs,\ba)}\right] 
\\&\leq \EE_{\bs\sim \mathcal{D}_{\bs}} \left[\abs{\hat{Q}^A(\bs) - Q^A(\bs)}\right] +2 \alpha \epsilon_{\pi}
\\&\leq \EE_{\bs\sim \mathcal{D}_{\bs}} \left[\abs{\hat{Q}^A(\bs) - \bar{Q}^A(\bs)}\right]  + \EE_{\bs\sim \mathcal{D}_{\bs}} \left[\abs{\bar{Q}^A(\bs) - Q^A(\bs)}\right]+2 \alpha \epsilon_{\pi}
\\ &\leq\frac{\phi}{1-\gamma}\max_{k \in[N]} \sqrt{\frac{2\by^T_k(\bH^\infty)^{-1} \by_k}{n}} +\gamma^N \frac{2R}{1-\gamma}+ O\left(
\sqrt{\frac{\log(\frac{n}{\lambda P})}{n}} 
\right) +\frac{\gamma \alpha \epsilon_\pi}{1-\gamma}
}
As a result, the error of $Q$-function estimation is bounded. Next, we proceed to the reward estimation. According to Lemma~\ref{lem:deep} and Lemma~\ref{lem:rebar-re}, we have
\alis{
&\EE_{(\bs,\ba)\sim \mathcal{D_{\bs,\ba}}} [\abs{\hat{r}(\bs,\ba) - {r}(\bs,\ba)}] 
\\&\leq (1+\gamma)\EE_{(\bs,\ba)\sim \mathcal{D_{\bs,\ba}}} \left[\abs{{Q}(\bs,\ba) - \hat{Q}(\bs,\ba)} \right] 
+\EE_{(\bs,\ba)\sim \mathcal{D_{\bs,\ba}}} [\abs{\bar{r}(\bs,\ba) - \hat{r}(\bs,\ba)}]+ (2+\gamma) \alpha \epsilon_{\pi}
\\ &\leq (1+\gamma)\left\{\frac{\psi}{1-\gamma}\max_{k \in[N]} \sqrt{\frac{2\by^T_k(\bH^\infty)^{-1} \by_k}{n}} +\gamma^N \frac{2R}{1-\gamma}+ O\left(
\sqrt{\frac{\log(\frac{n}{\lambda P})}{n}} 
\right) +\frac{\gamma \alpha \epsilon_\pi}{1-\gamma}\right\}
\\&+ \sqrt{\frac{2\by^T(\bH^\infty)^{-1} \by}{n}} +(2+\gamma) \alpha \epsilon_\pi+ O\left(
\sqrt{\frac{\log(\frac{n}{\lambda P})}{n}} 
\right)
\\ &\leq  \frac{(1+\gamma)\psi}{1-\gamma}\max_{k \in[N]} \sqrt{\frac{2\by^T_k(\bH^\infty)^{-1} \by_k}{n}}+\gamma^N \frac{2R(1+\gamma)}{1-\gamma}+
\frac{2\alpha\epsilon_{\pi}}{1-\gamma}
\\&+\sqrt{\frac{2\by^T(\bH^\infty)^{-1} \by}{n}} 
+O{\left(
\sqrt{\frac{\log(\frac{n}{\lambda P})}{n}} 
\right)},
}
with probability $1-(N+2)P$. Note that the proved results is slightly different from the Theorem~\ref{thm:nonasym}. The two are consistent if we select $C = \max(\psi, R, 1/\lambda)$, and treat $P(N+2)$ as a constant for the tail probability.

 \section{Extended Experiments}
 \label{sec:experiments}
 We now provide extended experiments. 
\subsection{Details of Synthetic Experiments}
\label{sec:syn-details}
In this section, we provide more details about the synthetic experiments implemented in Section~\ref{sec:synthetic}. 
\subsubsection{Data Generation Process}
We consider an MDP defined by the tuple $\curly{\mathcal{S}, \mathcal{A}, \textnormal{P}, \gamma, r}$.
\begin{itemize}
\item $\mathcal{S} =[-p, p]^p$, with $p \in\curly{5,10,20,40}$, denotes the state space. 
\item $\mathcal{A} =[5]$ is the action space.
\item The reward function is defined as
\eqs{
r(\bs,\ba) = \frac{\ba\tanh(\curly{\bs^\top/p, \ba/4}\cdot \bomega )}{4\bm{1}^\top \bomega},
}
where $\bomega$ is a $(p+1)\times 1$ vector.  
\item $\gamma =0.9$ is the discount factor. 
\item $\alpha = 1$.
\item The transition is defined as
\eqs{
\textnormal{P}(\bs'\given \bs,\ba) = 
\begin{cases} 
1 & s' = \bs + \ba/5 - 0.5 \in \mathcal{S} \\ \frac{1}{\abs{\mathcal{S}}} &\textnormal{otherwise}
\end{cases}.  
}
\end{itemize} 
Next, we solve the MDP with a deep energy-based policy with $\gamma = 0.9$ and $\alpha = 1$, using the soft $Q$-learning method in \citet{haarnoja2017reinforcement}. 
By conducting the learned policy for $50000$ steps, we obtain the demonstration dataset, on which we compare PQR, MaxEnt-IRL, AIRL, and SPL-GD. We assume that $\gamma$ and $\alpha$ are known as required by existing methods. Specifically, the hyperparameters used for the soft Q-learning are reported in Table~\ref{tab:soft-q}.

\begin{table*}[h]
\centering
\begin{tabular*}{1\textwidth}{@{\extracolsep{\fill} }c c c c c}
\hline
  \multicolumn{1}{ c}{\multirow{2}{*}{$p$}} &\multicolumn{1}{ c}{\multirow{2}{*}{Optimization Method}} &\multicolumn{1}{ c}{\multirow{2}{*}{Learning Rate}}& Final Loss&\multicolumn{1}{ c}{\multirow{2}{*}{Number of Epochs}}\\ 
\multicolumn{1}{ c  }{}&\multicolumn{1}{ c  }{}&\multicolumn{1}{ c  }{}&for $Q$-Function&\multicolumn{1}{ c  }{}\\
 \hline 
5 &Adagrad  &$10^{-4}$&$4\times10^{-3}$& 100\\
 10 & Adagrad &$10^{-4}$&$6\times10^{-3}$&150\\
20 & Adagrad &$10^{-4}$&$9\times10^{-3}$&175\\
  40& Adagrad &$10^{-4}$&$1.2\times10^{-2}$&175\\   
  \hline 
\end{tabular*}
\caption {Hyperparameters for soft Q-learning.}
  \label{tab:soft-q}
\end{table*}
\subsubsection{Competing Methods}
In this section, we provide more details regarding how we implement existing methods. 
\paragraph{MaxEnt-IRL}
According to our analysis, without an appropriate identification procedure, the estimated $Q$-function may be shaped by any function of $\bs$. Therefore, we implement MaxEnt-IRL with a grounding procedure by Algorithm~\ref{alg:maxent-irl}.
\begin{center}
{\centering 
\begin{minipage}{.6\linewidth}
\begin{algorithm}[H]
		\caption{MaxEnt-IRL with Grounding Procedure}
		\label{alg:maxent-irl}
		\begin{algorithmic}[1]
			\Input $\mathds{X} = \curly{\bs_0, \ba_0, \bs_1, \ba_1, \cdots, \bs_T, \ba_T}$, and $Q(0,0)$
			\Output $\hat{Q}(\bs, \ba)$.
			\State $\hat{Q}(\bs, \ba) \leftarrow $\Call{MaxEnt-IRL}{$\mathds{X}$}
			\State $\hat{Q}(\bs, \ba) \leftarrow \hat{Q}(\bs,\ba) - \hat{Q}(0,0) + Q(0,0)$. 
			\State \Return $\hat{Q}(\bs, \ba)$.
		\end{algorithmic}
	\end{algorithm}
\end{minipage}}
\end{center}
In other words, we allow MaxEnt-IRL to access the ground truth $Q$ value $Q(0,0)$. Empirically, this procedure significantly improves the performance of MaxEnt-IRL. 

\paragraph{SPL-GD}
The original SPL-GD can only deal with discrete and finite state variables. To make the SPL-GD method available for our setting, we provide access to the true value function and Q-function. The procedure is summarized in Algorithm~\ref{alg:SPL-GD}.
\begin{center}
{\centering 
\begin{minipage}{.6\linewidth}
\begin{algorithm}[H]
		\caption{SPL-GD for Continuous States}
		\label{alg:SPL-GD}
		\begin{algorithmic}[1]
			\Input Dataset: $\mathds{X} = \curly{\bs_0, \ba_0, \bs_1, \ba_1, \cdots, \bs_T, \ba_T}$, $Q(\bs,\ba)$ and $V(\bs)$.
			\Output $\hat{r}(\bs, \ba)$.
			\For {$t\in [T]$}
			\State $y_t \leftarrow Q(\bs_t,\ba_t)-\gamma V(\bs_{t+1})$
			\EndFor
			\State Train a linear regression with $\curly{\bs_t}_{t=0}^{T-1}$ and $\curly{\ba_t}_{t=0}^{T-1}$ for $\hat{r}(\bs,\ba)$.
			\State \Return $\hat{r}(\bs, \ba)$.

		\end{algorithmic}
	\end{algorithm}
	\end{minipage}}
\end{center}

\subsubsection{PQR Deep Networks Hyperparameters}
In this section, we provide hyperparameters for training the deep neural networks in PQR. There are three types of deep neural networks in PQR: networks for policy estimation, $Q$ estimation, and reward estimation. The parameters are provided in Table~\ref{tab:pqr}.  
\begin{table*}[h]
\centering
\begin{tabular*}{0.9\textwidth}{@{\extracolsep{\fill} }c c c  c}
\hline
  \multicolumn{1}{ c}{\multirow{2}{*}{Types of Networks}} &\multicolumn{1}{ c}{\multirow{2}{*}{Optimization Method}} &\multicolumn{1}{ c}{\multirow{2}{*}{Learning Rate}}& \multicolumn{1}{ c}{\multirow{2}{*}{Number of Steps}}\\ 
\multicolumn{1}{ c  }{}&\multicolumn{1}{ c  }{}&\multicolumn{1}{ c  }{}&\multicolumn{1}{ c  }{}\\
 \hline 
Policy Estimation &Adagrad  &$5\times10^{-4}$& 600\\
 $Q$ Estimation & Adagrad &$1\times 10^{-3}$&1000\\
Reward Estimation & Adagrad &$1\times 10^{-3}$&1000\\
  \hline 
\end{tabular*}
\caption {Hyperparameters for PQR.}
\label{tab:pqr}
\end{table*}
The training errors of the neural networks are provided in Figure~\ref{fig:train-error} 
\begin{figure*}[h]
	\centering
	\begin{subfigure}{0.3\textwidth}
		\centering
		\includegraphics[scale=0.2]{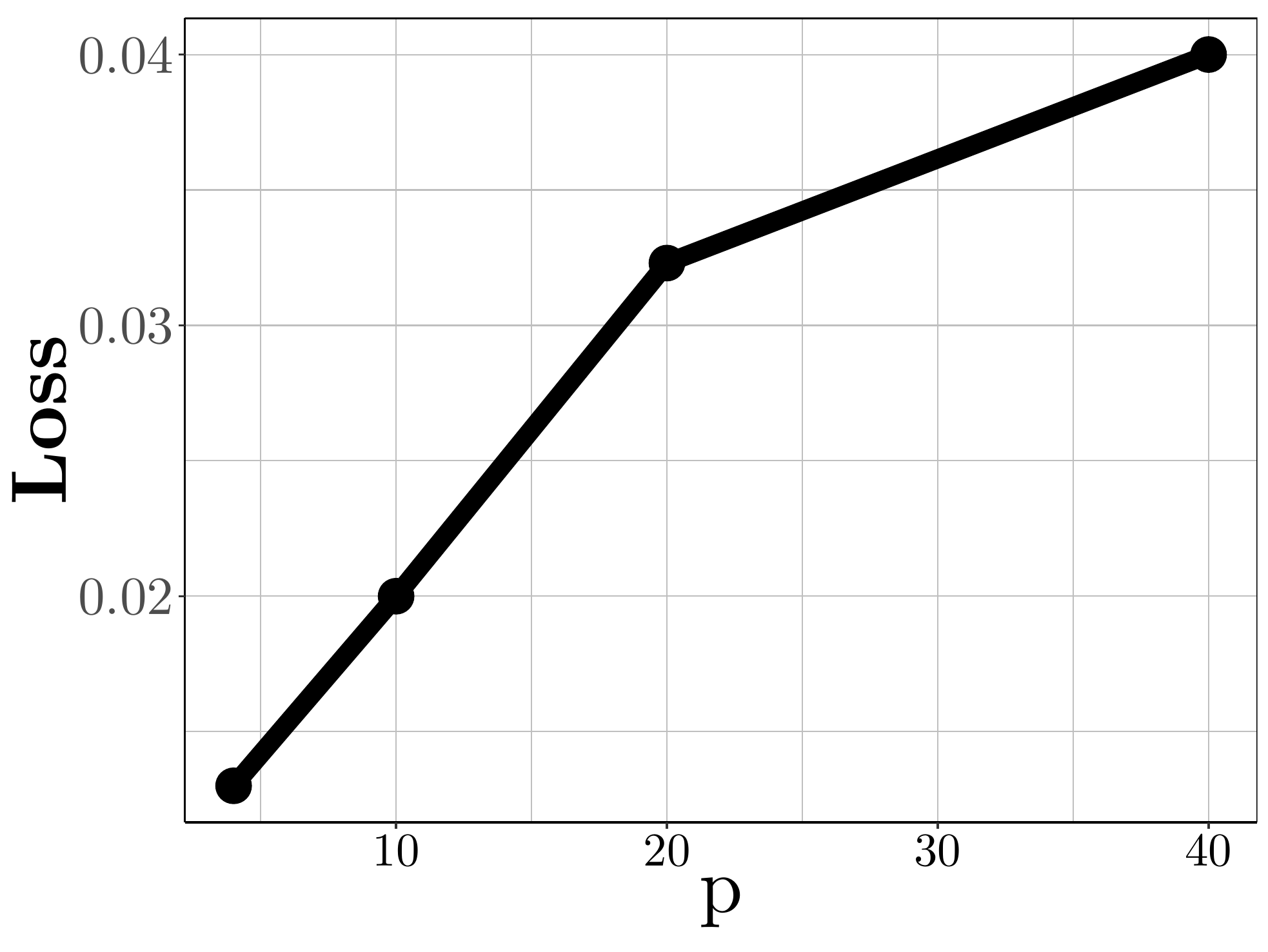}
		\caption{Neural networks for policy.}
	\end{subfigure}
	\centering
	\begin{subfigure}{0.3\textwidth}
		\centering
		\includegraphics[scale=0.2]{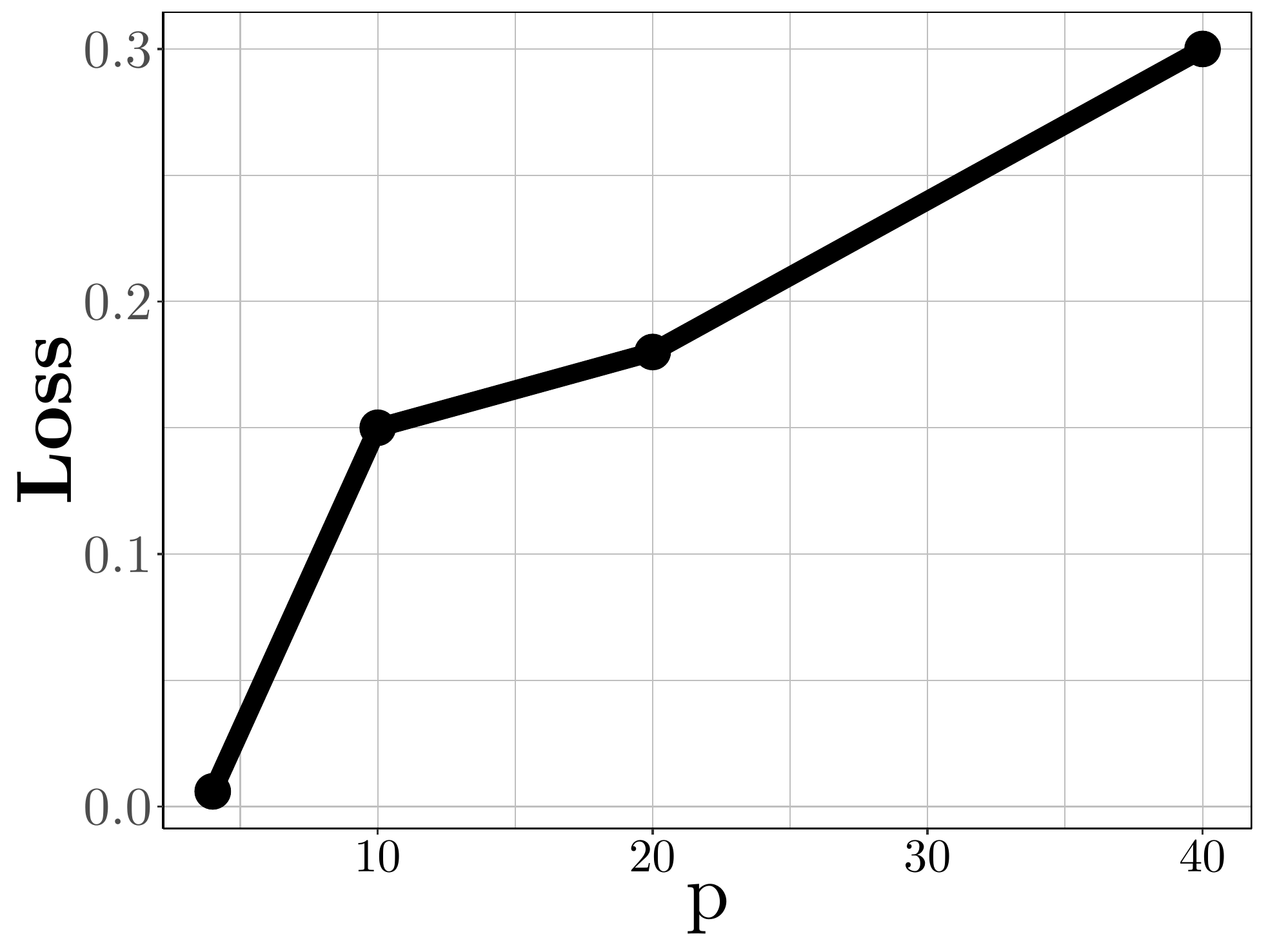}
		\caption{Neural networks for $Q$.}
	\end{subfigure}
	\centering
	\begin{subfigure}{0.3\textwidth}
		\centering
		\includegraphics[scale=0.2]{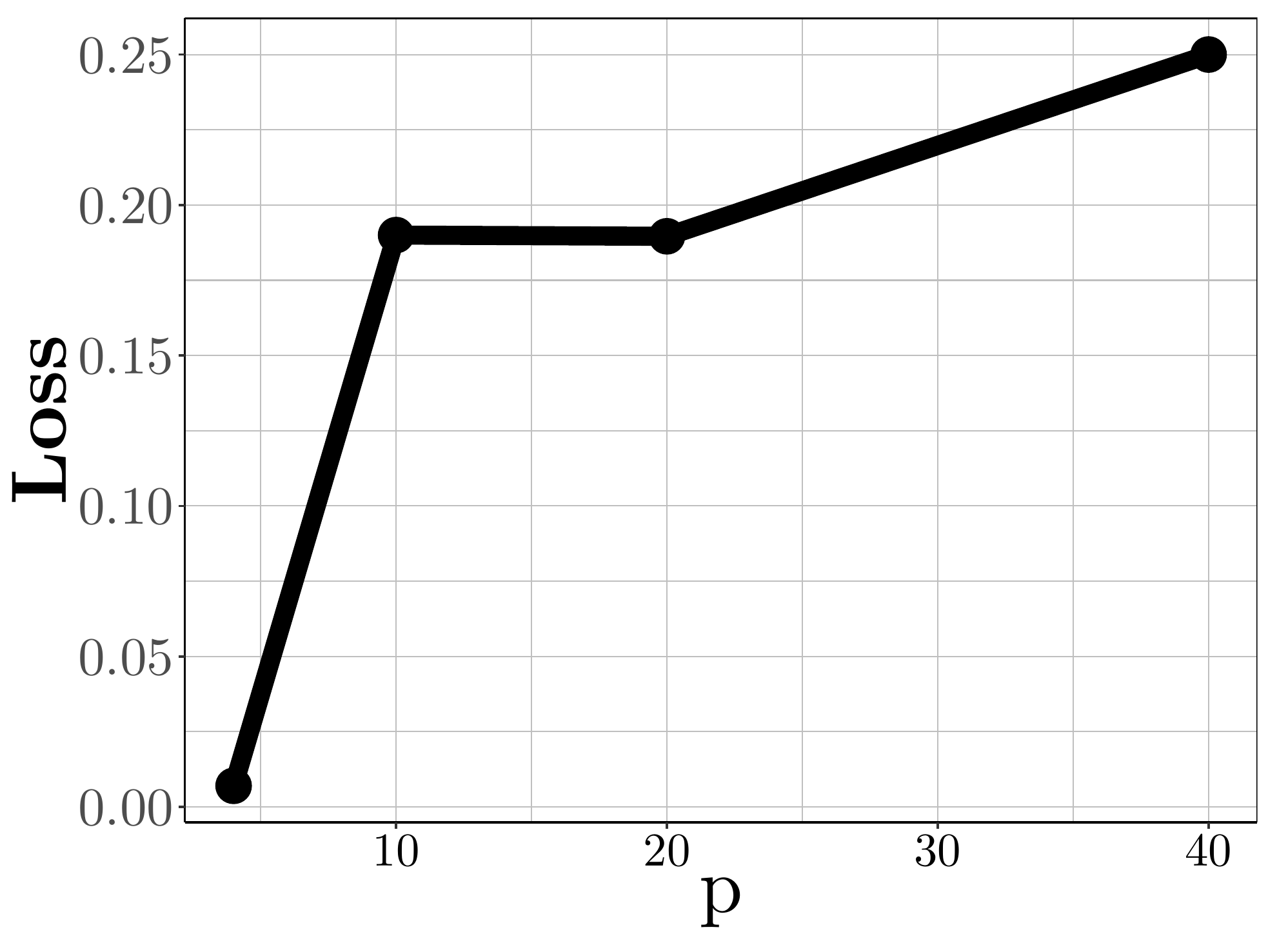} 
		\caption{Neural Networks for reward.}
	\end{subfigure}
	\centering
	\caption{Final loss on the training datasets of neural networks for different tasks for MDPs with $p = \curly{5,10,20,40}$. }
	\label{fig:train-error}
\end{figure*}

\subsection{Extended Synthetic Experiments}
\label{sec:ext-syn-sup}

\subsubsection{Robustness Analysis}
\label{sec:robust-analysis}
It should be noticed that the performance of the proposed PQR method relies on the existence of anchor actions in Assumption~\ref{asm:identification}. While as we have mentioned this assumption is reasonable in many real-world scenarios~\citep{hotz1993conditional,bajari2010identification,manzanares2015improving}, it contradicts the setting of SPL-GD and AIRL, where the reward function is assumed to only depend on the state variables. Therefore, in this section, we study the performance of the proposed method when the underlying reward is indeed a state-only function.

Specifically, we replace the reward function in Section~\ref{sec:syn-details} with the following function:
\eqs{
r(\bs) = \tanh(\curly{\bs^\top/p, \ba/4}\cdot \bomega ). 
}
Then, we use the same configurations as in Section~\ref{sec:synthetic}, and compare the performances of PQR, SPL-GD, and AIRL in terms of reward recovery for an MDP with a 10-dimension state variable. When implement the PQR method, we randomly choose one action as the anchor action. 
The results are summarized in Figure~\ref{fig:r-robust}. 

\begin{figure}[H]
	\centering
		\centering
		\includegraphics[scale=0.25]{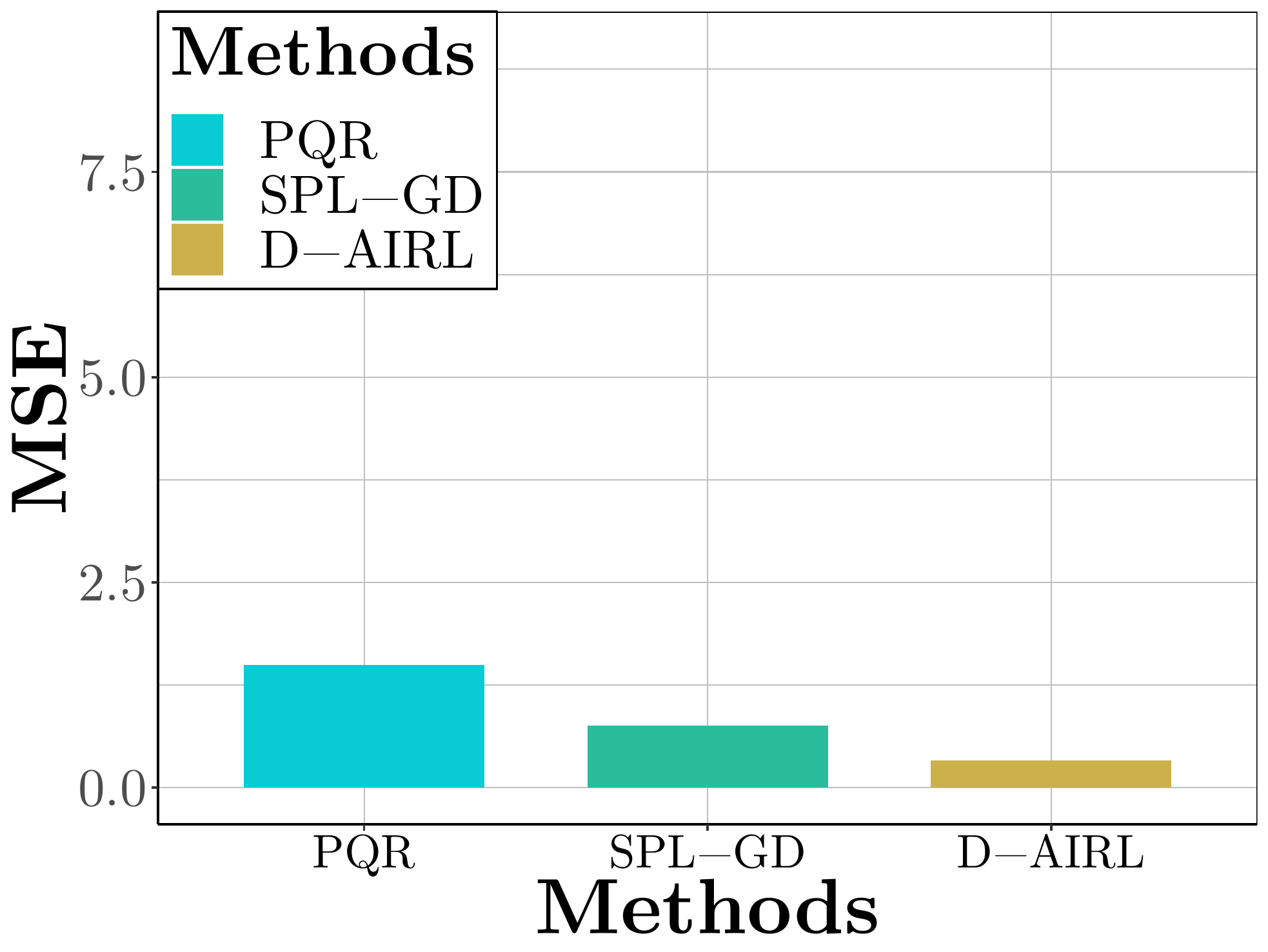}
		\caption{MSE for reward recovery}
		\label{fig:r-robust} 
\end{figure}

Note that D-AIRL performs the best if the reward function is indeed indecent with the action variable. The performance of the proposed method may not be guaranteed if anchor-action assumption is not satisfied. This is not surprising since the PQR method wrongly takes one action as the anchor action for the reward estimation under this setting.

 \subsubsection{Sensitivity Analysis}
 \label{sec:sensitivity-analysis}
Theoretically, it is hard to identify the true $\alpha$ and $\gamma$. However, in some applications, there are some well-justified and widely used values for $\gamma$ and $\alpha$, like taking $\gamma$ as the risk-free return rate in financial applications~\citep{merton1973intertemporal}. Therefore, we treat $\alpha$ and $\gamma$ as two hyperparameters, and provide a sensitivity analysis for PQR.

To analyze $\gamma$, we set $p=10$ and $\alpha = 1$, and solve an MDP to generate three datasets with $\gamma = [0, 0.5, 0.9]$. We apply PQR with $\gamma = 0.9$ to the 3 datasets and check the reward estimation accuracy. We compare with the MaxEnt-IRL method, which is equivalent to taking $\gamma = 0$. Figure~\ref{fig:delta} plots the results. Recall from Section~\ref{sec:PQR} that, when true $\gamma = 0$, the reward function equals to the $Q$-function. Therefore, MaxEnt-IRL gives a very accurate estimation, outperforming PQR with its mis-specified $\gamma=0.9$. As true $\gamma$ increases, PQR performance improves, while MaxEnt-IRL performance deteriorates.
At true $\gamma = 0.9$, PQR achieves low error, while the error of MaxEnt-IRL blows up.
It is important to note that the PQR error does not blow, as the identification procedure Algorithm~\ref{alg:fqii} constrains the reward estimation to the right scale. 

To analyze $\alpha$, we set $p=10$ and $\gamma = 0.9$, and solve an MDP to generate three datasets with $\alpha = [0.1, 0.5, 1]$. We apply PQR with $\alpha = 1$ and report results in Figure~\ref{fig:delta}. As true $\alpha$ approaches $1$, PQR performance improves. When true $\alpha = 0.1$, PQR error increases significantly. This indicates that PQR may be relatively more sensitive to $\alpha$. Meanwhile, most IRL methods operate under $\alpha = 1$, with no ability to tune it. We provide a strategy to choose $\alpha$ in Section~\ref{sec:alpha}. 

\begin{figure}[H] 
\begin{subfigure}{0.48\textwidth}
	\centering 
		\includegraphics[scale=0.2]{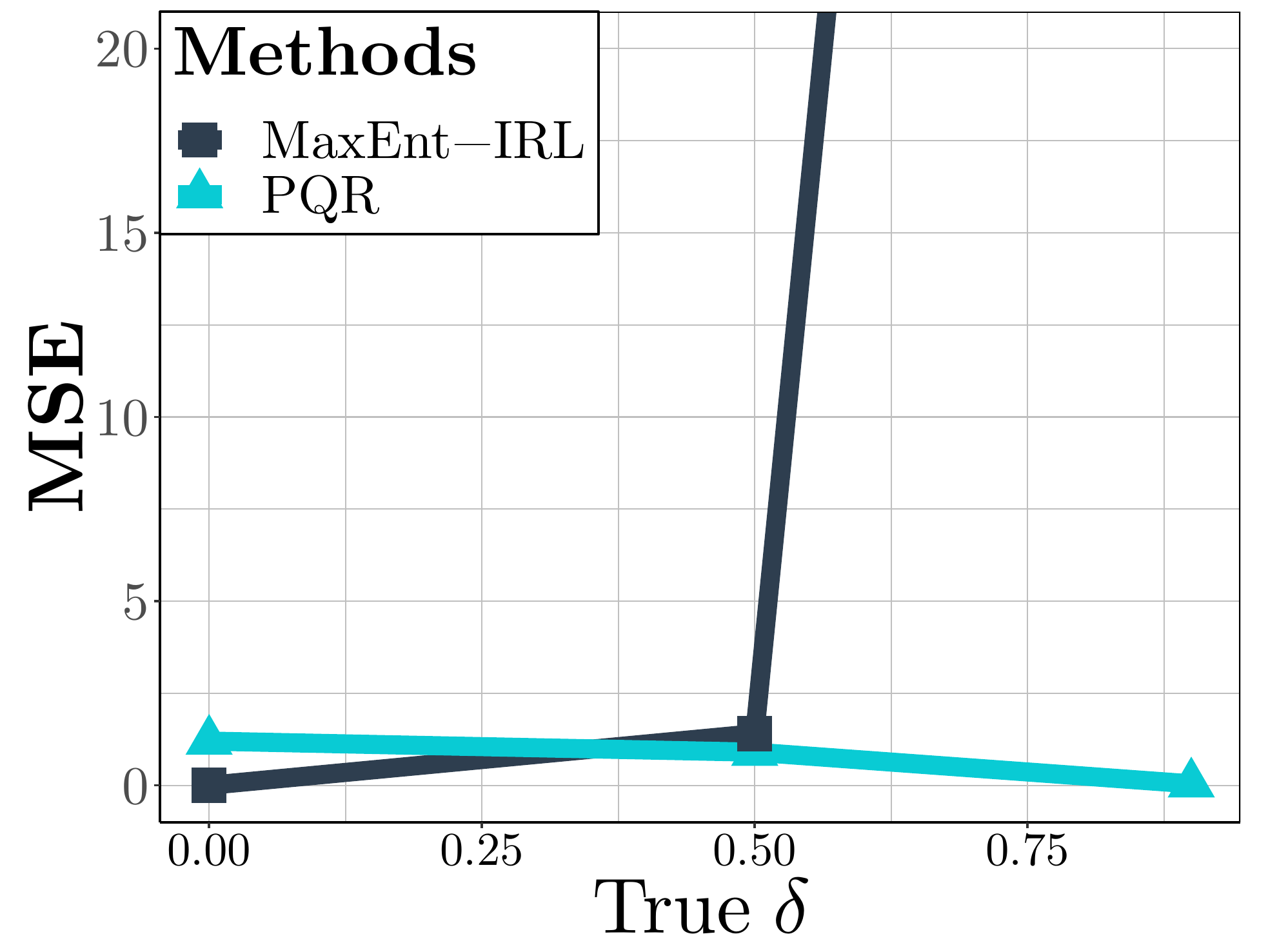}
	\end{subfigure}
	\begin{subfigure}{0.48\textwidth}
		\centering
		\includegraphics[scale=0.2]{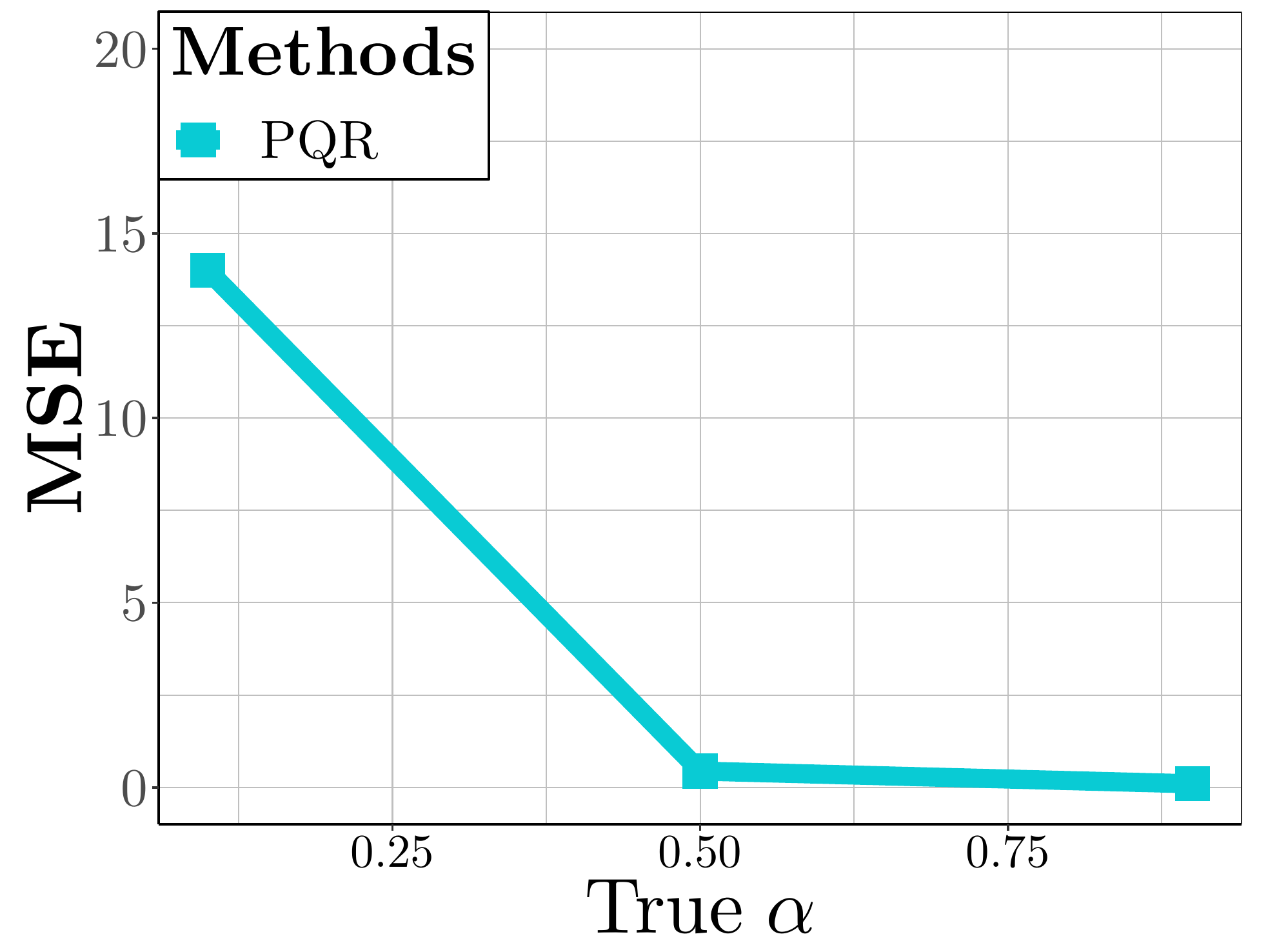}
	\end{subfigure}
	\centering
	\caption{Sensitivity analysis for $\gamma$ and $\alpha$}
	\label{fig:delta}
\end{figure}

\subsubsection{$\alpha$ Selection}
\label{sec:alpha}
In this section, we propose a strategy to select $\alpha$, under the assumption that we approximately know the scale of the true reward function, such as the average reward, maximal reward, or the minimal reward achieved in the dataset.  
To start with, we study the effect of $\alpha$. 

\begin{theorem}
Given the demonstration dataset $\mathds{X} = \curly{\bs_0, \ba_0, \bs_1, \ba_1, \cdots, \bs_T, \ba_T}$,
we use $\hat{r}_{\alpha_1}$ and $\hat{r}_{\alpha_2}$ to denote two estimated reward functions by PQR, using $\alpha_1, \alpha_2 \in (0,\infty)$ respectively. We assume that $\hat{\pi}$ is an accurate estimation to ${\pi^*}$, and the expectation in PQR estimators are exactly calculated. Therefore,
\alis{ 
\frac{\hat{r}_{\alpha_1}(\bs,\ba) - \frac{g(\bs)}{1-\gamma} + \gamma \EE\left[\frac{g(\bs')}{1-\gamma} \given \bs, \ba \right]}{
\hat{r}_{\alpha_2}(\bs,\ba) - \frac{g(\bs)}{1-\gamma} + \gamma \EE\left[\frac{g(\bs')}{1-\gamma} \given \bs, \ba \right]
}
= \frac{\alpha_1}{ \alpha_2},
}
where the expectation is on $\bs'$ over one-step transition. 
\end{theorem}

\begin{proof}
First, it should be noticed that $\alpha$ does not affect the estimation to $\pi^*$. 

Then, according to \ref{eq:contraction}, we have
\alis{
    \hat{Q}^A_{\alpha_1}(\bs) &= \hat{\mathcal{T}}_{\alpha_1}\hat{Q}^A_{\alpha_1}(\bs)
    \\ \hat{Q}^A_{\alpha_2}(\bs) &= \hat{\mathcal{T}}_{\alpha_2}\hat{Q}^A_{\alpha_2}(\bs),
}
where 
\alis{
    \hat{\mathcal{T}}_{\alpha_1}f(\bs) &:= g(\bs)+\gamma \hat{\EE}_{\bs'}{\left[ -\alpha_1 \log(\hat{\pi}(\bs', \ba^{A})) + f(\bs') \given \bs, \ba^{A} \right]}
    \\\hat{\mathcal{T}}_{\alpha_2}f(\bs) &:= g(\bs)+\gamma \hat{\EE}_{\bs'}{\left[ -\alpha_2 \log(\hat{\pi}(\bs', \ba^{A})) + f(\bs') \given \bs, \ba^{A} \right]}.
}

Therefore, 
\alis{
    \hat{Q}^A_{\alpha_1}(\bs) - \frac{g(\bs)}{1-\gamma} &:= \gamma \hat{\EE}_{\bs'}{\left[ -\alpha_1 \log(\hat{\pi}(\bs', \ba^{A})) +  \hat{Q}^A_{\alpha_1}(\bs) - \frac{g(\bs)}{1-\gamma}  \given \bs, \ba^{A} \right]}
    \\\hat{Q}^A_{\alpha_2}(\bs) - \frac{g(\bs)}{1-\gamma} &:= \gamma \hat{\EE}_{\bs'}{\left[ -\alpha_2 \log(\hat{\pi}(\bs', \ba^{A})) +  \hat{Q}^A_{\alpha_2}(\bs) - \frac{g(\bs)}{1-\gamma}  \given \bs, \ba^{A} \right]},
} 
which suggests that 
\eqs{
    \frac{\hat{Q}^A_{\alpha_1}(\bs) - \frac{g(\bs)}{1-\gamma}}{\hat{Q}^A_{\alpha_2}(\bs) - \frac{g(\bs)}{1-\gamma}} = \frac{\alpha_1}{\alpha_2}.
}
Therefore, by \ref{eq:q}, we can conclude that
\eq{alpha-iden1}{
 \frac{\hat{Q}_{\alpha_1}(\bs, \ba) - \frac{g(\bs)}{1-\gamma}}{\hat{Q}_{\alpha_2}(\bs, \ba) - \frac{g(\bs)}{1-\gamma}} = \frac{\alpha_1}{\alpha_2}.
}
By taking \eqref{eq:alpha-iden1} into \ref{eq:main}, we have

\alis{
\frac{\hat{r}_{\alpha_1}(\bs,\ba) - \frac{g(\bs)}{1-\gamma} + \gamma \EE\left[\frac{g(\bs')}{1-\gamma} \given \bs, \ba \right]}{
\hat{r}_{\alpha_2}(\bs,\ba) - \frac{g(\bs)}{1-\gamma} + \gamma \EE\left[\frac{g(\bs')}{1-\gamma} \given \bs, \ba \right]
}
= \frac{\alpha_1}{ \alpha_2}.
}

\end{proof}
As a special case, when $g(\bs) = 0$, we have
\eqs{
\frac{\hat{r}_{\alpha_1}(\bs,\ba)}{
\hat{r}_{\alpha_2}(\bs,\ba)
}
= \frac{\alpha_1}{ \alpha_2}.
}

Therefore, the $\alpha$ used in PQR determines the scale of the estimated reward function. If we know the scale of the true reward function, we can tune $\alpha$ so that the estimated reward function has the right scale. When $g(\bs) = 0$, the procedure to select $\alpha$ is summarized in Algorithm~\ref{alg:alpha}:
\begin{center}
    
\begin{minipage}{.6\linewidth}
\begin{algorithm}[H]
\caption{$\alpha$ Selection}
		\label{alg:alpha} 
		\begin{algorithmic}[1]
			\Input Dataset: $\mathds{X} = \curly{\bs_0, \ba_0, \bs_1, \ba_1, \cdots, \bs_T, \ba_T}$.
			\Input $R$, the average reward achieved on $\mathds{X}$.
			\Input $\gamma$ and $N$.
			\State $\hat{r}(\bs,\ba) \leftarrow \Call{PQR}{\mathds{X}, \alpha=1, \gamma, N}$
			\State$\hat{\alpha} =\frac{R\cdot (T+1)}{\sum_{(\bs,\ba) \in \mathds{X}}\hat{r}(\bs,\ba)}$
			\State \Return $\hat{\alpha}$

		\end{algorithmic} 
	\end{algorithm}
\end{minipage}
\end{center}

\subsection{Details of Real-World Experiments}
\label{sec:real-world-sup}
In this section, we provide detailed configurations of the PQR method when applied to the airline market entry analysis. To start with, we follow the convention in~\citet{benkard2004dynamic,berry2010tracing} and let $\alpha = 1$. As mentioned in Section~\ref{sec:airline}, we only focus on the top 60 CSAs, summarized in Table~\ref{tab:csa}, and 11 airline companies summarized in Table~\ref{tab:dict-airline}. 

We take $\delta = 0.95$. Specifically, the hyperparameters used are provided in Table~\ref{tab:deep-airline}. We train the methods using the data from the start of 2013 to the end of 2014, and calculate the likelihood using the data of January 2015.  

\begin{table*}[h]
\centering
\begin{tabular*}{0.9\textwidth}{@{\extracolsep{\fill} }c c| c  c}
\hline
  \multicolumn{1}{ c}{\multirow{2}{*}{CSAs }} &\multicolumn{1}{ c|}{\multirow{2}{*}{Airport Codes}} &\multicolumn{1}{ c}{\multirow{2}{*}{CSAs}}& \multicolumn{1}{ c}{\multirow{2}{*}{Airport Codes}}\\ 
\multicolumn{1}{ c  }{}&\multicolumn{1}{ c | }{}&\multicolumn{1}{ c  }{}&\multicolumn{1}{ c  }{}\\
 \hline 
Albuquerque&ABQ&Orlando&MCO\\
Albany&ALB&Chicago&ORD, MDW\\
Anchorage&ANC&Memphis&MEM\\
Atlanta&ATL&Milwaukee&MKE\\
Austin&AUS&Minneapolis-St.Paul&MSP\\
Hartford&BDL&NewOrleans&MSY\\
Birmingham&BHM&SanFrancisco&SFO, SJC, OAK\\
Nashville&BNA &Kahului&OGG\\ 
Boise&BOI&OklahomaCity&OKC\\
Boston&BOS, MHT, PVD&Omaha&OMA\\
Buffalo&BUF&Norfolk&ORF\\
LosAngeles&LAX, ONT, SNA, BUR&PalmBeach&PBI\\
WashingtonDC&IAD, DCA, BWI&Portland&PDX\\
Cleveland&CLE&Philadelphia&PHL\\
Charlotte&CLT&Phoenix&PHX\\
Columbus&CMH&Pittsburgh&PIT\\
Cincinnati&CVG&Raleigh-Durham&RDU\\
Dallas&DFW, DAL&Reno&RNO\\
Denver&DEN&Southwest Florida&RSW\\
Detroit&DTW&SanDiego&SAN\\ 
ElPaso&ELP&SanAntonio&SAT\\
New York&JFK, LGA, EWR&Louisville&SDF\\
Miami&MIA, FLL&Seattle&SEA\\
Spokane&GEG&SanJuan&SJU\\ 
Honolulu&HNL&SaltLakeCity&SLC\\
Houston&IAH, HOU&Sacramento&SMF\\
Indianapolis&IND&St.Louis&STL\\
Jacksonville&JAX&Tampa&TPA\\
LasVegas&LAS&Tulsa&TUL\\ 
KansasCity&MCI&Tuscon&TUS\\
\hline 
\end{tabular*}
\caption {Top 60 CSAs with the most itineraries in 2002.}
\label{tab:csa}
\end{table*}

\begin{table*}[h]
\centering
\begin{tabular}{c c}
\hline
  \multicolumn{1}{ c}{\multirow{2}{*}{Airline Companies }} &\multicolumn{1}{ c}{\multirow{2}{*}{Airline Company Codes}} \\ 
\multicolumn{1}{ c  }{}&\multicolumn{1}{ c }{}\\
\hline
American Airlines & AA\\
Alaska Airlines & AS\\
Jetblue Airlines &B6\\
Continental Airlines& CO\\
Delta &DL\\
America West&HP\\
Northwest &NW\\
TWA &TW\\
United Airlines&UA\\
US Airways &US\\
Southwest  &WN\\
 \hline 
\hline
\end{tabular}
\caption {List of airline companies and their codes.} 
\label{tab:dict-airline}  
\end{table*}

The state variables include the distance between CSAs, 
the populations of the origin and destination CSAs, the log of the current passenger density for the market, whether the carrier is already flying for the market, the number of nonstop competitors for the market, the share of flights operated out of the origin CSA and the destination CSA, whether the carrier operates any flights out of the origin or destination, and the number of flights operated by the carrier in the current period.

The estimated reward functions of all the 11 airline companies for the top 5 markets are provided in Figure~\ref{fig:airline-re-all}.

\begin{table*}[h]
\centering
\begin{tabular*}{0.9\textwidth}{@{\extracolsep{\fill} }c c c  c}
\hline
  \multicolumn{1}{ c}{\multirow{2}{*}{Types of Networks}} &\multicolumn{1}{ c}{\multirow{2}{*}{Optimization Method}} &\multicolumn{1}{ c}{\multirow{2}{*}{Learning Rate}}& \multicolumn{1}{ c}{\multirow{2}{*}{Number of Steps}}\\ 
\multicolumn{1}{ c  }{}&\multicolumn{1}{ c  }{}&\multicolumn{1}{ c  }{}&\multicolumn{1}{ c  }{}\\
 \hline 
Policy Estimation &Adagrad  &$1.5\times10^{-4}$& 1000\\
 $Q$ Estimation & Adagrad &$0.5\times 10^{-4}$&500\\
Reward Estimation & Adagrad &$1\times 10^{-3}$&1000\\
  \hline 
\end{tabular*}
\caption {Hyperparameters for PQR for the airline market entry analysis.}
\label{tab:deep-airline}
\end{table*}

\begin{figure}[H] 
	\begin{subfigure}{0.48\textwidth}
		\centering
		\includegraphics[scale=0.14]{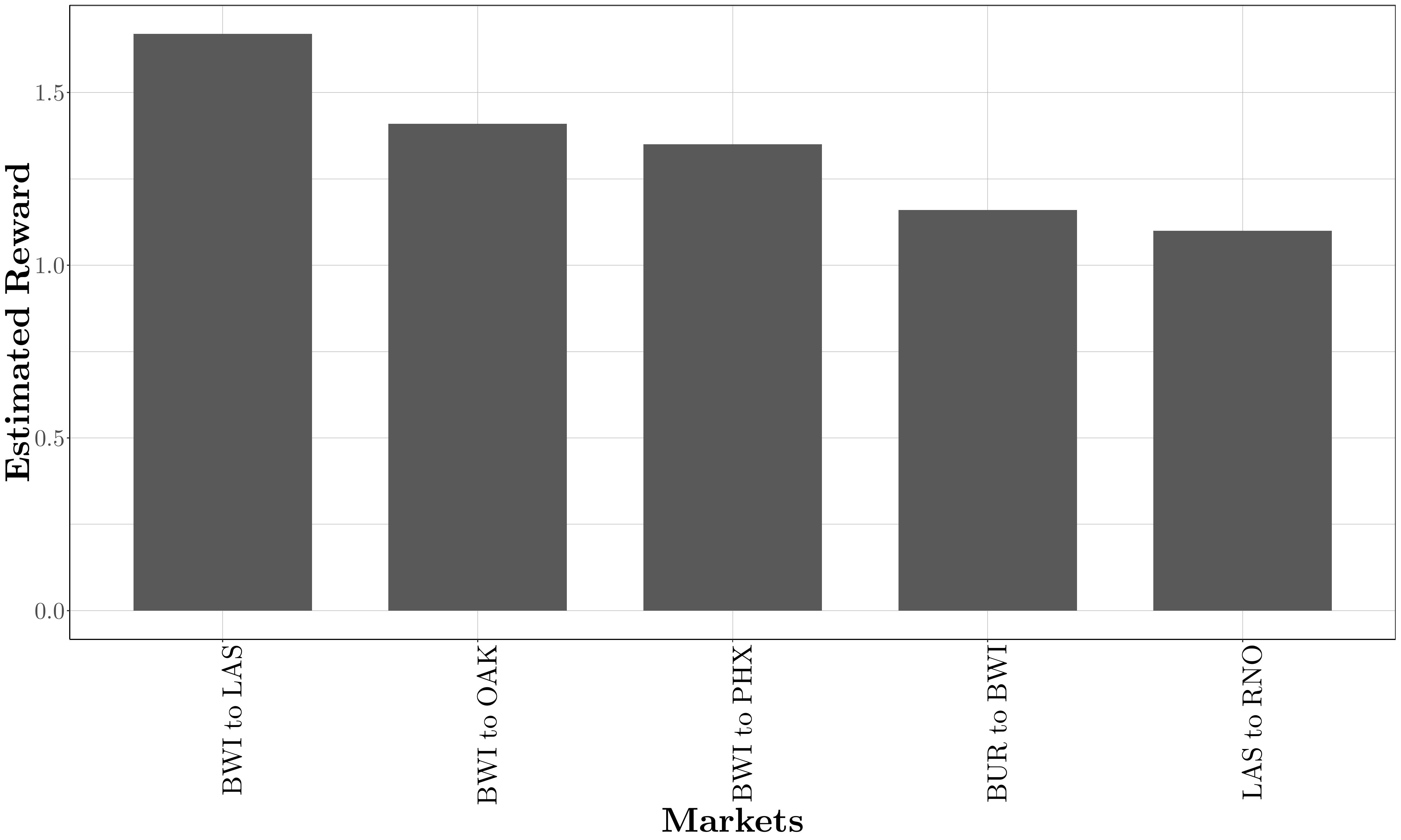}
		\caption{Alaska Airlines} 
	\end{subfigure} 
	\centering 
\begin{subfigure}{0.48\textwidth}
	\centering  
		\includegraphics[scale=0.14]{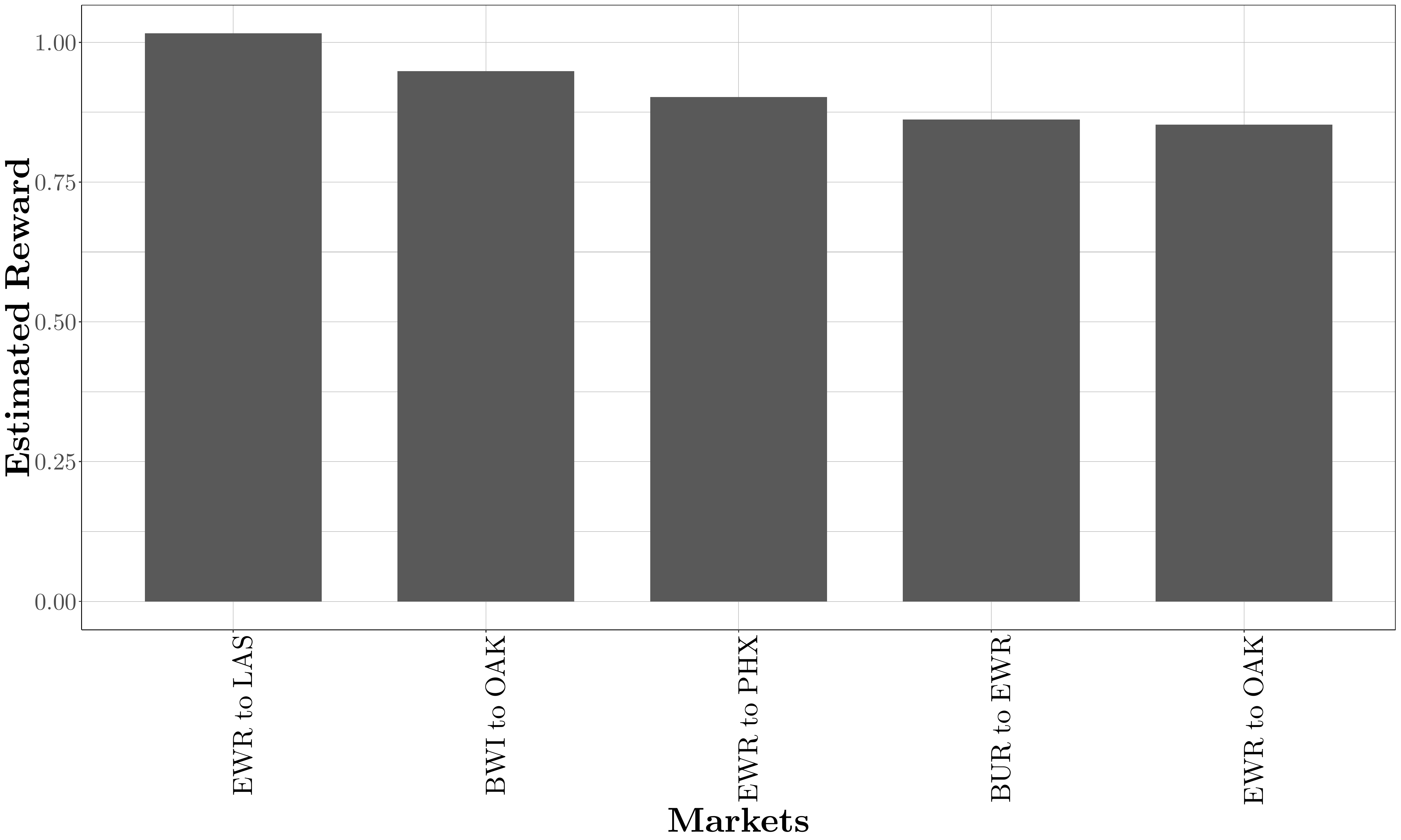} 
		\caption{Jetblue Airlines}
	\end{subfigure}
	\begin{subfigure}{0.48\textwidth}
		\centering
		\includegraphics[scale=0.14]{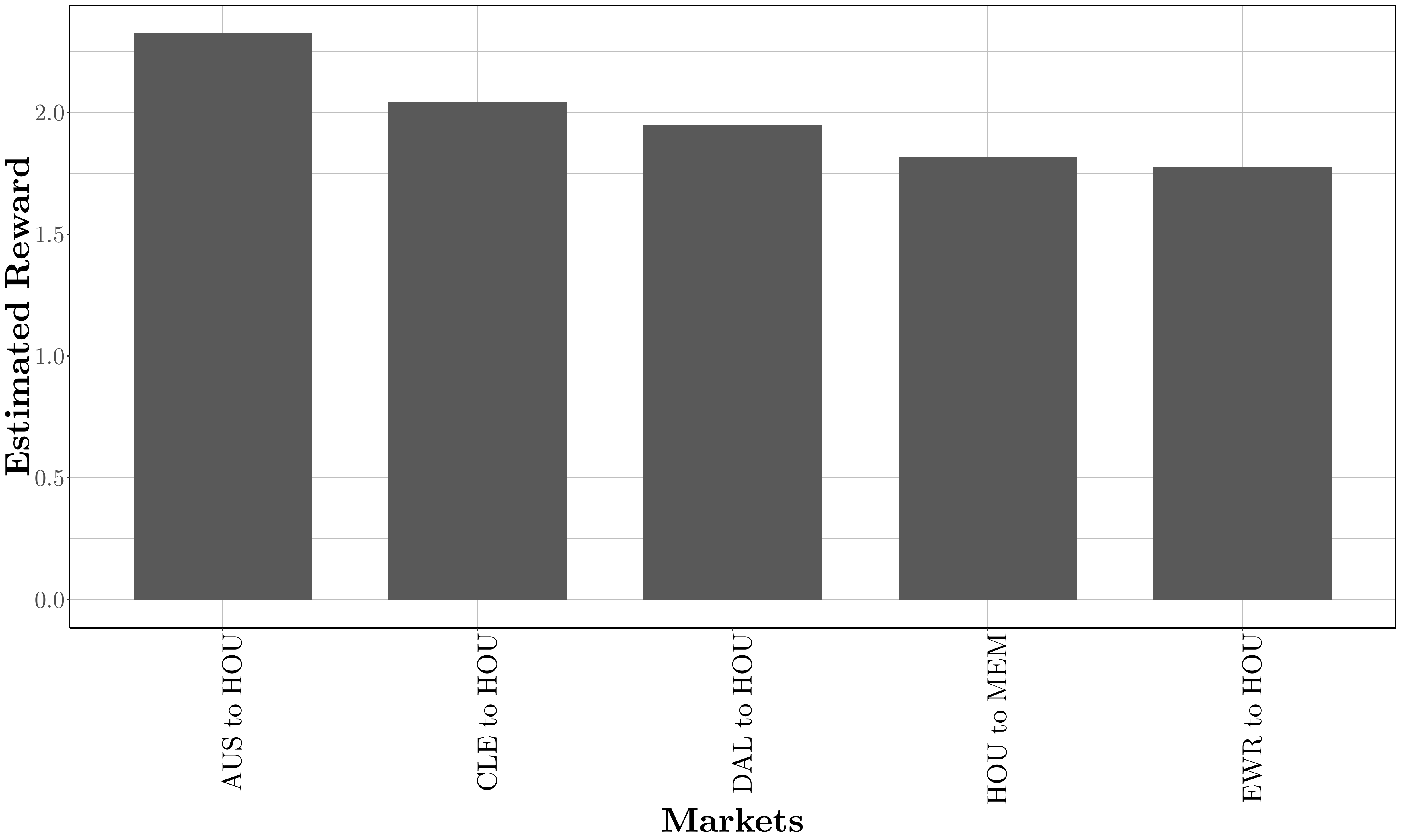}
		\caption{Continental Airline} 
	\end{subfigure} 
	\centering
\begin{subfigure}{0.48\textwidth}
	\centering 
		\includegraphics[scale=0.14]{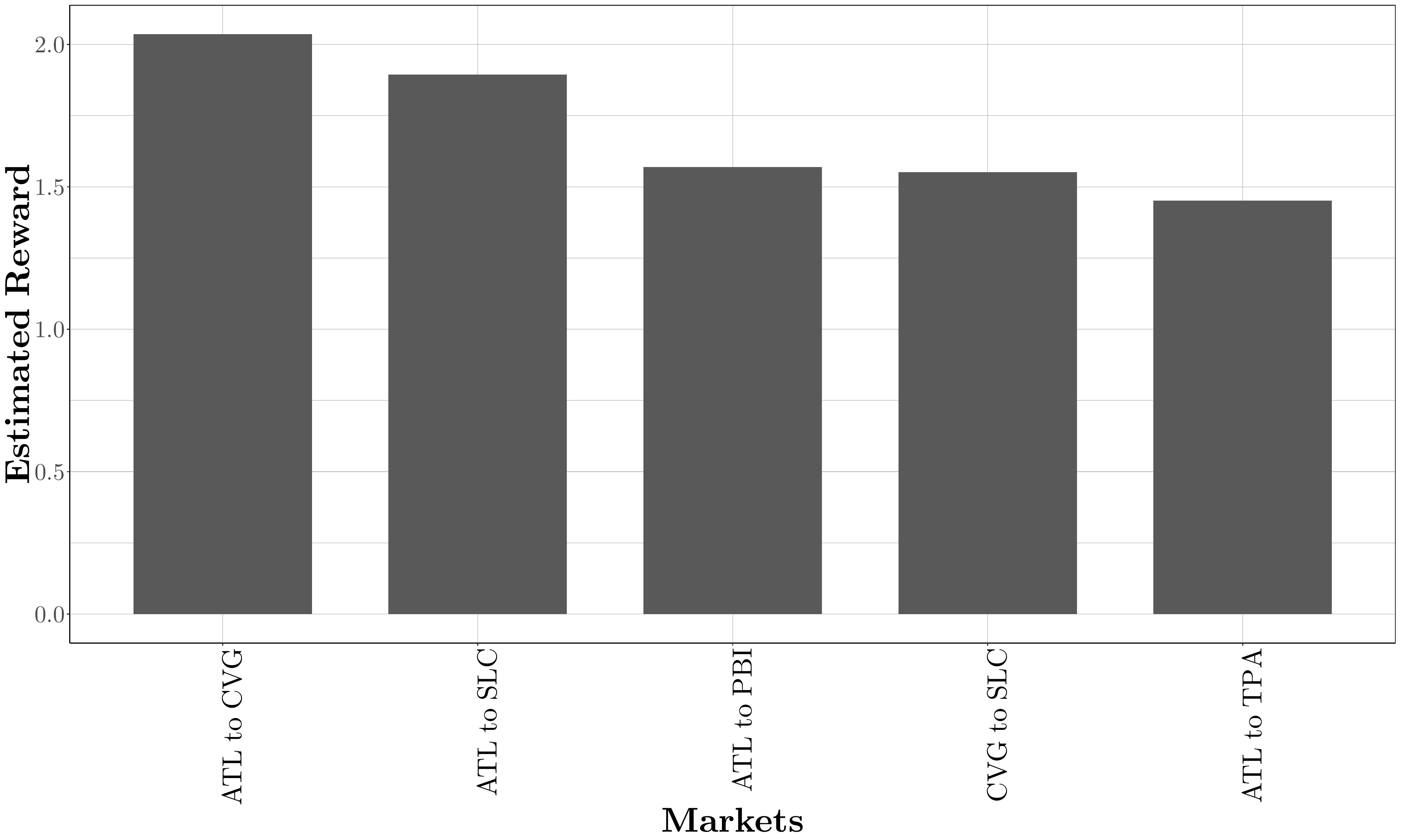} 
		\caption{Delta}
	\end{subfigure}
	\begin{subfigure}{0.48\textwidth}
		\centering
		\includegraphics[scale=0.14]{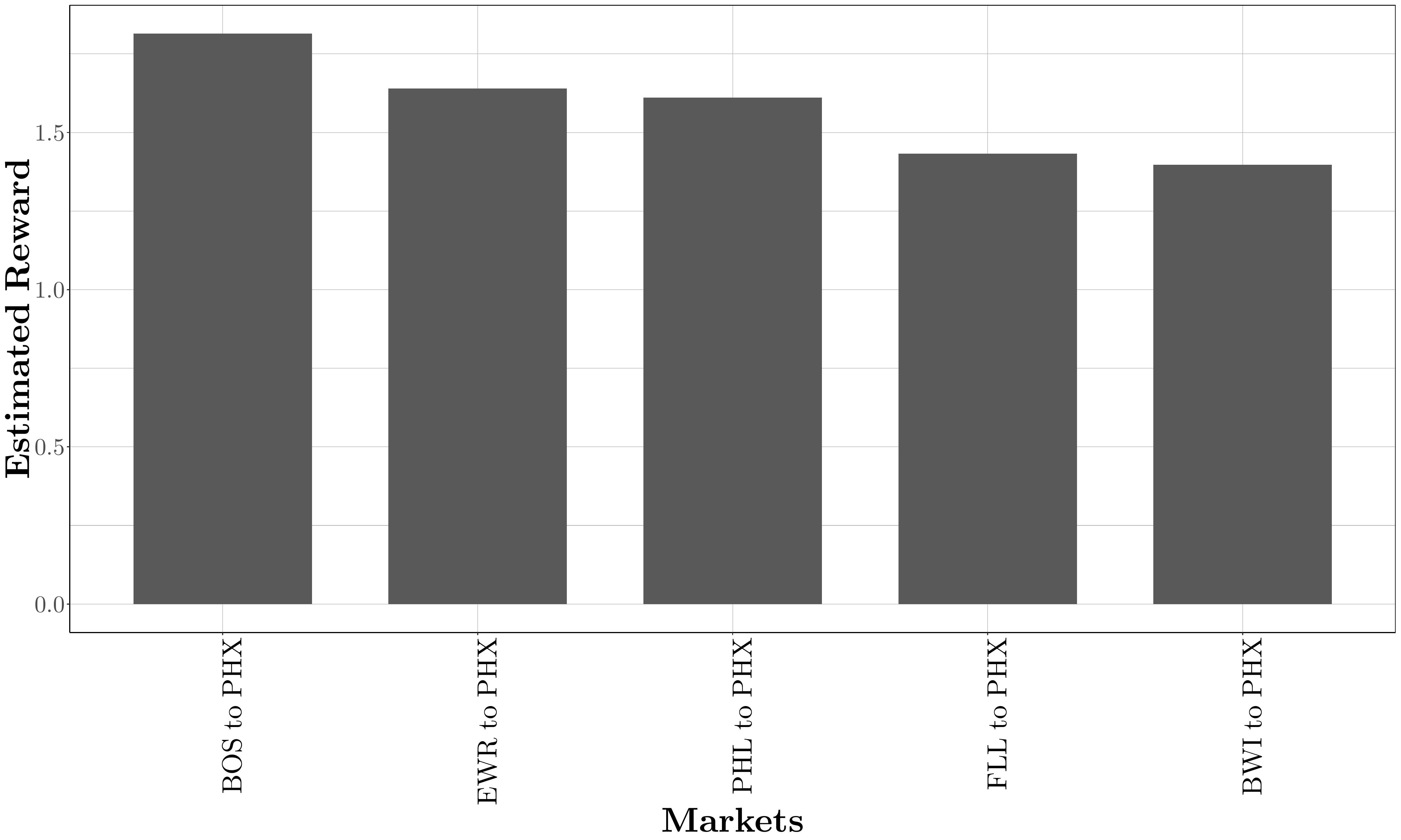}
		\caption{America West} 
	\end{subfigure} 
	\centering
\begin{subfigure}{0.48\textwidth}
	\centering 
		\includegraphics[scale=0.14]{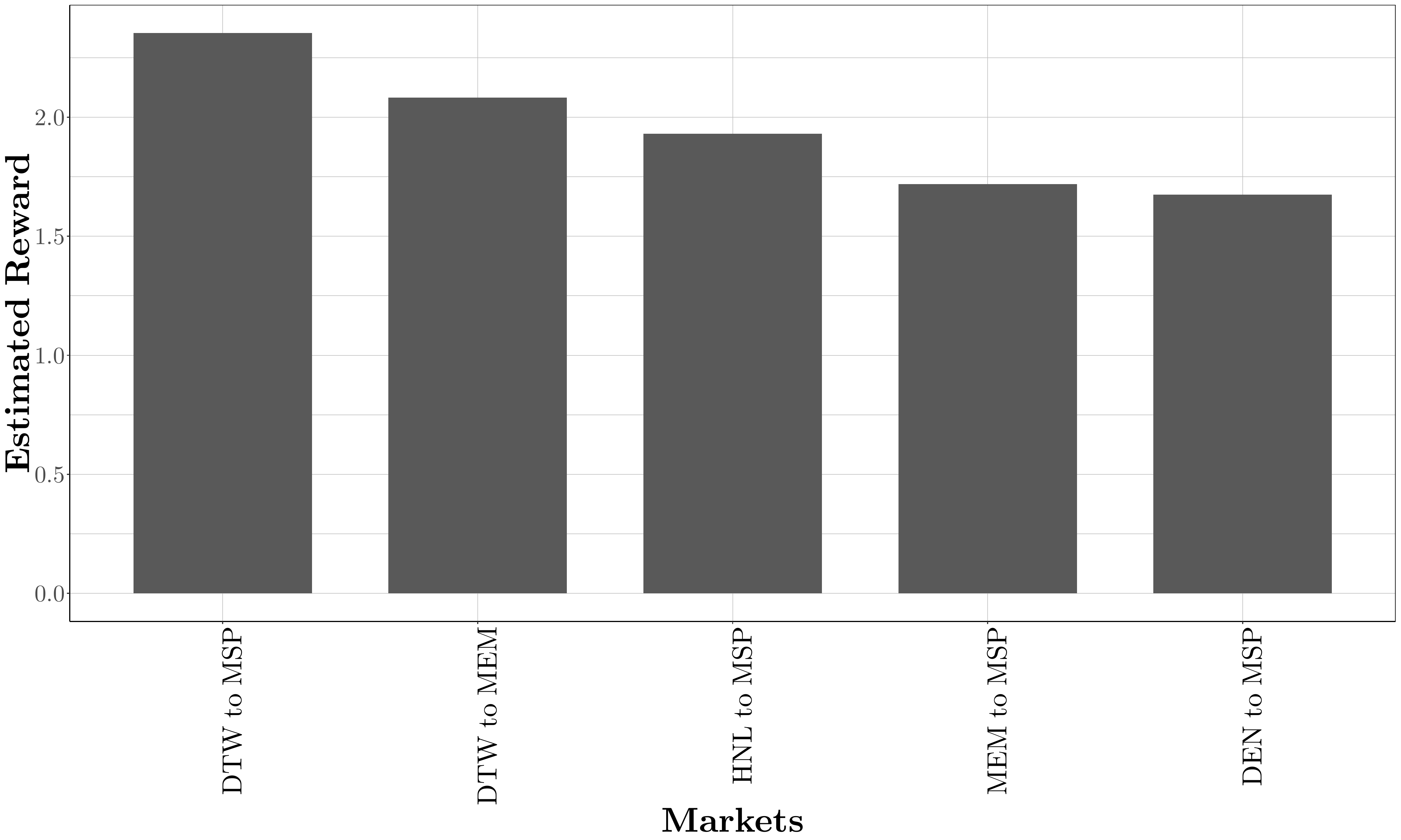} 
		\caption{Northwest}
	\end{subfigure}
	\begin{subfigure}{0.48\textwidth}
		\centering
		\includegraphics[scale=0.14]{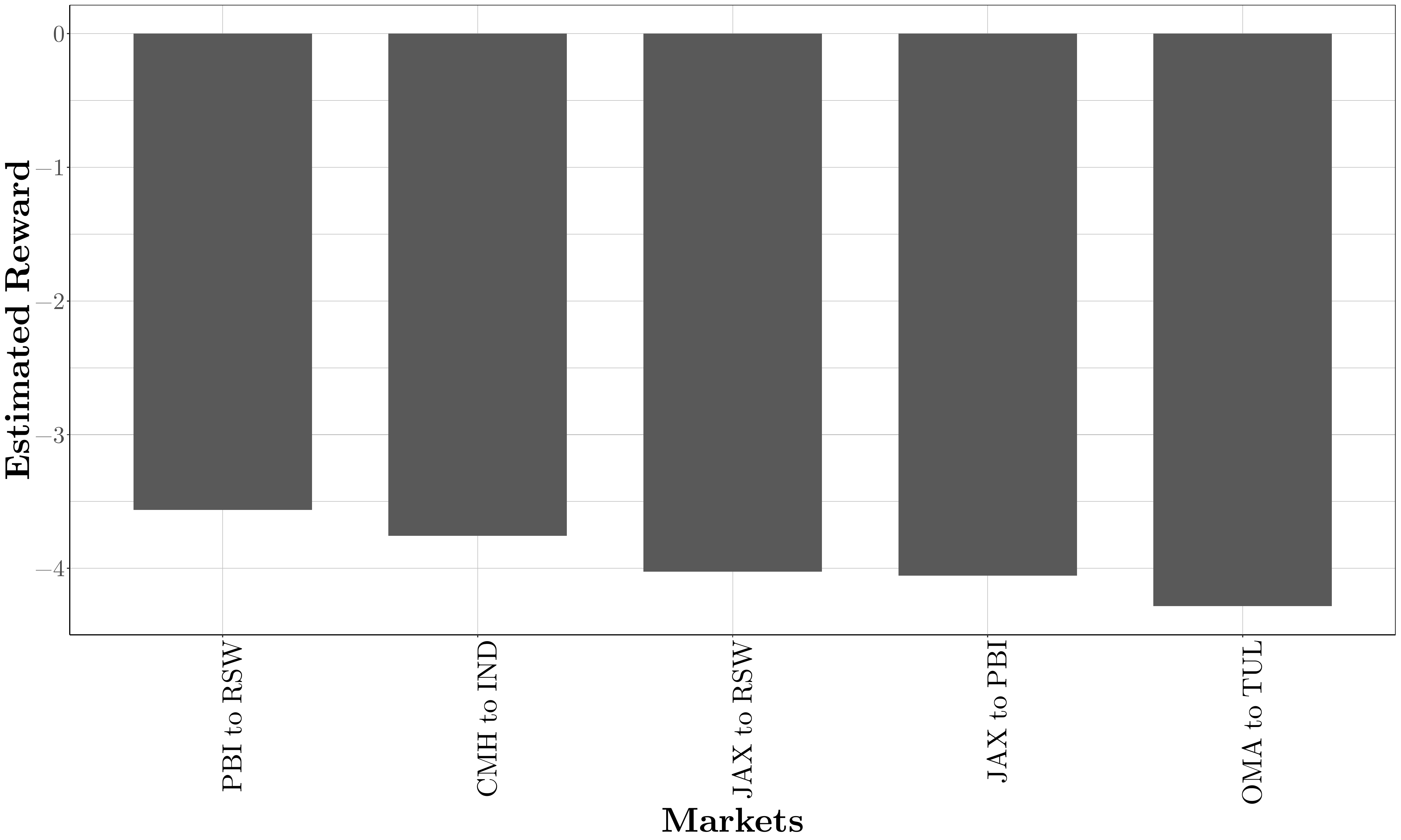}
		\caption{TWA} 
	\end{subfigure} 
	\centering
\begin{subfigure}{0.48\textwidth}
	\centering 
		\includegraphics[scale=0.14]{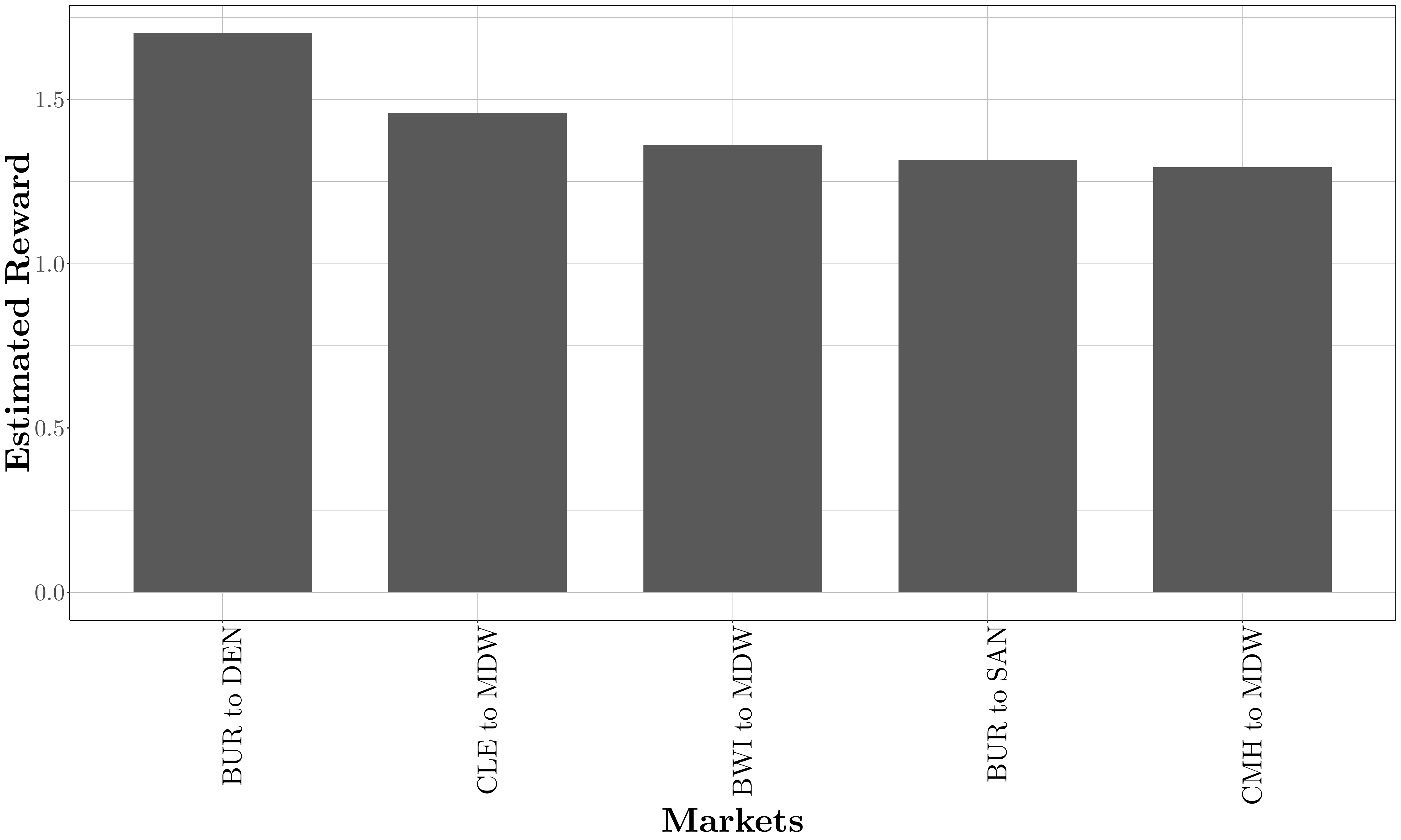} 
		\caption{United Airlines}
	\end{subfigure} 
	\centering
	\begin{subfigure}{0.48\textwidth}
	\centering 
		\includegraphics[scale=0.14]{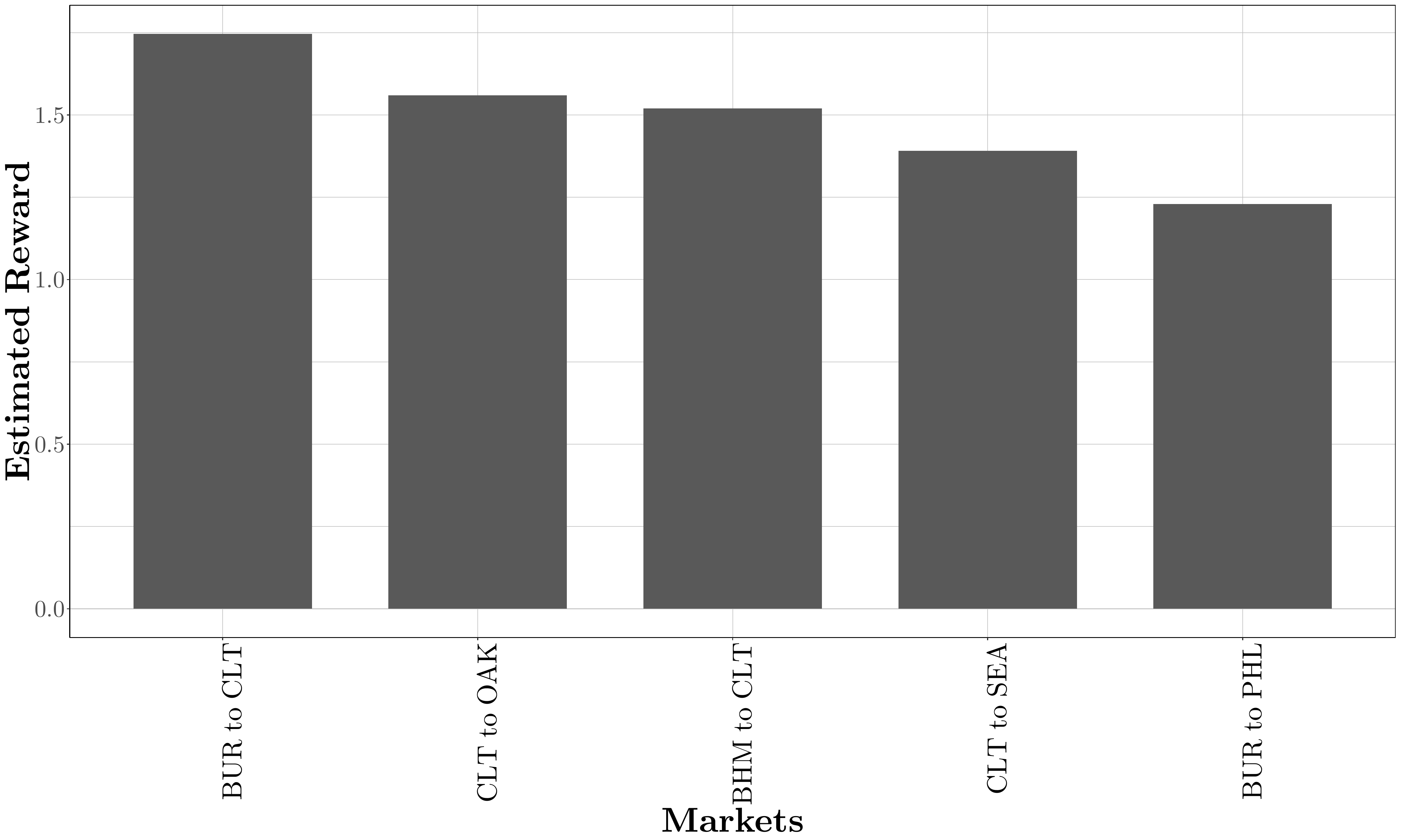} 
		\caption{US Airways}
	\end{subfigure} 
	\begin{subfigure}{0.48\textwidth}
	\centering 
		\includegraphics[scale=0.14]{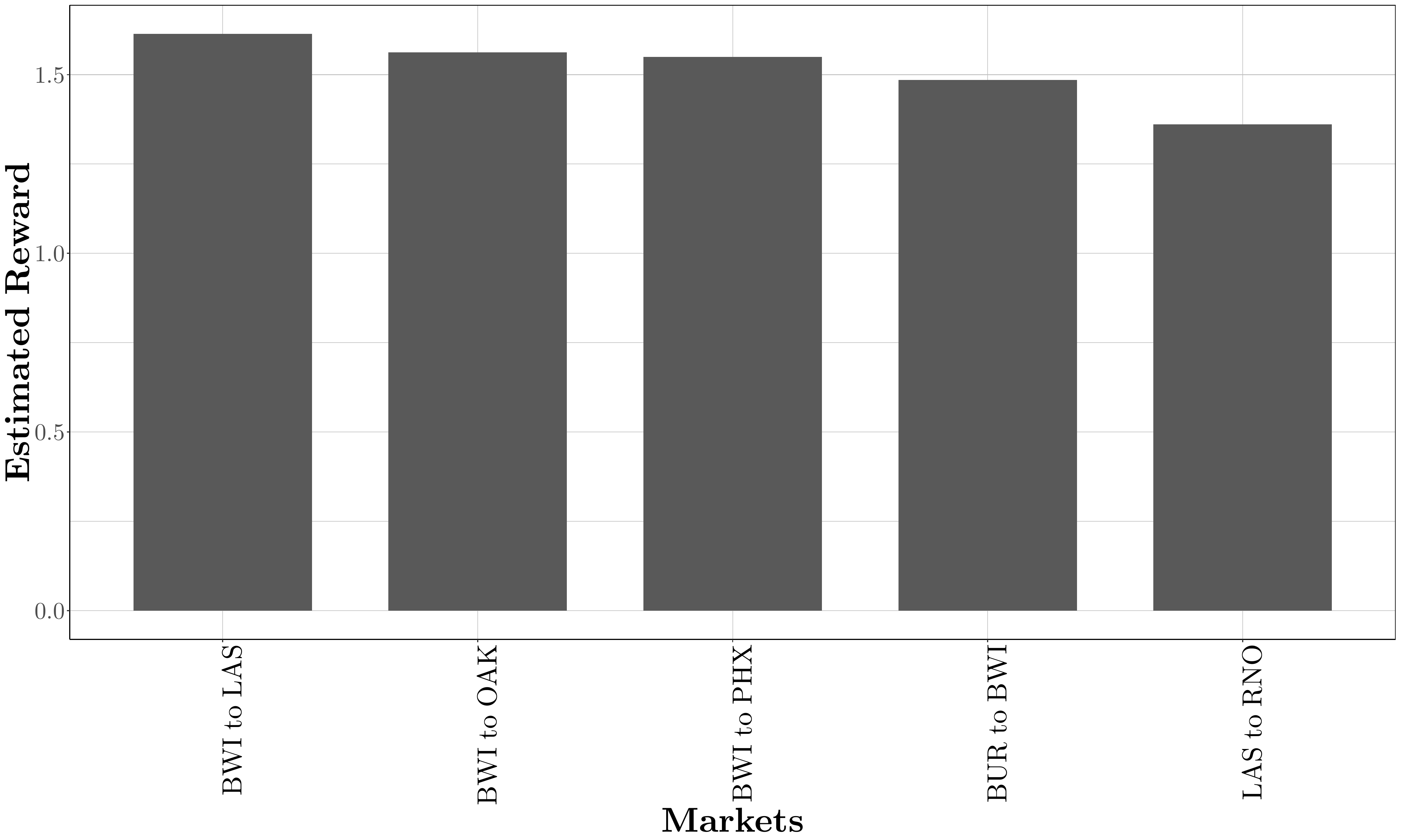} 
		\caption{Southwest}
	\end{subfigure} 
	\caption{Estimated reward functions for airline companies for top 5 markets. }
	
	\label{fig:airline-re-all}
\end{figure}

\end{document}